\documentclass[twoside,11pt]{article}

\usepackage{comment}
\usepackage{xcolor}

\usepackage{blindtext}

\usepackage{jmlr2e}

\usepackage{amsmath}
\usepackage[capitalise]{cleveref}

\newcommand{\sign}{\mathrm{sign}}

\newcommand{\iid}{\mathrm{i.i.d.}}
\newcommand{\beq}{\begin{equation}}
\newcommand{\eeq}{\end{equation}}

\newcommand{\polylog}{\mathrm{polylog}}

\newcommand{\Unif}{\mathrm{Unif}}

\newcommand{\sto}{\mathsf{stop}}
\newcommand{\iidsim}{\overset{\scriptsize\iid}{\sim}}
\newcommand{\norm}[1]{\left\|{#1}\right\|}

\newcommand{\R}{\mathbb{R}}

\newcommand{\veps}{\varepsilon}

\renewcommand{\P}{\mathbb{P}}
\renewcommand{\S}{\mathbb{S}}
\newcommand{\NN}{\mathbb{N}}

\newcommand{\bT}{\mathrm{\bf T}}

\newcommand{\bg}{\mathrm{\bf g}}

\newcommand{\bS}{\mathrm{\bf S}}

\newcommand{\xx}{\text{\boldmath $x$}}

\newcommand{\HH}{\text{\boldmath $H$}}
\newcommand{\XX}{\text{\boldmath $X$}}

\newcommand{\ww}{\text{\boldmath $w$}}

\newcommand{\WW}{\text{\boldmath $W$}}
\newcommand{\VV}{\text{\boldmath $V$}}

\newcommand{\vv}{\text{\boldmath $v$}}

\newcommand{\uu}{\text{\boldmath $u$}}

\newcommand{\ee}{\text{\boldmath $e$}}

\newcommand{\hh}{\text{\boldmath $h$}}

\newcommand{\balpha}{\text{\boldmath $\alpha$}}

\newcommand{\bbeta}{\text{\boldmath $\beta$}}

\newcommand{\bone}{\mathrm{\bf 1}}
\newcommand{\bzero}{\mathrm{\bf 0}}

\newcommand{\bsigma}{\text{\boldmath $\sigma$}}

\newcommand{\cF}{\mathcal{F}}

\def\normal{{\sf N}}
\def\id{{\boldsymbol I}}

\newcommand{\eps}{\varepsilon}

\newtheorem{thm}{Theorem}[section]

\newtheorem{lem}{Lemma}[section]

\newtheorem{prop}{Proposition}[section]

\newtheorem{rem}{Remark}[section]

\newcommand{\modif}[1]{{\textcolor{black}{#1}}}

\usepackage{lastpage}
\jmlrheading{25}{2024}{1-\pageref{LastPage}}{1/24; Revised
5/24}{6/24}{24-0006}{Yuchen Wu and Kangjie Zhou}

\ShortHeadings{Sharp Analysis of Power Iteration for Tensor PCA}{Wu and Zhou}
\firstpageno{1}

\begin{document}

\title{Sharp Analysis of Power Iteration for Tensor PCA}

\author{\name Yuchen Wu \email wuyc14@wharton.upenn.edu \\
       \addr Department of Statistics and Data Science\\
       University of Pennsylvania\\
       Philadelphia, PA 19104-6303, USA
       \AND
       \name Kangjie Zhou \email kangjie@stanford.edu \\
       \addr Department of Statistics\\
       Stanford University\\
       Stanford, CA 94305-2004, USA}

\editor{Daniel Hsu}

\maketitle

\begin{abstract}
We investigate the power iteration algorithm for the tensor PCA model introduced in \cite{richard2014statistical}. Previous work studying the properties of tensor power iteration is either limited to a constant number of iterations, or requires a non-trivial data-independent initialization. In this paper, we move beyond these limitations and analyze the dynamics of randomly initialized tensor power iteration up to polynomially many steps. Our contributions are threefold: First, we establish sharp bounds on the number of iterations required for power method to converge to the planted signal, for a broad range of the signal-to-noise ratios. Second, our analysis reveals that the actual algorithmic threshold for power iteration is smaller than the one conjectured \modif{in the literature} by a $\polylog(n)$ factor, where $n$ is the ambient dimension. Finally, we propose a simple and effective stopping criterion for power iteration, which provably outputs a solution that is highly correlated with the true signal. Extensive numerical experiments verify our theoretical results.
\end{abstract}


\section{Introduction}

Tensors are multi-dimensional arrays that have found wide applications across various domains, including neuroscience \citep{wozniak2007neurocognitive,zhou2013tensor}, recommendation systems \citep{rendle2010pairwise,shah2019iterative}, image processing \citep{liu2012tensor,sidiropoulos2017tensor}, community detection \citep{nickel2011three,jing2021community}, and genomics \citep{hore2016tensor,wang2019three}. 
In these applications, oftentimes the tensor exhibits a \emph{low-rank} structure, meaning that the data admits the form of a low-dimensional signal corrupted by random noise. Efficient recovery of this intrinsic low-rank signal not only facilitates various important machine learning tasks, e.g., clustering \citep{zhou2019tensor,luo2022tensor,zhou2023heteroskedastic}, but also spurs the development of important methodologies in the field of scientific computing \citep{khoromskij2011tensor,grasedyck2013literature}. 

In this paper, we study the problem of recovering a low-rank tensor from noisy observations of its entries. Such a problem is also known as Tensor Principal Component Analysis (Tensor PCA) \modif{in the literature}. 
To set the stage, we consider the single-spike model proposed by Richard and Montanari \citep{richard2014statistical}. Under this model, we observe a $k$-th order rank-one tensor corrupted by random noise:
\begin{equation}\label{eq:tensor_PCA}
	\bT = \lambda_n \vv^{\otimes k} + \WW.
\end{equation}
Here, $\lambda_n > 0$ is the signal-to-noise ratio that depends on the ambient dimension $n$, $\vv \in \S^{n-1}$ is a planted signal that lies on the $n$-dimensional unit sphere, and $\WW \in (\R^n)^{\otimes k}$ stands for the random noise that has i.i.d. standard Gaussian entries and is independent of the signal. 
We denote by $k \in \NN_+$ the order of the tensor. The special case $k = 2$ has been well studied by statisticians, where model \eqref{eq:tensor_PCA} reduces to the spiked \modif{Wigner} model \citep{johnstone2001distribution}. In particular, detection and estimation problems under the spiked matrix model have been extensively investigated under various contexts \citep{johnstone2001distribution,baik2005phase,benaych2012singular,lelarge2019fundamental,montanari2022fundamental}, and numerous computationally efficient algorithms have been proposed to recover the signal \citep{journee2010generalized,ma2013sparse,deshp2016sparse,montanari2021estimation}. The main focus of the present paper will be tensors of order $k \ge 3$.

Comparing to its matrix counterpart, tensor problems with order $k \geq 3$ are, in many scenarios, much more challenging. For instance, the spectral decomposition of a matrix or matrix PCA can be performed efficiently using polynomial-time algorithms, while tensor decomposition and tensor PCA are known to be NP-hard \citep{hillar2013most}. Despite the NP-hardness of tensor PCA, researchers have designed scalable algorithms that are capable of recovering the planted signal, if one additionally assumes that the data are random and follow natural distributional assumptions. This random data perspective not only simplifies analysis in many scenarios, but also offers valuable insights on the computational complexity and estimation accuracy of the proposed methodologies from an average-case point of view. 
Exemplary algorithms that come with average-case theoretical guarantees include iterative algorithms \citep{ma2013sparse,anandkumar2017homotopy,zhang2018tensor,ben2020algorithmic,han2022optimal,huang2022power}, sum-of-squares (SOS) algorithms \citep{hopkins2015tensor,hopkins2016fast,potechin2017exact,kim2017community}, and spectral algorithms \citep{montanari2018spectral,xia2019polynomial,cai2019nonconvex}. 

Another striking feature of the tensor PCA model \eqref{eq:tensor_PCA} is that it exhibits the so-called \emph{computational-to-statistical gap}, meaning that there exists regime of signal-to-noise ratios within which it is information theoretically possible to recover the planted signal, while no polynomial-time algorithms are known \citep{berthet2013optimal,arous2019landscape,gamarnik2021overlap,dudeja2021statistical}.  
A sequence of works have 
established that the information-theoretic threshold of tensor PCA \eqref{eq:tensor_PCA} is of order $\Theta(\sqrt{n})$ \citep{richard2014statistical,lesieur2017statistical,chen2019phase}.   
On the other hand, the algorithmic threshold---the minimal signal-to-noise ratio above which recovery is efficiently achievable---is conjectured to be $\Theta (n^{k / 4})$ \citep{richard2014statistical}.\footnote{Under the scaling of \cite{richard2014statistical}, these two thresholds are $\Theta (1)$ (constant order) and $\Theta (n^{(k - 2) / 4})$, respectively.} This critical threshold, in fact, has been achieved by various algorithms which originate from different ideas \citep{hopkins2015tensor,anandkumar2017homotopy,biroli2020iron}.

In spite of the strong theoretical guarantees achieved by strategically crafted algorithms, in practice, it is often preferable to resort to simple iterative algorithms.
Among them, \emph{tensor power iteration} has been extensively applied to solving a large number of tensor problems, e.g., tensor decomposition and tensor PCA \citep{kolda2011shifted,richard2014statistical,wang2016online,an2017analyzing,huang2022power,wu2023lower}. 
For tensor PCA, doing power iteration is equivalent to running projected gradient ascent on a
non-convex polynomial objective function with infinite step size. 
Towards understanding the dynamics of this algorithm, in \cite{richard2014statistical} the authors proved that tensor power iteration with random initialization converges rapidly to the true signal provided that $\lambda_n  \gtrsim n^{k / 2}$. They also employed a heuristic argument to suggest that the necessary and sufficient condition for convergence is actually $\lambda_n \gtrsim n^{(k - 1) / 2}$.\footnote{Again, under their scaling, these two conditions should be  $\lambda_n  \gtrsim n^{(k - 1) / 2}$ and $\lambda_n  \gtrsim n^{(k - 2) / 2}$, respectively.} Later, a more refined analysis was carried out by \cite{huang2022power}. The authors showed that tensor power iteration with a constant number of iterates succeeds when $\lambda_n \gtrsim n^{(k - 1) / 2 + \eps}$, and fails when $\lambda_n \lesssim n^{(k - 1) / 2 - \eps}$ for an arbitrarily small positive constant $\eps$, thus partially confirming the $n^{(k - 1) / 2}$ threshold conjectured by \cite{richard2014statistical}. However, their analysis is restricted to a fixed number of power iterates, and therefore fails to capture the dynamics of tensor power iteration when the number of iterations grows with the input dimension. Further, they did not characterize the asymptotic behavior of tensor power iteration when $\lambda_n \asymp n^{(k - 1) / 2}$. Therefore, a complete picture is still lacking. As a side note, past works have also considered 
tensor power iteration with a warm start depending on some extra side information \citep{richard2014statistical,huang2022power}. However, how to obtain such initialization in practice remains elusive. In this paper, we will only consider tensor power iteration with random initialization independent of the data.

\subsection{Our contribution}

This paper is devoted to establishing a more comprehensive picture of tensor power iteration starting from a random initialization. Our contributions are summarized below. 
\paragraph{Algorithmic threshold.} 
First, we give a partial answer to the open problem in \cite{richard2014statistical}.
To be concrete, our results imply that tensor power iteration with a random initialization provably converges to the planted signal in polynomially many steps, requiring only $\lambda_n \gtrsim n^{(k - 1) / 2} (\log n)^{-C}$ for some positive constant $C$ that depends only on $k$.
Recall that the conjectured threshold in \cite{richard2014statistical} is $\Theta( n^{(k - 1) / 2})$, our conclusion actually shows that the true phase transition for power iteration occurs at a slightly lower signal-to-noise ratio than the one conjectured in \cite{richard2014statistical}. In order to establish such a result, we introduce the concept of \emph{alignment} to measure the correlation between the iterates obtained from power iteration and the planted signal, and show that the evolution of this alignment can be well approximated by a low-dimensional \emph{polynomial recurrence process}. We then conduct a precise analysis on the dynamics of this process to establish the convergence of power iteration for tensor PCA.

\paragraph{Number of iterations required for convergence.} Second, we present a sharp characterization of the number of iterations required for convergence, for $\lambda_n$ ranging from $\ll n^{(k - 1) / 2}$ to $\gg n^{(k - 1) / 2} $. To be precise, when $\gamma_n := n^{-(k - 1) / 2} \lambda_n \in [c, \modif{n^{o(1)}}]$ for $c > 0$ an arbitrary constant that does not depend on $n$, we show that $(1 + o_n(1)) \log_{k - 1} (\log_{k - 1} n / \log_{k - 1} \gamma_n)$ iterations are both necessary and sufficient for convergence to occur. In a different weak signal regime where $(\log n)^{-(k - 2) / 2} \ll \gamma_n \ll 1$, we establish upper and lower bounds on the number of iterations that are of the same order of magnitude on a logarithmic scale (see \cref{thm:conv_power_iter} for a formal statement). 
\modif{Our analysis is similar to that of \cite{huang2022power}. However, the technical treatment in \cite{huang2022power} is only able to analyze the dynamics of tensor power iteration up to a constant number of steps, which is not sufficient if the initialization is random. To overcome this difficulty, in this work we develop a novel Gaussian conditioning lemma (\cref{lemma:Gaussian-conditioning}), which allows us to accurately analyze the dynamics of tensor power iteration up to polynomially many steps.} To the best of our knowledge, this is the first result that studies the dynamics of tensor power iteration beyond a constant number of iterations under the setting of model \eqref{eq:tensor_PCA}.

\paragraph{Stopping criterion.} We also propose a stopping criterion that allows us to decide when to terminate the iteration in practice. Our proposal is simple, effective, and comes with rigorous theoretical guarantee. To summarize, the proposed stopping criterion finds an iterate that with high probability correlates well with the hidden spike. Besides, if we implement the proposed stopping rule, then the actual number of power iteration we implement matches well with the upper and lower bounds we have established, emphasizing both accuracy and efficiency of our proposal.  

\paragraph{Gaussian conditioning beyond constant steps.} The tool that we employ to establish the above results is based on the Gaussian conditioning technique, which has been widely applied to analyze the Approximate Message Passing (AMP) algorithm \citep{bayati2011dynamics} as well as many other iterative algorithms. Prior art along this line of research mostly studies only a constant number of iterations. 
Encouragingly, 
recent years have witnessed significant progress towards generalizing such Gaussian conditioning type analysis to accommodate settings that allow the number of iterations to grow simultaneously with the input dimension \citep{rush2018finite,li2022non,li2023approximate,wu2023lower}. 
Our work contributes to this active field of research by establishing the first result of this kind under the tensor PCA model \eqref{eq:tensor_PCA}. 
From a technical perspective, we believe our results not only push forward
the development of AMP theory, but also enrich the toolbox to analyze general iterative
algorithms.


\paragraph{Future directions.}
\modif{One interesting future direction is to generalize our convergence analysis to sub-Gaussian tensors, which requires developing non-asymptotic AMP-type analysis for sub-Gaussian random ensembles beyond a constant number of iterations, and is highly non-trivial. 
Progress in this direction is made only recently by \cite{jones2024diagram}, and their main results are based on some complicated combinatorial arguments. }

\modif{Another fascinating future direction would be to extend our main results to the multi-rank case, namely, the observed tensor $\bT$ is a rank-$r$ tensor perturbed by random Gaussian noise for some $r > 1$. Under this setting, we believe that the same Gaussian conditioning technique can be employed to analyze the (properly defined) alignment between the power iterates and the planted rank-$r$ signal, and we leave that for future work. }

\subsection{Organization}

The rest of this paper is organized as follows. \cref{sec:main} formulates the framework and gives our main result. We present the proof of the main theorem in \cref{sec:proof-sketch}, while deferring the proofs of several auxiliary lemmas  to appendices. In \cref{sec:experiments} we report numerical experiments that support our theorems.


\subsection{Notation}\label{sec:notation}
Throughout the proof, with a slight abuse of notation, we use letters $c, C$ to represent various
constants (which can only depend on the tensor order $k$), whose values might not necessarily be the same in each occurrence.
For a matrix $\bS \in \R^{d \times m}$, we denote by $\Pi_{\bS} \in \R^{d \times d}$ the projection matrix onto the column space of $\bS$, and let $\Pi_{\bS}^{\perp} := \id_d - \Pi_{\bS}$. For $n \in \NN_+$, we define $[n] = \left\{1,2,\cdots, n \right\}$.
For two sequences of positive numbers $\{a_n\}_{n \in \NN_+}$ and $\{b_n\}_{n \in \NN_+}$, we say $a_n \gtrsim b_n$ if there exists a positive constant $c$, such that $a_n \geq c\,b_n$, we say $a_n = o_n(b_n)$ if $a_n / b_n \to 0$ as $n \to \infty$, and we say $a_n \gg b_n$ if $a_n / b_n \to \infty$ as $n \to \infty$.  


\section{Main results} 
\label{sec:main}

We summarize in this section our main results. We first give a formal definition of tensor power iteration from a random initialization. 
Then, we present our main theorem in which we determine the regime of convergence and characterize the number of iterations required.    
We also give a stopping criterion that determines when to terminate the power iteration. 

\subsection{Tensor power iteration}

We denote by $\vv^0 = \tilde{\vv}^0 \sim \Unif(\S^{n - 1})$ the random initialization that is independent of $\bT$. Tensor power iteration initialized at $\vv^0$ is defined recursively as follows:
\begin{align}
	\label{eq:TPI}
	\begin{split}
		\vv^{t+1} =\, & \bT [ ( \tilde{\vv}^t )^{\otimes (k-1)} ] = \lambda_n \langle \vv, \tilde{\vv}^t \rangle^{k - 1} \vv + \WW [ ( \tilde{\vv}^t )^{\otimes (k-1)} ], \\
		\tilde{\vv}^{t + 1} =\, & \frac{\vv^{t + 1}}{\|\vv^{t + 1}\|_2},
	\end{split}
\end{align}
where $\WW [ ( \tilde{\vv}^t )^{\otimes (k-1)} ]$ is an $n$-dimensional vector whose $i$-th entry is $\langle \WW, \ee_i \, \otimes \, (\tilde{\vv}^t)^{\otimes (k - 1)}  \rangle$.
Here, $\ee_i$ is an $n$-dimensional vector that has the $i$-th entry being one and all the rest being zero. 

As a side remark, iteration \eqref{eq:TPI} can be regarded as projected gradient ascent with infinite step size for the following constrained optimization problem: 
\begin{align*}
	\mbox{maximize }\langle \bT, \bsigma^{\otimes k} \rangle, \quad \mbox{subject to }\bsigma \in \S^{n - 1}. 
\end{align*}

\subsection{Convergence analysis}

Next, we study the number of iterations required for algorithm \eqref{eq:TPI} to converge. 
To this end,
we first define the convergence criterion. 
For any fixed positive constant $\delta$, let
\begin{align}
	\label{eq:T-delta-conv}
	T_{\delta}^{\rm conv} := \min \left\{t \in \NN_+: |\langle \tilde{\vv}^t, \vv \rangle| \geq 1 - \delta \right\}.
\end{align} 
Our main result provides upper and lower bounds on $T_{\delta}^{\rm conv}$, in a signal-to-noise ratio regime when we simultaneously have $\gamma_n \gg (\log n)^{-(k - 2) / 2}$ and $\gamma_n = \modif{n^{o(1)}}$.

\begin{thm}\label{thm:conv_power_iter}
	
	Recall that $\gamma_n = n^{-(k - 1) / 2}\lambda_n$. Assume $\gamma_n \gg (\log n)^{-(k - 2) / 2}$ and $\gamma_n = \modif{n^{o(1)}}$. Then for any fixed $\delta, \eta > 0$, with probability $1 - o_n(1)$ we have 
	%
	\begin{align}
		\label{eq:lower-upper}
		\begin{split}
			T_{\delta}^{\rm conv} \ge \, & \max \left\{ \exp \left( \frac{1 - \eta}{2} \left( \frac{C_k}{\gamma_n} \right)^{2/(k-2)} \right), \ (1 - \eta) \log_{k-1} \frac{\log_{k-1} n}{\max\{\log_{k-1} \gamma_n, 1\}} \right\}, \\
			T_{\delta}^{\rm conv} \le \, & \exp \left( \frac{1 + \eta}{2} \left( \frac{1}{\gamma_n} \right)^{2/(k-2)} \right) + (1 + \eta) \log_{k-1} \frac{\log_{k-1} n}{\max\{\log_{k-1} \gamma_n, 1\}},
		\end{split}
	\end{align}
	where $C_k = (k-2)^{k-2}/(k-1)^{k-1}$. 
\end{thm}

\begin{rem}
	When $\gamma_n \gtrsim 1$, \cref{thm:conv_power_iter} implies that 
 \begin{equation}
     \frac{T_{\delta}^{\rm conv}}{\log_{k-1} \frac{\log_{k-1} n}{\max\{\log_{k-1} \gamma_n, 1\}}} \overset{\P}{\to} 1,
 \end{equation} 
 which gives a sharp characeterization of the number of steps required for convergence. On the other hand, as $\gamma_n$ drops below the constant level, $T_{\delta}^{\rm conv}$ grows drastically, but is still polynomial in $n$ provided that $\gamma_n \gg (\log n)^{-(k - 2) / 2}$. In addition, based on the upper and lower bounds presented in the theorem, we conjecture that the time complexity of tensor power iteration is super-polynomial when $\gamma_n \ll (\log n)^{-(k - 2) / 2}$. \modif{However, due to the use of the Gaussian conditioning scheme (\cref{lemma:Gaussian-conditioning}), our current approach can only accurately analyze the dynamics of tensor power iteration up to polynomially many steps. Therefore, justifying this conjecture would require the development of new theoretical tools, which we leave for future work. Further, our \cref{thm:conv_power_iter} assumes $\gamma_n = n^{o(1)}$, since \cite{huang2022power} already proved that a constant number of power iterations is sufficient to recover the true signal when $\gamma_n \gtrsim n^{c}$, for any constant $c > 0$ that does not depend on $n$. }
\end{rem}

We present the proof of \cref{thm:conv_power_iter} in \cref{sec:proof-sketch}, with proofs of auxiliary lemmas and propositions deferred to the appendices.

\subsection{Stopping criterion}
\label{sec:stopping-rule}

\cref{thm:conv_power_iter} gives lower and upper bounds on the number of iterations required for convergence. However, the theorem falls short of providing practical guidance regarding when should we terminate the power iteration, as we do not assume we know any prior information about the signal-to-noise ratio $\lambda_n$, or equivalently $\gamma_n$. 

To tackle this issue, we propose in this section a simple while effective stopping criterion that with high probability finds an iterate that aligns well with the hidden spike. In addition, we give upper and lower bounds on the actual number of power iterations we implement if we follow the proposed stopping criterion, which matches that introduced in \cref{thm:conv_power_iter}. 

To give a high level description, we propose to terminate the algorithm if we find any two consecutive iterates being moderately correlated with each other. To be precise, we define
\begin{align}
\label{eq:T-stop}
	T_{\sto} := \inf\left\{t \in \NN_+: \big|\langle \tilde\vv^{t - 2}, \tilde \vv^{t - 3} \rangle \big| \geq 1 / 2 \right\}. 
\end{align}
We shall output $\tilde{\vv}^{T_{\sto}}$ as an estimate of $\vv$. We give theoretical guarantee for our approach in the theorem below, which will be proved in \cref{sec:proof_stop}.
\begin{thm}
	\label{thm:stopping}
	We assume the conditions of \cref{thm:conv_power_iter}. Then, for any positive constant $\delta$, with probability $1 - o_n(1)$ we have $|\langle \tilde{\vv}^{T_{\sto}}, \vv \rangle| \geq 1 - \delta$. In addition, with high probability we still get the upper and lower bounds as indicated in \cref{eq:lower-upper} if we replace $T_{\delta}^{\rm{conv}}$ with $T_{\sto}$. 
\end{thm}


\section{Proof of \cref{thm:conv_power_iter}}
\label{sec:proof-sketch}

We present in this section the proof of \cref{thm:conv_power_iter}. The idea is to track the alignment between the iterates obtained from tensor power iteration and the planted signal.
Equipped with the Gaussian conditioning technique, we are able to  control the difference between this alignment and a scalar polynomial recurrence process that we define below, and prove that they are with high probability close to each other.  
This allows us to simplify our analysis by resorting to a reduction, and the remaining convergence analysis is conducted directly on this polynomial recurrence process.

\subsection{Reduction to the polynomial recurrence process} 

In this section we define the alignment with the true signal and show that it can be captured by a polynomial recurrence process.

For $t \in \NN_+$, we define 
\begin{align*}
	\alpha_t := \lambda_n \langle \vv, \tilde{\vv}^{t-1} \rangle^{k - 1} = \gamma_n \left( \sqrt{n} \langle \vv, \tilde{\vv}^{t-1} \rangle \right)^{k-1}.
\end{align*} 
The magnitude of $\alpha_t$ measures the level of alignment between the obtained iterates and the hidden spike.
We also define $\alpha_0 = 0$ for convenience. 
In the first iteration, the initial alignment is of the same order as $\gamma_n$, since by taking a random initialization roughly speaking we have $\modif{\vert} \langle \vv, \tilde{\vv}^{0} \rangle \modif{\vert} \asymp n^{-1/2}$. 
Throughout the paper, $\alpha_t$ will be the key quantity that characterizes the evolution of iteration \eqref{eq:TPI}.

As we have mentioned, the main goal of this section is to establish that $\{ \alpha_t \}_{t \in \NN_+}$ can be closely tracked by a one-dimensional discrete Markov process $\{ X_t \}_{t \in \NN}$, given by the following recurrence equation:   
\begin{align}\label{eq:polynomial-process}
	X_0 = 0, \quad \mbox{and} \ X_{t + 1} = \gamma_n(X_t + Z_t)^{k - 1}
	\ \mbox{for} \ t \geq 0,
\end{align}
where $\{ Z_t \}_{t \in \NN_+}$ is a sequence of i.i.d. standard Gaussian random variables.
To this end, we first develop the recurrence equation for the alignment sequence $\{ \alpha_t \}_{t \in \NN_+}$. 
Our derivation is based on the Gaussian conditioning technique, which has been widely applied to study the AMP algorithm \citep{bayati2011dynamics}.

\subsubsection*{\bf Decomposing tensor power iterates}

Next, we give a useful decomposition of the tensor power iterates.
By definition of tensor power iteration, we have
\begin{align*}
	\vv^{t + 1} =\, &  \lambda_n \langle \vv, \tilde{\vv}^t \rangle^{k - 1} \vv + \WW \left[ (\tilde{\vv}^t)^{\otimes (k - 1)} \right] \\
	=\, & \alpha_{t + 1} \vv + \WW \left[ (\tilde{\vv}^t)^{\otimes (k - 1)} \right],
\end{align*}
where we recall that $\alpha_{t + 1} = \lambda_n \langle \vv, \tilde{\vv}^t \rangle^{k - 1}$.

Before proceeding, we shall first introduce several concepts that are useful for our analysis. 
For $t \in \NN$, we let $\VV_t  \in \R^{n \times (t + 1)}$ be a matrix whose $i$-th column is $\vv^{i - 1}$. Based on the column space of $\VV_t$, we can decompose the vector $\vv^t$ as $\vv^t = \vv^t_{\perp} + \vv^t_{\parallel}$, where $\vv^t_{\perp} := \Pi_{\VV_{t - 1}}^{\perp} \vv^t$ and $\vv^t_{\parallel} :=  \Pi_{\VV_{t - 1}} \vv^t$. Analogously, we set $\tilde\vv^t_{\perp} = \vv^t_{\perp} / \|\vv^t\|_2$ and $\tilde{\vv}^t_{\parallel} = \vv^t_{\parallel} / \|\vv^t\|_2$ as their normalized versions. When the original vector is an all-zero one, we simply set its normalized version to be itself.

We immediately see that the vectors $\{ \tilde{\vv}^{i}_{\perp} / \Vert \tilde{\vv}^{i}_{\perp} \Vert_2 : 0 \le i \le t \}$ form an orthonormal basis of the linear space spanned by $\{ \tilde{\vv}^i: 0 \le i \le t \}$. As a consequence, $\tilde{\vv}^t$ admits the following decomposition:
\begin{equation*}
	\tilde{\vv}^t = \sum_{i=0}^{t} \frac{\langle \tilde{\vv}^t, \tilde{\vv}^{i}_{\perp} \rangle}{\|{\tilde{\vv}^{i}_{\perp}}\|_2} \cdot \frac{\tilde{\vv}^{i}_{\perp}}{\|{\tilde{\vv}^{i}_{\perp}\|}_2}.
\end{equation*}
For $(i_1, \cdots, i_{k-1}) \in \{0, 1, \cdots, t\}^{k-1}$,
we define
\begin{align*}
	& \beta^{(t)}_{i_1, i_2, \cdots, i_{k - 1}} := \prod_{j = 1}^{k - 1} \frac{\langle \tilde{\vv}^{i_j}_{\perp}, \tilde{\vv}^t \rangle}{\|\tilde{\vv}_{\perp}^{i_j}\|_2} \in \R, \\
	& \ww_{i_1, i_2, \cdots, i_{k - 1}} := \prod_{j = 1}^{k - 1} \|\tilde{\vv}_{\perp}^{i_j}\|_2^{-1} \WW \left[ \tilde{\vv}_{\perp}^{i_1} \otimes \tilde{\vv}_{\perp}^{i_2} \otimes \cdots \otimes  \tilde{\vv}_{\perp}^{i_{k - 1}}\right] \in \R^n. 
\end{align*}
%
As will become clear soon, with probability 1 over the randomness of the data generation process, it holds that $\|\tilde{\vv}_{\perp}^{t}\|_2 \neq 0$ for all $t = O(n^{1 / (2k - 2)})$. Therefore, $\ww_{i_1, i_2, \cdots, i_{k - 1}}$ and $\beta^{(t)}_{i_1, i_2, \cdots, i_{k - 1}}$ are almost surely well-defined.
With these definitions, we see that $\vv^{t + 1}$ can be decomposed as the sum of the following terms:  
\begin{align}\label{eq:expression-vt+1}
	\begin{split}
		\vv^{t + 1} =\, & \alpha_{t+1} \vv + \WW \left[ (\tilde{\vv}^t)^{\otimes (k - 1)} \right] \\
		=\, & \alpha_{t + 1} \vv + \sum_{(i_1, \cdots, i_{k - 1}) \in \HH_t} \beta^{(t)}_{i_1, i_2, \cdots, i_{k - 1}} \ww_{i_1, i_2, \cdots, i_{k - 1}} \\
		=\, & \alpha_{t + 1} \vv + \sum_{(i_1, \cdots, i_{k - 1}) \in \HH_{t - 1}} \beta^{(t)}_{i_1, i_2, \cdots, i_{k - 1}} \ww_{i_1, i_2, \cdots, i_{k - 1}} + \sqrt{1 - \|\tilde{\vv}^t_{\parallel}\|_2^{2k - 2}} \bg_{t + 1},
	\end{split}
\end{align}
where $\HH_t = \{ 0, 1, \cdots, t \}^{k-1}$, and
\begin{align}
	\label{eq:g-t+1}
	\bg_{t + 1} = \frac{1}{\sqrt{1 - \|\tilde{\vv}^t_{\parallel}\|_2^{2k - 2}}} \sum_{(i_1, i_2, \cdots, i_{k - 1}) \in \HH_t \backslash \HH_{t - 1}} \beta^{(t)}_{i_1, i_2, \cdots, i_{k - 1}} \ww_{i_1, i_2, \cdots, i_{k - 1}}.
\end{align}
Here, we make the convention that $\HH_{-1} = \emptyset$. 
In what follows, we will characterize the joint distribution of the vectors $\ww_{i_1, i_2, \cdots, i_{k - 1}}$ for all $(i_1, i_2, \cdots, i_{k - 1}) \in \HH_t$ via a Gaussian conditioning lemma, and derive the relationship between $\alpha_{t+1}$ and $\alpha_t$.

\subsubsection*{\bf The Gaussian conditioning lemma}

As an important ingredient of our conditioning analysis, we introduce the sigma-algebra $\cF_t$, which roughly speaking, is generated by the vectors associated with the first $t$ iterations. To be precise, we define
\begin{equation}\label{eq:def_F_t}
	\cF_t :=  \sigma \left(  \left\{ \ww_{i_1, i_2, \cdots, i_{k - 1}}: (i_1, i_2, \cdots, i_{k - 1}) \in \HH_{t - 1}  \right\} \cup \left\{\bg,  \vv^0, \vv^1, \cdots, \vv^t, \vv \right\}\right).
\end{equation}
Notice that $\cF_t \subseteq \cF_{t + 1}$. With these notations, we state our Gaussian conditioning lemma as follows.
\begin{lem}\label{lemma:Gaussian-conditioning}
	For all $t < n$ and all $(i_1, i_2, \cdots, i_{k - 1}) \in \HH_t \backslash \HH_{t - 1}$, it holds that  $\ww_{i_1, i_2, \cdots, i_{k - 1}} \perp \cF_t$. Furthermore, $\ww_{i_1, i_2, \cdots, i_{k - 1}} \iidsim \normal(\mathbf{0}, \id_n)$.  
\end{lem}
The proof of \cref{lemma:Gaussian-conditioning} is deferred to Appendix \ref{sec:proof-of-conditioning-lemma}. With the aid of this lemma, we know that $\bg_{t+1}$ is independent of the $\ww_{i_1, i_2, \cdots, i_{k - 1}}$'s for $(i_1, \cdots, i_{k-1}) \in \HH_{t-1}$. Further, since
\begin{align}
	\sum_{(i_1, i_2, \cdots, i_{k - 1}) \in \HH_t \backslash \HH_{t - 1}} \left( \beta^{(t)}_{i_1, i_2, \cdots, i_{k - 1}} \right)^2 = \, & 1 - \sum_{(i_1, i_2, \cdots, i_{k - 1}) \in \HH_{t - 1}} \left( \beta^{(t)}_{i_1, i_2, \cdots, i_{k - 1}} \right)^2 \\
	= \, & 1 - \prod_{j=1}^{k-1} \sum_{i_j = 0}^{t-1} \frac{\langle \tilde{\vv}^{i_j}_{\perp}, \tilde{\vv}^t \rangle^2 }{\|\tilde{\vv}_{\perp}^{i_j}\|_2^2} \\
	= \, & 1 - \left( \sum_{i = 0}^{t-1} \frac{\langle \tilde{\vv}^{i}_{\perp}, \tilde{\vv}^t \rangle^2 }{\|\tilde{\vv}_{\perp}^{i}\|_2^2} \right)^{k-1} \\
	= \, & 1 - \|\tilde{\vv}^t_{\parallel}\|_2^{2k - 2},
\end{align}
it follows that $\bg_{t+1} \sim \normal (\bzero, \id_n)$.

\subsubsection*{\bf Recurrence equation for the alignment}

With the aid of \cref{lemma:Gaussian-conditioning} and decomposition \eqref{eq:expression-vt+1}, we are ready to establish the recurrence equation for the alignment.
For notational convenience, we let
\begin{equation*}
	\hh_{t+1} := \sum_{(i_1, \cdots, i_{k - 1}) \in \HH_{t - 1}} \beta^{(t)}_{i_1, i_2, \cdots, i_{k - 1}} \ww_{i_1, i_2, \cdots, i_{k - 1}},
\end{equation*} 
where we make the convention that $\hh_0  = \bzero$.
It then follows that
\begin{align*}
	& \vv^{t+1} = \, \alpha_{t + 1} \vv + \hh_{t+1} + \sqrt{1 - \|\tilde{\vv}^t_{\parallel}\|_2^{2k - 2}} \bg_{t + 1}, \\
	\implies \, & \langle \vv^{t+1}, \vv \rangle = \alpha_{t+1} + \langle \hh_{t+1}, \vv \rangle + \sqrt{1 - \|\tilde{\vv}^t_{\parallel}\|_2^{2k - 2}} \,\, \langle \bg_{t + 1}, \vv \rangle, 
\end{align*}
which further implies
\begin{align*}
	\alpha_{t+2} = \, & \lambda_n \langle \tilde{\vv}^{t+1}, \vv \rangle^{k-1} = \gamma_n \cdot \left( \frac{\sqrt{n}}{\norm{\vv^{t+1}}_2} \right)^{k-1} \cdot \langle \vv^{t+1}, \vv \rangle^{k-1} \\
	= \, & \gamma_n \left( \frac{\sqrt{n}}{\norm{\vv^{t+1}}_2} \right)^{k-1} \left( \alpha_{t+1} + \langle \hh_{t+1}, \vv \rangle + \sqrt{1 - \|\tilde{\vv}^t_{\parallel}\|_2^{2k - 2}} \cdot \langle \bg_{t + 1}, \vv \rangle \right)^{k-1}.
\end{align*}
Decrementing the index by one, we get the following recurrence equation for the sequence $\{ \alpha_t \}_{t \in \NN}$:
\begin{equation}\label{eq:recurrence_alpha_t}
	\alpha_{t+1} = \, \gamma_n \left( \frac{\sqrt{n}}{\norm{\vv^{t}}_2} \right)^{k-1} \left( \alpha_{t} + \langle \hh_{t}, \vv \rangle + \sqrt{1 - \|\tilde{\vv}^{t-1}_{\parallel}\|_2^{2k - 2}} \cdot \langle \bg_{t}, \vv \rangle \right)^{k-1}.
\end{equation}
The remaining parts of this section will be devoted to the analysis of $\{ \alpha_t \}_{t \in \NN}$ based on the above equation.

\subsubsection*{\bf Controlling the error terms}

We then show that the recurrence equation~\eqref{eq:recurrence_alpha_t} can be viewed as a perturbed version of the polynomial process~\eqref{eq:polynomial-process} with small errors.
We start with defining some key quantities in Eq.~\eqref{eq:recurrence_alpha_t}:
\begin{equation}\label{eq:err_terms}
	\zeta_t := \left( \frac{\sqrt{n}}{\norm{\vv^{t}}_2} \right)^{k-1}, \quad b_t := \langle \hh_t, \vv \rangle, \quad c_t := \sqrt{1 - \|\tilde{\vv}^{t-1}_{\parallel}\|_2^{2k - 2}}.
\end{equation}
In the above display, we make the convention that $\tilde\vv_{\parallel}^{-1} = \mathbf{0}$.
At initialization, since $\zeta_0 = 1 $, $b_0 = 0$, and $c_0 = 1$, we know that the first iteration of \cref{eq:recurrence_alpha_t} is equivalent to
\begin{equation*}
	\alpha_{1} = \gamma_n \zeta_0 \left( \alpha_0 + b_0 + c_0 Z_0 \right)^{k-1} = \gamma_n \left( \alpha_0 + Z_0 \right)^{k-1},
\end{equation*}
where $Z_0 = \langle \bg_{0}, \vv \rangle \sim \normal (0, 1)$ is by \cref{lemma:Gaussian-conditioning} (recall that $\bg_0$ is defined in \cref{eq:g-t+1}). Similarly, by the law of large numbers, we know that $\zeta_1 = 1 + o_{n, \P} (1)$, $b_1 = o_{n, \P} (1)$ and $c_1 = 1 + o_{n, \P} (1)$, hence the next iteration has the following approximation:
\begin{equation*}
	\alpha_{2} = \gamma_n \zeta_1 \left( \alpha_1 + b_1 + c_1 Z_1 \right)^{k-1} \approx \gamma_n \left( \alpha_1 + Z_1 \right)^{k-1},
\end{equation*}
where $Z_1 = \langle \bg_{1}, \vv \rangle \sim \normal (0, 1)$ is independent of $\alpha_1$.
Indeed, we will show that the above approximation is valid up to polynomially many steps along the power iteration path until the alignment $\alpha_t$ reaches a certain threshold. To be precise, we establish the following lemma: 
\begin{lem}\label{lem:err_bound}
	For any fixed $\veps \in (1/4, 1/2)$, define the stopping time
	\begin{align}
		\label{eq:T-eps}
		T_{\veps} := \min \left\{t \in \NN_+: |\alpha_t| \geq n^{\varepsilon} \right\}.
	\end{align}
	Then, there exists an absolute constant $C > 0$, such that with probability no less than $1 - \exp(-C \sqrt{n})$, the following happens: For all $t < \min(T_{\veps}, n^{1/2(k-1)})$,
	\begin{equation}
		\zeta_t \in [1 - n^{-1/6}, 1 + n^{-1/6}], \ \vert b_t \vert \le Cn^{1/4 + (k-1) (\veps - 1/2)}, \ \vert c_t - 1 \vert \le C n^{2(k-1) (\veps - 1 / 2)}.
	\end{equation}
\end{lem}
We defer the proof of \cref{lem:err_bound} to Appendix \ref{sec:proof-lem:err_bound}.
As a direct corollary of \cref{lem:err_bound}, we immediately obtain the following proposition:
\begin{prop}\label{prop:summary_alignment}
	Under the same setting as in \cref{lem:err_bound}, and let $\veps = \veps_k = (6k - 11)/12 (k-1) $, which satisfies $\veps_k \in (1/4, 1/2)$ for all $k \ge 3$. Then, we have
	\begin{equation}\label{eq:alignment_process}
		\alpha_{t+1} = \gamma_n \zeta_t (\alpha_t + b_t + c_t Z_t)^{k-1}, \ \alpha_0 = 0.
	\end{equation}
	where $Z_t \sim \normal(0, 1)$ is independent of $(\zeta_t, \alpha_t, b_t, c_t)$. Further, with probability at least $1 - \exp(-C \sqrt{n})$, the following happens: For all $t < \min(T_{\veps}, n^{1/2(k-1)})$,
	\begin{equation}
		\zeta_t \in [1 - n^{-1/6}, 1 + n^{-1/6}], \ \vert b_t \vert \le Cn^{-1/6}, \ \vert c_t - 1 \vert \le C n^{-5/6}.
	\end{equation}
\end{prop}
The above proposition establishes that up to $\min(T_{\veps}, n^{1/2(k-1)})$ steps, the iteration of the alignment is closely tracked by that of the one-dimensional stochastic process defined in \cref{eq:polynomial-process}. In what follows, we show that the convergence of power iteration for tensor PCA can be precisely characterized by the stopping time $T_{\veps}$.
Before proceeding, we establish high probability upper and lower bounds on $T_{\eps}$, detailed by the following lemma. 

%
\begin{lem}\label{lem:conv_power_iter}
	Under the assumptions of \cref{thm:conv_power_iter}, and let $\veps = \veps_k = (6k - 11)/12 (k-1) $ as in the statement of \cref{prop:summary_alignment}.
	Then, for any sufficiently large $n \in \mathbb{N}$ and any $\eta \in (0, 1)$, with probability $1 - o_n (1)$ one has
	\begin{align}
		T_{\veps} \ge \, & \max \left\{ \exp \left( \frac{1 - \eta}{2} \left( \frac{C_k}{\gamma_n} \right)^{2/(k-2)} \right), \ (1 - \eta) \log_{k-1} \frac{\log_{k-1} n}{\max\{\log_{k-1} \gamma_n, 1\}} \right\}, \label{eq:Te-lower} \\
		T_{\veps} \le \, & \exp \left( \frac{1 + \eta}{2} \left( \frac{1}{\gamma_n} \right)^{2/(k-2)} \right) + (1 + \eta) \log_{k-1} \frac{\log_{k-1} n}{\max\{\log_{k-1} \gamma_n, 1\}}, \label{eq:Te-upper}
	\end{align}
	where $C_k = (k-2)^{k-2}/(k-1)^{k-1}$. 
\end{lem}
The proof of \cref{lem:conv_power_iter} is based on Proposition \ref{prop:summary_alignment}. 
For the compactness of presentation, we delay the proof of \cref{lem:conv_power_iter} to Appendix \ref{sec:proof-lem:conv_power_iter}.


%

%

%
%

%

%


\subsection{Convergence of tensor power iteration}

Recall that $T_{\delta}^{\rm conv}$ is defined in \cref{eq:T-delta-conv} and $T_{\eps}$ is defined in \cref{eq:T-eps}. 
For fixed positive constants $\delta$ and $\eps \in (1/4, 1/2)$, we see that for $n$ large enough we have $T_{\delta}^{\rm conv} \ge T_{\veps}$.
In this section, we also show that with high probability $T_{\delta}^{\rm conv} \le T_{\veps} + 1$.
Putting together these results, 
we conclude that if we can establish bounds on $T_{\veps}$, then this automatically gives bounds on $T_{\delta}^{\rm conv}$ as well.

Now let $t = T_{\veps}$. Naively we have $t-1 < T_{\veps}$ and $\vert \alpha_t \vert \ge n^{\veps}$. According to the power iteration equation, we obtain that
\begin{equation*}
	\vv^{t} = \, \alpha_{t} \vv + \hh_{t} + \sqrt{1 - \|\tilde{\vv}^{t-1}_{\parallel}\|_2^{2k - 2}} \bg_{t }.
\end{equation*}
Invoking \cref{lem:conv_power_iter}, we see that for a large enough $n$, with probability $1 - o_n(1)$ it holds that $T_{\veps} \le n^{1/2(k-1)}$. In this case we have $t-1 < \min(T_{\veps}, n^{1/2(k-1)})$. Re-examining the proof of concentration of $c_t$ in the proof of \cref{lem:err_bound}, we find that (note $\veps > 1 / 4$)
\begin{equation}
	\P \left( \|\tilde{\vv}^{t-1}_{\parallel}\|_2 \le C n^{\veps - 1/2} \right) \ge 1 - \exp(- C n^{2 \veps}) \ge 1 - \exp (- C \sqrt{n}),
\end{equation}
where $C$ is a positive constant,
and consequently 
\begin{equation*}
	\norm{\hh_t}_2 \le 2 \sqrt{n} \cdot \|\tilde{\vv}^{t-1}_{\parallel}\|_2^{k - 1} \le C n^{1/12}
\end{equation*}
if we choose $\veps = \veps_k = (6k - 11)/12 (k-1)$ as per \cref{prop:summary_alignment}. Note that
\begin{align*}
	\norm{\vv^t}_2 = \, & \norm{\alpha_{t} \vv + \hh_{t} + \sqrt{1 - \|\tilde{\vv}^{t-1}_{\parallel}\|_2^{2k - 2}} \bg_{t }}_2 \\
	\le \, & \vert \alpha_t \vert + \norm{\bg_t}_2 + Cn^{1/12} \le \vert \alpha_t \vert + C\sqrt{n}
\end{align*}
with probability at least $1 - \exp(-C \sqrt{n})$. Therefore,
\begin{equation}
	|\langle \tilde{\vv}^t, \vv \rangle| = \frac{\vert \langle \vv^t, \vv \rangle \vert}{\norm{\vv^t}_2} = \frac{\vert \alpha_t + b_t + c_t Z_t \vert}{\norm{\vv^t}_2} \ge \frac{\vert \alpha_t \vert - \vert b_t \vert - \vert c_t Z_t \vert}{\vert \alpha_t \vert + C\sqrt{n}}.
\end{equation}
Again according to the proof of \cref{lem:err_bound} and the choice of $\veps$ in \cref{prop:summary_alignment}, we know that $\vert b_t \vert \le Cn^{-1/6}$ and $|c_t| \leq 1 + Cn^{-5/6}$ with probability at least $1 - \exp(C\sqrt{n})$. Further since $Z_t \sim \normal(0, 1)$, it finally follows that with high probability $|\langle \tilde{\vv}^t, \vv \rangle| \ge C n^{\veps - 1/2}$, which leads to
\begin{equation}
	\left\vert \alpha_{t+1} \right\vert = \gamma_n \left\vert \sqrt{n} \langle \vv, \tilde{\vv}^{t} \rangle \right\vert^{k-1} > C n^{1/2 + 1/24}
\end{equation}
for sufficiently large $n$, provided that $\gamma_n \gg (\log n)^{-(k-2)/2}$. Consider the next iteration:
\begin{equation*}
	\vv^{t+1} = \, \alpha_{t+1} \vv + \hh_{t+1} + \sqrt{1 - \|\tilde{\vv}^{t}_{\parallel}\|_2^{2k - 2}} \bg_{t+1}.
\end{equation*}
Using standard concentration arguments, we know that
\begin{equation*}
	\norm{\hh_{t+1} + \sqrt{1 - \|\tilde{\vv}^{t}_{\parallel}\|_2^{2k - 2}} \bg_{t+1}}_2 \le Cn^{1/2} \le Cn^{-1/24} \vert \alpha_{t+1} \vert
\end{equation*}
with probability at least $1 - \exp(-Cn)$, which immediately implies that
\begin{equation}
	\label{eq:25}
	|\langle \tilde{\vv}^{t+1}, \vv \rangle| \ge 1 - Cn^{-1/24}.
\end{equation}
Therefore, $T_{\delta}^{\rm conv} \le t + 1 = T_{\veps} + 1$ with probability at least $1 - \exp(-C \sqrt{n})$. We summarize the main result of this section in the following lemma:
\begin{lem}\label{lem:T_eps_and_conv}
	Assume $\gamma_n \gg (\log n)^{-(k - 2) / 2}$ and $\gamma_n = \modif{n^{o(1)}}$. Then, with probability at least $1 - \exp(-C \sqrt{n})$, we have
	\begin{equation}
		|\langle \tilde{\vv}^{T_{\veps}+1}, \vv \rangle| \ge 1 - Cn^{-1/24}.
	\end{equation}
	Namely, tensor power iteration converges in one step after $\alpha_t$ reaches the level $n^{\veps}$.
\end{lem}
Combining the conclusions of \cref{lem:conv_power_iter} and \cref{lem:T_eps_and_conv} completes the proof of \cref{thm:conv_power_iter}.

\section{Numerical experiments}
\label{sec:experiments}

We present in this section simulations that support our theories.
For the simplicity of presentation, in the main text we only present experiments for several representative settings. 
We refer interested readers to Appendix \ref{sec:additional_exp} for simulation outcomes under more settings.

\subsection{Comparing alignment and the polynomial recurrence process}
\label{sec:exp-compare}

As demonstrated in \cref{sec:proof-sketch}, a key ingredient of our proof is to connect the tensor alignments $\{\alpha_t\}_{t \geq 0}$ with the polynomial recurrence process $\{X_t\}_{t \geq 0}$ defined in \cref{eq:polynomial-process}. Theoretical result that suggests their closeness has already been established in \cref{prop:summary_alignment}. We complement to this result in this section by providing empirical evidence.  

To set the stage, we choose $n = 200$, $k =3$, $\lambda_n = n^{(k - 1) / 2}$, and generate the tensor data according to \cref{eq:tensor_PCA}. We then run tensor power iteration with random initialization and compare the marginal distributions of $\alpha_t$ and $X_t$, for all $t \in \{1, 2, 3, 4\}$. We repeat this procedure 1000 times independently, and collect the realized values of $\alpha_t$ to form the corresponding empirical distributions. 
We also simulate the polynomial recurrence process $\{X_t\}_{t \geq 0}$ and obtain 1000 independent samples. 
We display the simulation outcomes in Figure \ref{fig:comparison}, which suggests that the marginal distributions already match well with a moderately large $n$.    
\begin{figure}[ht]
  \begin{minipage}[t]{.5\linewidth}
    \includegraphics[width=\linewidth]{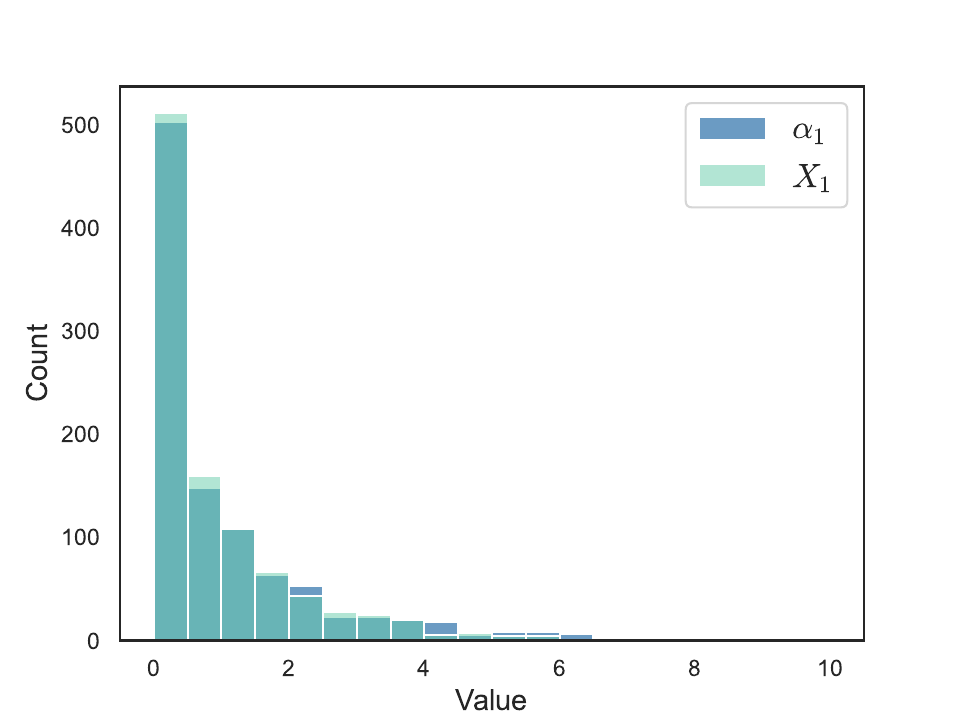}%
  \end{minipage}\hfil
  \begin{minipage}[t]{.5\linewidth}
    \includegraphics[width=\linewidth]{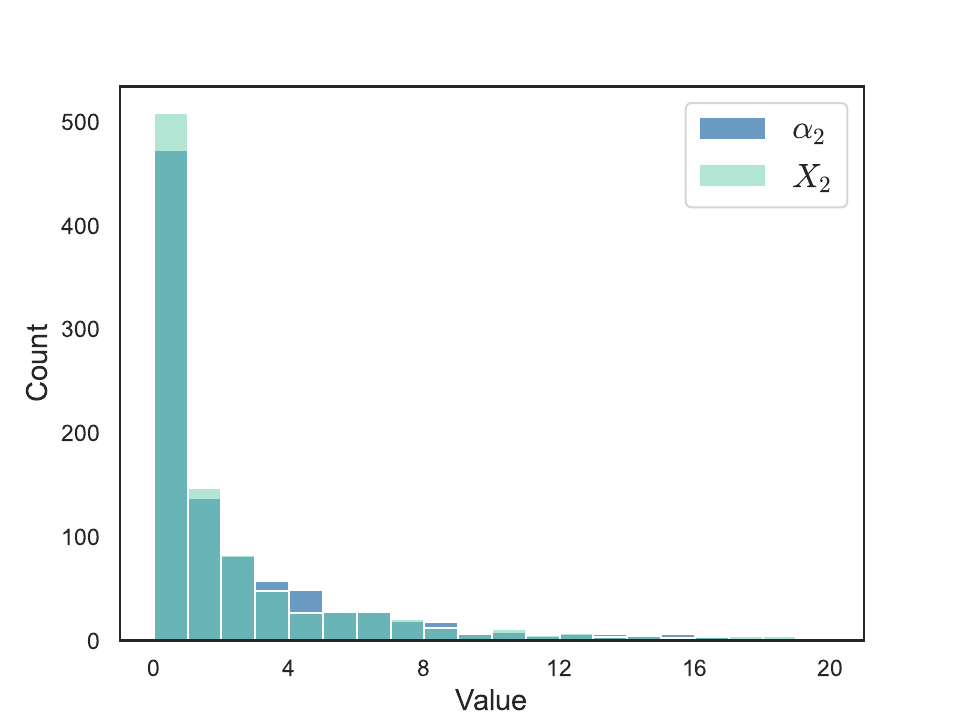}%
  \end{minipage}\hfil \\
  \begin{minipage}[t]{.5\linewidth}
    \includegraphics[width=\linewidth]{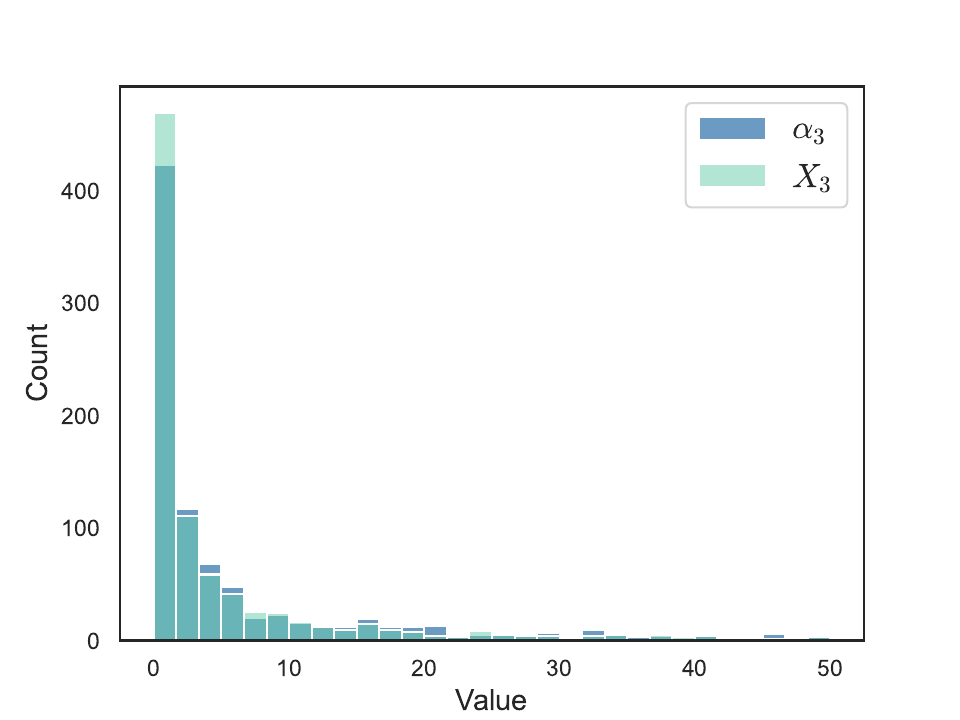}%
  \end{minipage}\hfil
  \begin{minipage}[t]{.5\linewidth}
    \includegraphics[width=\linewidth]{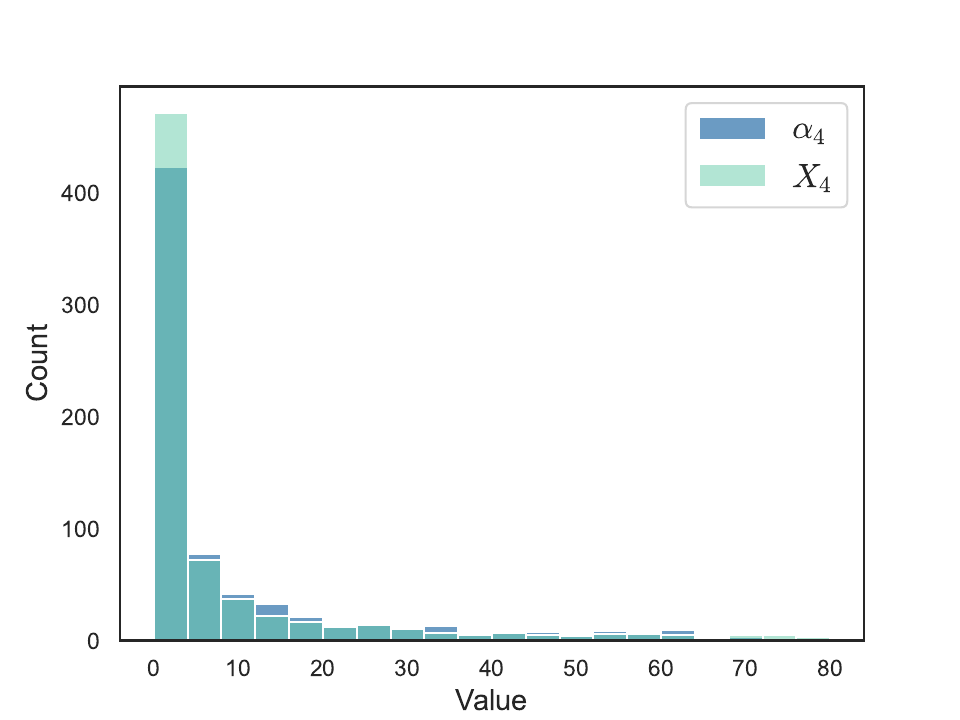}%
  \end{minipage}
  \caption{Comparison of the marginal distributions between $\alpha_t$ and $X_t$, for $t \in \{1, 2, 3, 4\}$. Here, we set $n = 200$, $k = 3$, $\lambda_n = n^{(k - 1) / 2}$, and run tensor power iteration from random initialization on independent datasets for 1000 times. \modif{Note that in this figure, the histograms for $\alpha_t$ and $X_t$ overlap a lot with each other (their overlapping regions are indicated by the third color), meaning that the marginal distributions of $\alpha_t$ and $X_t$ are indeed very close. }}
\label{fig:comparison}
\end{figure}

\subsection{Evolution of correlation}

\cref{thm:conv_power_iter} implies that as long as $\lambda_n \gtrsim 1$, tensor power iteration with random initialization will converge to the planted spike within $O(\log \log n)$ iterations. 
In this experiment, we provide numerical evidence that supports this claim. Throughout the experiment, we  set $\lambda_n = n^{(k - 1) / 2}$. In Figure \ref{fig:correlation}, we plot the evolution of correlation $|\langle \tilde{\vv}^t, \vv \rangle|$ as a function of the number of iterations $t$. From the figure, we see that the correlation rapidly increases from 0 to 1 as $t$ increases. Furthermore, the number of iterations required for convergence is nearly independent of the input dimension, suggesting the correctness of our claim. 

\begin{figure}[ht]
	\centering
	\includegraphics[width=\linewidth]{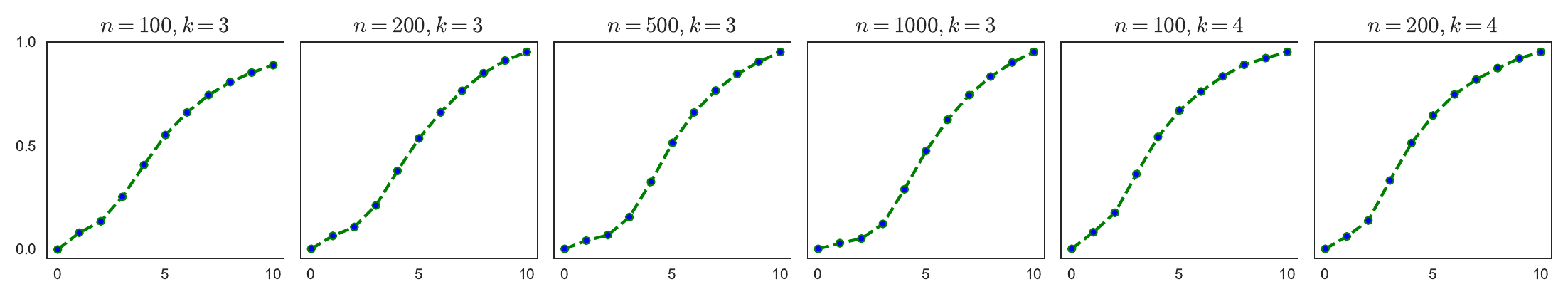}
	\caption{Evolution of correlation $|\langle \tilde{\vv}^t, \vv \rangle|$ as a function of the number of iterations $t$. Here, the $x$ axis represents the number of iterations ranging from 0 to 10, and the $y$ axis gives the level of correlation. We repeat the experiment independently for 1000 times for every combination of $(n, k)$, and compute the average correlation. }
	\label{fig:correlation}
\end{figure}

\subsection{Convergence probability}

Next, we investigate the probability of tensor power iteration with a random initialization converging to the planted spike. For this part we let $\lambda_n = n^{(k - 1) / 2}$, $k = 3$, and use different values of $n$. For each tensor realization, we run tensor power iteration from a random initialization for a sufficiently large number of iterations and check the convergence. 
For each $n \in \{25, 50, 100, 200, 400, 800\}$, we repeat this procedure independently for 1000 times and compute the empirical convergence probability. Here, we say an iterate $\tilde\vv^t$ converges to the true spike if and only if $|\langle \tilde \vv^t, \vv \rangle| > 0.99$. 
We plot such empirical probabilities in Figure \ref{fig:probability}. Inspecting the figure, we see that the $\gamma$-threshold above which power iteration with a random start achieves near probability one convergence decreases and approaches 0 as $n \to \infty$, once again suggesting the correctness of our main theorem. 

\begin{figure}[ht]
	\centering
	\includegraphics[width=0.5\linewidth]{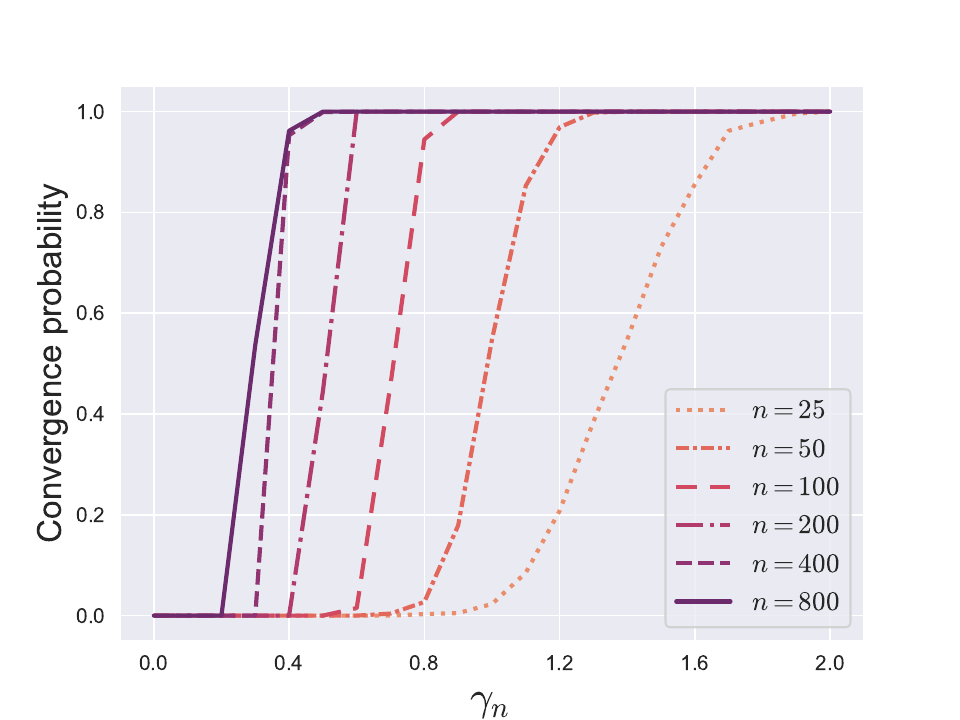}
	\caption{Probability of tensor power iteration with random initialization converging to the hidden spike. The $x$ axis stands for $\gamma_n$ and the $y$ axis gives the empirical convergence probability averaged over 1000 independent experiments. }
	\label{fig:probability}
\end{figure} 

\modif{\subsection{Stopping rule}}
\label{sec:stopping-rule-experiments}

\modif{Finally, we evaluate the performance of the stopping rule proposed in \cref{sec:stopping-rule}. 
In this experiment, we fix $\lambda_n = n^{(k - 1) / 2}$ and consider different combinations of $(n, k)$. For each configuration, we independently generate five tensors following model \eqref{eq:tensor_PCA}, and implement tensor power iteration from a random initialization. We then compute the correlation between the true spike and the iterates measured by $|\langle \tilde \vv^t, \vv \rangle|$, and plot the evolution of this correlation over the first 100 iterations. We also calculate $T_{\sto}$ using \cref{eq:T-stop}. We present the simulation results in Figure \ref{fig:stopping-rule}. From the figure, we see that the proposed stopping rule effectively terminates the algorithm at an early stage while simultaneously maintains high estimation accuracy. 
\begin{figure}[ht]
	\centering
	\includegraphics[width=0.9\linewidth]{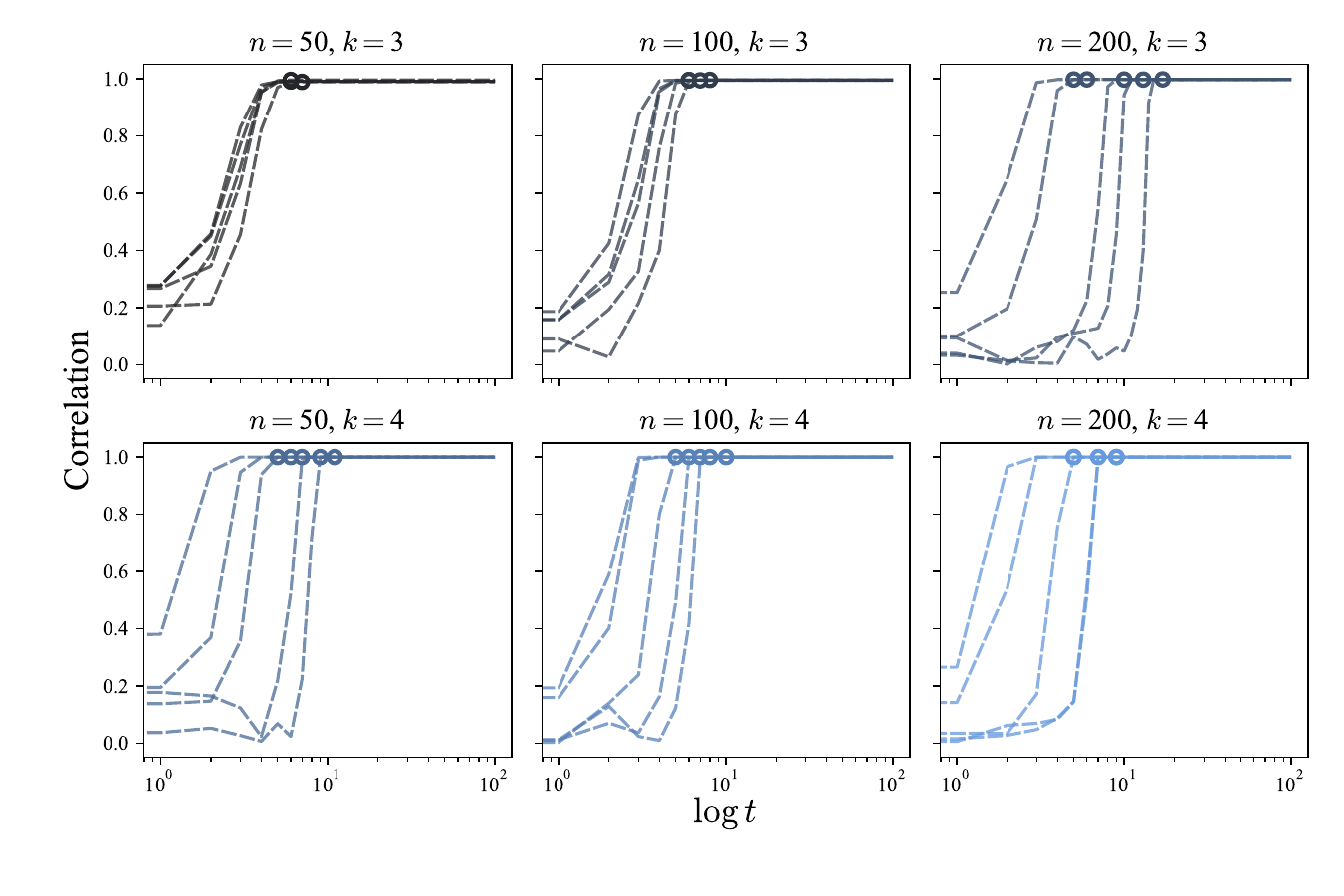}
	\caption{Illustration of the effectiveness of the stopping rule. The $x$ axis here is the logarithmic of the number of iterations, and the $y$ axis shows the correlation. We independently repeat the experiment 5 times for each setting, and record the correlation along the power iteration trajectory. Here, $T_{\sto}$ is computed using the power iteration iterates and is marked with a circle in the figure. }
	\label{fig:stopping-rule}
\end{figure} \newline
In Appendix \ref{sec:stopping-rule-extra}, we alter the value of the stopping threshold (which is $1/2$ in the current experiment) and still observe good performance, suggesting the proposed stopping rule is not sensitive to the choice of the stopping threshold. }


\acks{The authors would like to thank Tselil Schramm for suggesting this topic, as well as for many helpful conversations.
	The authors also thank Jiaoyang Huang for an insightful discussion.
	K.Z. was supported by the NSF through award DMS-2031883 and the Simons Foundation through Award 814639 for the Collaboration on the Theoretical Foundations of Deep Learning and by the NSF grant CCF-2006489.}


\newpage
\bibliography{TPI}

\newpage
\appendix

\section{Technical lemmas}

\begin{lem}[Tails of the normal distribution]\label{lemma:Gaussian-tail}
	Let $g \sim \normal(0,1)$. Then for all $t > 0$, we have 
	\begin{align*}
		\left( \frac{1}{t} - \frac{1}{t^3} \right) \cdot \frac{1}{\sqrt{2\pi}} e^{-t^2 / 2} \leq \P\left(g \geq t \right) \leq \frac{1}{t} \cdot \frac{1}{\sqrt{2\pi}} e^{-t^2 / 2}.
	\end{align*}
\end{lem}
\begin{lem}[Bernstein's inequality]\label{lemma:Bernstein}
	Let $X_1, \cdots, X_N$ be independent, mean zero, sub-exponential random variables. Then, for every $t \geq 0$, we have
	\begin{align*}
		\P\left( \left|\sum_{i = 1}^N X_i \right| \geq t \right) \leq 2 \exp \left[ -c \min \left( \frac{t^2}{\sum_{i = 1}^N \|X_i\|_{\Psi_1}^2}, \frac{t}{\max_{i} \|X_i\|_{\Psi_1}} \right) \right].
	\end{align*}
\end{lem}
\begin{lem}\label{lem:norm_concentration}
	Let $\xx_i \stackrel{\iid}{\sim} \normal (\bzero, \id_n)$ for $i = 1, \cdots, m$, where $m < \sqrt{n}$. Then, we have
	\begin{equation*}
		\P \left( \sup_{\boldsymbol{\alpha} \in \S^{m-1}} \left\vert \norm{\sum_{i=1}^{m} \alpha_i \xx_i}_2 - \sqrt{n} \right\vert \ge \veps \sqrt{n} \right) \le \exp(- C \veps^2 n).
	\end{equation*}
	for some absolute constant $C > 0$.
\end{lem}
\begin{proof}
	We use a covering argument. For $\boldsymbol{\alpha}, \boldsymbol{\beta} \in \S^{m-1}$ with $\norm{\boldsymbol{\alpha} - \boldsymbol{\beta}}_2 \le \veps$, we have
	\begin{equation*}
		\left\vert \norm{\sum_{i=1}^{m} \alpha_i \xx_i}_2 - \norm{\sum_{i=1}^{m} \beta_i \xx_i}_2 \right\vert \le \norm{\XX}_{\rm op} \norm{\balpha - \bbeta}_2 \le C \sqrt{n} \veps
	\end{equation*}
	with probability at least $1 - \exp(-Cn)$, where $\XX$ is the matrix whose $i$-th column is $\xx_i$.
	For any fixed $\balpha \in \S^{m-1}$, we have $\sum_{i=1}^{m} \alpha_i \xx_i \sim \normal (\bzero, \id_n)$. Therefore, for any $t \in (0, \sqrt{n})$, we have
	\begin{equation*}
		\P \left( \left\vert \norm{\sum_{i=1}^{m} \alpha_i \xx_i}_2 - \sqrt{n} \right\vert \ge t \right) \le 2 \exp \left( - \frac{t^2}{8} \right),
	\end{equation*}
	where the last inequality follows from concentration of sub-exponential random variables. Now, let $N_{\veps}^{m}$ be an $\veps$-net of $\S^{m-1}$, we thus obtain that
	\begin{align*}
		& \P \left( \sup_{\boldsymbol{\alpha} \in \S^{m-1}} \left\vert \norm{\sum_{i=1}^{m} \alpha_i \xx_i}_2 - \sqrt{n} \right\vert \ge 2 C \sqrt{n} \veps \right) \\
		\le \, & \exp (-Cn) + \P \left( \sup_{\balpha \in N_{\veps}^{m}} \left\vert \norm{\sum_{i=1}^{m} \alpha_i \xx_i}_2 - \sqrt{n} \right\vert \ge C \sqrt{n} \veps \right) \\
		\le \, & \exp(-Cn) + \left( \frac{C}{\veps} \right)^m \times 2 \exp \left( - \frac{C^2 n \veps^2}{8} \right) \le \exp (-C \veps^2 n).
	\end{align*}
	Replacing $\veps$ by $C \veps$ completes the proof.
\end{proof}

\section{Proof of \cref{lemma:Gaussian-conditioning}} \label{sec:proof-of-conditioning-lemma}

	Recall that $\cF_t$ is defined in \cref{eq:def_F_t}. We first show that
	\begin{equation*}
		\cF_t =  \sigma \left(  \left\{ \ww_{i_1, i_2, \cdots, i_{k - 1}}: (i_1, i_2, \cdots, i_{k - 1}) \in \HH_{t - 1}  \right\} \cup \left\{\bg, \vv \right\}\right),
	\end{equation*}
	which is equivalent to proving that $\vv^0, \vv^1, \cdots, \vv^t$ are measurable with respect to the $\sigma$-algebra on the right hand side of the above equation. Since $\vv^0 = \bg / \norm{\bg}_2$, we know that $\vv^0$ is measurable. Using decomposition \eqref{eq:expression-vt+1} with $t=0$, we know that $\vv^1$ is measurable as well. Repeating this argument yields that $\vv^2, \cdots, \vv^t$ are all measurable. This proves our claim.
	
	Next, we are in position to prove the lemma. To avoid heavy notation, we denote
	\begin{equation*}
		\uu_{i_1, \cdots, i_{k-1}} = \prod_{j=1}^{k-1} \norm{\tilde{\vv}_{\perp}^{i_j}}_2^{-1} \cdot \left( \tilde{\vv}_{\perp}^{i_1} \otimes \cdots \otimes \tilde{\vv}_{\perp}^{i_{k-1}} \right). 
	\end{equation*}
	We further define for $(i_1, \cdots, i_{k-1}) \in \NN^{k-1}$ and $\WW \in (\R^n)^{\otimes k}$ the rank-one tensor:
	\begin{equation}
		 {\rm P}_{(i_1, \cdots, i_{k-1})} \WW = \WW \left[ \uu_{i_1, \cdots, i_{k-1}} \right] \otimes \uu_{i_1, \cdots, i_{k-1}}.
	\end{equation}
	Straightforward calculation reveals that for $(i_1, \cdots, i_{k-1}), \ (j_1, \cdots, j_{k-1}) \in \NN^{k-1}$,
	\begin{align*}
		&{\rm P}_{(j_1, \cdots, j_{k-1})} {\rm P}_{(i_1, \cdots, i_{k-1})} \WW = {\rm P}_{(i_1, \cdots, i_{k-1})} \WW \left[ \uu_{j_1, \cdots, j_{k-1}} \right] \otimes \uu_{j_1, \cdots, j_{k-1}} \\
		=\, & \left\langle \uu_{i_1, \cdots, i_{k-1}}, \uu_{j_1, \cdots, j_{k-1}} \right\rangle \cdot \WW \left[ \uu_{i_1, \cdots, i_{k-1}} \right] \otimes \uu_{j_1, \cdots, j_{k-1}} \\
		\stackrel{(i)}{=}\, & \delta_{i_1 j_1} \cdots \delta_{i_{k-1} j_{k-1}} \cdot \WW \left[ \uu_{i_1, \cdots, i_{k-1}} \right] \otimes \uu_{j_1, \cdots, j_{k-1}} \\
		=\, & \delta_{i_1 j_1} \cdots \delta_{i_{k-1} j_{k-1}} \cdot {\rm P}_{(i_1, \cdots, i_{k-1})} \WW,
	\end{align*}
	where $(i)$ follows from the fact that the $\uu_{i_1, \cdots, i_{k-1}}$'s are mutually orthogonal, and $\delta$ represents the Kronecker delta: $\delta_{ij} = \bone \{ i = j \}$. Moreover, for any subset $S \subset \NN^{k-1}$, define
	\begin{equation}
		{\rm P}_S \WW = \sum_{(i_1, \cdots, i_{k-1}) \in S} {\rm P}_{(i_1, \cdots, i_{k-1})} \WW, \quad {\rm P}_S^{\perp} \WW = \WW - {\rm P}_S \WW. 
	\end{equation}
	One can prove using the previous calculation that for $S, \ T \subset \NN^{k-1}$,
	\begin{equation}
		{\rm P}_S {\rm P}_T = {\rm P}_T {\rm P}_S = {\rm P}_{S \cap T}.
	\end{equation}
	We will show that for all $t \in \NN$,
	\begin{equation}\label{eq:decompose_W}
		\WW = {\rm P}_{\HH_{t-1}} \WW + {\rm P}_{\HH_{t-1}}^{\perp} \tilde{\WW}_t,
	\end{equation}
	where $\tilde{\WW}_t \stackrel{d}{=} \WW$ and is independent of $\cF_t$. We prove \cref{eq:decompose_W} by induction. For $t=0$, it is obvious that we can simply choose $\tilde{\WW}_0 = \WW$, since $\WW$ is independent of $\cF_0 = \sigma(\{ \bg, \vv \})$. Now assume \eqref{eq:decompose_W} holds for $t=s$, namely we have
	\begin{equation}\label{eq:decompose_W_s}
		\WW = {\rm P}_{\HH_{s-1}} \WW + {\rm P}_{\HH_{s-1}}^{\perp} \tilde{\WW}_s, \quad \tilde{\WW}_s \stackrel{d}{=} \WW, \ \tilde{\WW}_s \perp \cF_s.
	\end{equation}
	Then, for $t=s+1$, let us define
	\begin{equation}
		\tilde{\WW}_{s+1} = {\rm P}_{\HH_s}^{\perp} \tilde{\WW}_s + {\rm P}_{\HH_s} \WW'_s,
	\end{equation}
	where $\WW'_s \stackrel{d}{=} \WW$ and is independent of everything else. We show that $\tilde{\WW}_{s+1}$ defined as above satisfies our requirement. To this end, first note that since $\HH_{s-1} \subset \HH_s$,
	\begin{align*}
		{\rm P}_{\HH_s}^{\perp} \WW =\, & {\rm P}_{\HH_s}^{\perp} {\rm P}_{\HH_{s-1}} \WW + {\rm P}_{\HH_s}^{\perp} {\rm P}_{\HH_{s-1}}^{\perp} \tilde{\WW}_s = {\rm P}_{\HH_s}^{\perp} \tilde{\WW}_s, \\
		{\rm P}_{\HH_s}^{\perp} \tilde{\WW}_{s+1} =\, & {\rm P}_{\HH_s}^{\perp} {\rm P}_{\HH_s}^{\perp} \tilde{\WW}_s + {\rm P}_{\HH_s}^{\perp} {\rm P}_{\HH_s} \WW'_s = {\rm P}_{\HH_s}^{\perp} \tilde{\WW}_s,
	\end{align*}
	which further implies that ${\rm P}_{\HH_s}^{\perp} \WW = {\rm P}_{\HH_s}^{\perp} \tilde{\WW}_{s+1}$. Hence, we deduce that
	\begin{equation}
		\WW = {\rm P}_{\HH_s} \WW + {\rm P}_{\HH_s}^{\perp} \WW = {\rm P}_{\HH_s} \WW + {\rm P}_{\HH_s}^{\perp} \tilde{\WW}_{s+1},
	\end{equation}
	i.e., \cref{eq:decompose_W} holds for $t = s+1$. Next, it suffices to show that $\tilde{\WW}_{s+1} \stackrel{d}{=} \WW$ and that $\tilde{\WW}_{s+1}$ is independent of $\cF_{s+1}$. Recall that we already proved
	\begin{equation}
		\cF_{s+1} =  \sigma \left(  \left\{ \ww_{i_1, i_2, \cdots, i_{k - 1}}: (i_1, i_2, \cdots, i_{k - 1}) \in \HH_{s}  \right\} \cup \left\{\bg, \vv \right\}\right).
	\end{equation}
	According to \cref{eq:decompose_W_s} and direct calculation, we know that for $(i_1, \cdots, i_{k-1}) \in \HH_s \backslash \HH_{s-1}$,
	\begin{equation*}
		\ww_{i_1, \cdots, i_{k-1}} = \WW \left[ \uu_{i_1, \cdots, i_{k-1}}  \right] = \tilde{\WW}_s \left[ \uu_{i_1, \cdots, i_{k-1}}  \right],
	\end{equation*}
	and that $\uu_{i_1, \cdots, i_{k-1}} \in \cF_s$. Therefore,
	\begin{equation}
		\cF_{s+1} = \sigma \left( \cF_s \cup \sigma \left\{ \tilde{\WW}_s \left[ \uu_{i_1, \cdots, i_{k-1}}  \right]: (i_1, \cdots, i_{k-1}) \in \HH_s \backslash \HH_{s-1} \right\} \right).
	\end{equation}
	Next we compute the conditional distribution of ${\rm P}_{\HH_s}^{\perp} \tilde{\WW}_s$ given $\cF_{s+1}$, which is equivalent to the law of ${\rm P}_{\HH_s}^{\perp} \tilde{\WW}_s$ conditioning on $\cF_s$ and the random variables $\tilde{\WW}_s \left[ \uu_{i_1, \cdots, i_{k-1}}  \right]$ for $(i_1, \cdots, i_{k-1}) \in \HH_s \backslash \HH_{s-1}$. Here, we can view the $(k-1)$-tensors $\uu_{i_1, \cdots, i_{k-1}}$ as fixed since they are measurable with respect to $\cF_s$. By definition, these $\uu_{i_1, \cdots, i_{k-1}}$'s are mutually orthogonal and belong to $\HH_s$, and we know from induction hypothesis that $\tilde{\WW}_s \vert \cF_s \stackrel{d}{=} \WW$. Applying Lemma 4.1 in \cite{huang2022power} yields that the conditional distribution of ${\rm P}_{\HH_s}^{\perp} \tilde{\WW}_s$ is equal to the law of ${\rm P}_{\HH_s}^{\perp} \WW''_s$, where $\WW''_s$ is an independent copy of $\tilde{\WW}_s$ which is further independent of $\cF_{s+1}$. As a consequence, it follows that
	\begin{equation*}
		\tilde{\WW}_{s+1} \vert \cF_{s+1} \stackrel{d}{=} {\rm P}_{\HH_{s}}^{\perp} \WW''_s + {\rm P}_{\HH_s} \WW'_s \vert \cF_{s+1} \stackrel{d}{=} \WW,	
	\end{equation*}
	i.e., $\tilde{\WW}_{s+1} \stackrel{d}{=} \WW$ is independent of $\cF_{s+1}$. This completes the induction. Now, using \cref{eq:decompose_W}, we know that for $(i_1, \cdots, i_{k-1}) \in \HH_t \backslash \HH_{t-1}$, $\ww_{i_1, \cdots, i_{k-1}} = \tilde{\WW}_t [\uu_{i_1, \cdots, i_{k-1}}]$, where $\tilde{\WW}_t \perp \cF_t$, $\{ \uu_{i_1, \cdots, i_{k-1}} \}$ is an orthonormal set that is measurable with respect to $\cF_t$. It then follows immediately that $\ww_{i_1, \cdots, i_{k-1}} \iidsim \normal (\bzero, \id_n)$ are independent of $\cF_t$ for $(i_1, \cdots, i_{k-1}) \in \HH_t \backslash \HH_{t-1}$. This completes the proof of \cref{lemma:Gaussian-conditioning}.
	
\section{Proof of \cref{lem:err_bound}}
\label{sec:proof-lem:err_bound}

The entire argument is divided into three parts:
	\paragraph{Concentration of $\zeta_t$.} To this end, we need to estimate $\norm{\vv^t}_2$. From Eq.~\eqref{eq:expression-vt+1}, we know that
	\begin{equation}
		\vv^{t} =\, \alpha_{t} \vv + \sum_{(i_1, \cdots, i_{k - 1}) \in \HH_{t-1}} \beta^{(t-1)}_{i_1, i_2, \cdots, i_{k - 1}} \ww_{i_1, i_2, \cdots, i_{k - 1}}.
	\end{equation}
	Since $\ww_{i_1, i_2, \cdots, i_{k - 1}} \stackrel{\iid}{\sim} \normal (\bzero, \id_n)$, and $\sum_{(i_1, \cdots, i_{k - 1}) \in \HH_{t-1}} (\beta^{(t-1)}_{i_1, i_2, \cdots, i_{k - 1}})^2 = 1$, we deduce from \cref{lem:norm_concentration} that with probability at least $1 - \exp(-C \eta^2 n)$, for all $t < n^{1/2(k-1)}$,
	\begin{equation}\label{eq:vt_norm_estimate}
		(1 - \eta) \sqrt{n} \le \norm{\sum_{(i_1, \cdots, i_{k - 1}) \in \HH_{t-1}} \beta^{(t-1)}_{i_1, i_2, \cdots, i_{k - 1}} \ww_{i_1, i_2, \cdots, i_{k - 1}}}_2 \le (1 + \eta) \sqrt{n},
	\end{equation}
	where $\eta > 0$ is a small constant. Further, as long as $t < T_{\veps}$, by definition we have
	\begin{equation*}
		\norm{\alpha_t \vv}_2 = \left\vert \alpha_t \right\vert < n^{\veps},
	\end{equation*}
	which leads to
	\begin{equation}
		(1 - 2 \eta) \sqrt{n} \le (1 - \eta) \sqrt{n} - n^{\veps} \le \norm{\vv^t}_2 \le (1 + \eta) \sqrt{n} + n^\veps \le (1 + 2 \eta) \sqrt{n}.
	\end{equation}
	To summarize, we conclude that with probability at least $1 - \exp(-C \eta^2 n)$, the following happens:
	\begin{equation}\label{eq:vt_norm_bound}
		\text{For all} \ t < \min ( T_{\veps}, n^{1/2(k-1)} ), \
		\norm{\vv^t}_2 \in [(1 - \eta) \sqrt{n}, (1 + \eta) \sqrt{n}].
	\end{equation}
	Recall that $\zeta_t = (\sqrt{n} / \norm{\vv^t}_2)^{k-1}$, the above bound also implies that $\zeta_t \in [1 - 2(k-1) \eta, 1 + 2 (k-1) \eta]$ for sufficiently small $\eta > 0$. In fact, from the proof of \cref{lem:norm_concentration} we know that $\eta$ can be chosen as $n^{-1/6}/2(k-1)$, so that
	\begin{equation}
		\P \left( \zeta_t \in [1 - n^{-1/6}, 1 + n^{-1/6}] \ \text{for all} \ t < \min ( T_{\veps}, n^{1/2(k-1)} ) \right) \ge 1 - \exp(-C n^{2/3}).
	\end{equation}
	
	\paragraph{Concentration of $c_t$.} To show that $c_t$ is close to $1$ with high probability, we need to control $\|\tilde{\vv}^{t-1}_{\parallel}\|_2$. By definition, we have
	\begin{equation*}
		\norm{\tilde{\vv}^{t-1}_{\parallel}}_2 = \frac{\norm{\vv^{t-1}_{\parallel}}_2}{\norm{\vv^{t-1}}_2} \le \frac{1}{1 - \eta} \frac{\norm{\vv^{t-1}_{\parallel}}_2}{\sqrt{n}}
	\end{equation*}
	for all $t < \min ( T_{\veps}, n^{1/2(k-1)} ) + 1$ with probability at least $1 - \exp(-C \eta^2 n)$. It then suffices to establish an upper bound for $\norm{\vv^{t}_{\parallel}}_2$ with $t < \min ( T_{\veps}, n^{1/2(k-1)} )$. Note that
	\begin{align*}
		\vv^{t}_{\parallel} = \, & \Pi_{\VV^{t-1}} \vv^t = \Pi_{\VV^{t-1}} \left( \alpha_{t} \vv + \hh_{t} + \sqrt{1 - \|\tilde{\vv}^{t-1}_{\parallel}\|_2^{2k - 2}} \bg_{t} \right) \\
		= \, & \alpha_{t} \cdot \Pi_{\VV^{t-1}} \vv + \Pi_{\VV^{t-1}} \hh_{t} + \sqrt{1 - \|\tilde{\vv}^{t-1}_{\parallel}\|_2^{2k - 2}} \cdot \Pi_{\VV^{t-1}} \bg_{t},
	\end{align*}
	which further implies that
	\begin{equation}
		\norm{\vv^{t}_{\parallel}}_2 \le \alpha_t + \norm{\hh_t}_2 + \sqrt{1 - \|\tilde{\vv}^{t-1}_{\parallel}\|_2^{2k - 2}} \cdot \norm{\Pi_{\VV^{t-1}} \bg_{t}}_2.
	\end{equation}
	By definition of $\hh_t$, we know that $\hh_t$ can be written as
	\begin{equation*}
		\hh_t = \|\tilde{\vv}^{t-1}_{\parallel}\|_2^{k - 1} \cdot \sum_{j=1}^{(t-1)^{k-1}} \beta_j \ww_j,
	\end{equation*}
	where $\sum_{j=1}^{(t-1)^{k-1}} \beta_j^2 = 1$ and $\ww_j \sim_{\iid} \normal(\bzero, \id_n)$. Applying again \cref{lem:norm_concentration}, it follows that
	\begin{equation}
		\norm{\hh_t}_2 \le 2 \sqrt{n} \cdot \|\tilde{\vv}^{t-1}_{\parallel}\|_2^{k - 1} \ \text{for all} \ t < n^{1 / 2(k-1)}
	\end{equation}
	with probability at least $1 - \exp(-Cn)$. Further, since $\bg_t$ is independent of $\VV^{t-1}$, we have $\norm{\Pi_{\VV^{t-1}} \bg_{t}}_2^2 \sim \chi^2 (t)$. As a consequence,
	\begin{equation}
		\P \left( \norm{\Pi_{\VV^{t-1}} \bg_{t}}_2 \le C n^{\veps} \ \text{for all} \ t < n^{1/2(k-1)} \right) \ge 1 - \exp \left( -C n^{2 \veps} \right).
	\end{equation}
	To summarize, with probability $1 - \exp(-C n^{2 \veps})$, the following estimate holds for all $t < \min ( T_{\veps}, n^{1/2(k-1)} )$:
	\begin{align*}
		\norm{\vv^{t}_{\parallel}}_2 \le \, & \alpha_t + \norm{\hh_t}_2 + \sqrt{1 - \|\tilde{\vv}^{t-1}_{\parallel}\|_2^{2k - 2}} \cdot \norm{\Pi_{\VV^{t-1}} \bg_{t}}_2 \\
		\le \, & n^{\veps} + 2 \sqrt{n} \cdot \|\tilde{\vv}^{t-1}_{\parallel}\|_2^{k - 1} + Cn^{\veps} \sqrt{1 - \|\tilde{\vv}^{t-1}_{\parallel}\|_2^{2k - 2}} \\
		\le \, & C \left( n^{\veps} + \sqrt{n} \cdot \|\tilde{\vv}^{t-1}_{\parallel}\|_2^{k - 1} \right),
	\end{align*}
	which further implies that
	\begin{align*}
		\norm{\tilde{\vv}^{t}_{\parallel}}_2 \le C \left( n^{\veps - 1/2} + \norm{\tilde{\vv}^{t-1}_{\parallel}}_2^{k - 1} \right).
	\end{align*}
	At initialization, we have $\norm{\tilde{\vv}^{0}_{\parallel}}_2 = 0$. We will use induction to show that $\norm{\tilde{\vv}^{t}_{\parallel}}_2 \le (C+1) n^{\veps - 1/2}$ as long as the above inequality holds. Assume this is true for $t-1$, then we have
	\begin{align*}
		\norm{\tilde{\vv}^{t}_{\parallel}}_2 \le \, & C \left( n^{\veps - 1/2} + (C+1)^{k - 1} n^{(k-1) (\veps - 1/2)} \right) \\
		\le \, & n^{\veps - 1/2} \left( C + C (C+1)^{k-1} n^{(k-2)(\veps - 1/2)} \right) \\
		\le \, & (C + 1) n^{\veps - 1/2}
	\end{align*}
	for sufficiently large $n$. We thus conclude that
	\begin{equation}
		\P \left( \left\vert c_t - 1 \right\vert \le C n^{2(k-1) (\veps - 1 / 2)} \ \text{for all} \ t < \min ( T_{\veps}, n^{1/2(k-1)} ) \right) \ge 1 - \exp(-C n^{2 \veps}).
	\end{equation}
	\paragraph{Concentration of $b_t$.} Recall the definition of $b_t$:
	\begin{align*}
		b_t = \, \langle \hh_t, \vv \rangle = \, & \sum_{(i_1, \cdots, i_{k - 1}) \in \HH_{t - 2}} \beta^{(t-1)}_{i_1, i_2, \cdots, i_{k - 1}} \cdot \left\langle \ww_{i_1, i_2, \cdots, i_{k - 1}}, \vv \right\rangle \\
		= \, & \|\tilde{\vv}^{t-1}_{\parallel}\|_2^{k - 1} \cdot \sum_{j=1}^{(t-1)^{k-1}} \beta_j \langle \ww_j, \vv \rangle,
	\end{align*}
	where we have
	\begin{equation}
		\sum_{j=1}^{(t-1)^{k-1}} \beta_j^2 = 1, \ \text{and} \ \ww_j \stackrel{\iid}{\sim} \normal(\bzero, \id_n).
	\end{equation}
	It then follows that
	\begin{align*}
		\vert b_t \vert \le \, \|\tilde{\vv}^{t-1}_{\parallel}\|_2^{k - 1} \cdot \sup_{\bbeta \in \S^{(t-1)^{k-1} - 1}} \left\vert \sum_{j=1}^{(t-1)^{k-1}} \beta_j \langle \ww_j, \vv \rangle \right\vert \le \|\tilde{\vv}^{t-1}_{\parallel}\|_2^{k - 1} \cdot \norm{\WW \vv}_2.
	\end{align*}
	Since $\WW \vv \in \R^{(t-1)^{k-1}}$ has i.i.d. standard normal entries, we know that $\norm{\WW \vv}_2 \le C n^{1/4}$ with probability no less than $1 - \exp(-C\sqrt{n})$. Therefore, 
	\begin{equation}
		\P \left( \vert b_t \vert \le Cn^{1/4} \|\tilde{\vv}^{t-1}_{\parallel}\|_2^{k - 1} \ \text{for all} \ t < n^{1/2(k-1)} \right) \ge 1 - \exp(-C\sqrt{n}).
	\end{equation} 
	Using the upper bound on $\|\tilde{\vv}^{t-1}_{\parallel}\|_2$ we obtained in the previous paragraph, it follows that
	\begin{equation}
		\vert b_t \vert \le Cn^{1/4} \|\tilde{\vv}^{t-1}_{\parallel}\|_2^{k - 1} \le Cn^{1/4 + (k-1) (\veps - 1/2)}
	\end{equation}
	for all $t < \min(T_{\veps}, n^{1/2(k-1)})$ with probability at least $1 - \exp(-C\sqrt{n}) - \exp(-C n^{2 \veps}) \ge 1 - \exp(-C \sqrt{n})$. This proves the concentration bound on $b_t$.

\section{Proof of \cref{lem:conv_power_iter}}
\label{sec:proof-lem:conv_power_iter}

Note that \cref{prop:summary_alignment} implies that the following event occurs with probability at least $1 - \exp (-C\sqrt{n})$:
\begin{align}
\label{eq:event}
	E = \left\{ \zeta_t \in [1 - n^{-1/6}, 1 + n^{-1/6}], \ \vert b_t \vert \le Cn^{-1/6}, \ \vert c_t - 1 \vert \le C n^{-5/6} \mbox{ for all }t < \min(T_{\veps}, n^{1/2(k-1)})\right\}. 
\end{align}
To facilitate our analysis, we define an auxiliary stochastic process $\{\bar\alpha_t\}_{t \in \NN}$ as follows:
\begin{enumerate}
	\item $\{\bar\alpha_t\}$ and $\{ \alpha_t \}$ have the same initialization, i.e., $\bar{\alpha}_0 = \alpha_0 = 0$.
	\item For all $t \in \NN$, we have $\bar{\alpha}_{t+1} = \gamma_n \bar{\zeta}_t ( \bar{\alpha}_t + \bar{b}_t + \bar{c}_t Z_t )^{k-1}$, where $( \bar{\zeta}_t, \bar{b}_t, \bar{c}_t ) = (\zeta_t, b_t, c_t)$ if $t < \min(T_{\veps}, n^{1/2(k-1)})$, and $(1, 0, 1)$ otherwise.
\end{enumerate}
By definition, we know that $\alpha_t = \bar\alpha_t$ for all $t \in \{0, 1, \cdots, \min(T_{\veps}, n^{1/2(k-1)}) \}$. Further, setting $\delta = (C+1) n^{-1/6}$, \cref{prop:summary_alignment} then implies that
\begin{equation}
	\max \left\{ \left\vert \bar{\zeta}_t - 1 \right\vert, \left\vert \bar{b}_t \right\vert, \left\vert \bar{c}_t - 1 \right\vert \right\} \le \delta, \quad \forall t \in \NN.
\end{equation}
Below we will establish lower and upper bounds on the first hitting time of $\{\bar\alpha_t\}_{t \in \NN}$ to certain level sets, which is defined as follows:
\begin{equation}\label{eq:T-bar}
	\bar T_{\veps} = \min \left\{t \in \mathbb{N}_+ : \vert \bar\alpha_t\vert \ge n^{\veps} \right\}.
\end{equation}
Then, we will show that $\bar{T}_{\veps} = T_{\veps}$ with high probability, and consequently obtain the same lower and upper bounds on $T_{\veps}$.
To begin with, we state a helper lemma that is useful for establishing upper and lower bounds on $\bar T_{\veps}$:
\begin{lem}\label{lem:dominant_process}
	Let $\delta := (C + 1)n^{-1/6}$.
	Fix $\Delta \in [\delta, 1]$. Define the deterministic sequences $\{ \overline{b}_{t, \Delta} \}_{t \ge 0}$ and $\{ \underline{b}_{t, \Delta} \}_{t \ge 0}$ recursively as follows:
	\begin{align}\label{eq:upper_process}
		& \overline{b}_{0, \Delta} \ge 0, \quad \overline{b}_{t+1, \Delta} = \gamma_n (1 + \Delta)^k \cdot \overline{b}_{t, \Delta}^{k-1}, \\
		\label{eq:lower_process}
		& \underline{b}_{0, \Delta} \ge 0, \quad \underline{b}_{t+1, \Delta} = \gamma_n (1 - \Delta)^k \cdot \underline{b}_{t, \Delta}^{k-1}.
	\end{align}
	Then, we have
	\begin{equation}\label{eq:formula_polynomial}
	\begin{split}
		\overline{b}_{t, \Delta} =\, & \left( \gamma_n (1 + \Delta)^k \right)^{((k-1)^t - 1) / (k-2)} \cdot \overline{b}_{0, \Delta}^{(k-1)^t}, \\
		\underline{b}_{t, \Delta} =\, & \left( \gamma_n (1 - \Delta)^{k} \right)^{((k-1)^t - 1) / (k-2)} \cdot \underline{b}_{0, \Delta}^{(k-1)^t}.
	\end{split}
	\end{equation}
	Furthermore, for any $t \in \NN$,
	\begin{align}\label{eq:process_upper_bd}
	\begin{split}
		& \P \left( \vert \bar  \alpha_{t+s} \vert \le \overline{b}_{s, \Delta}, \ \forall s \in \NN \ \big\vert \ \vert \bar \alpha_t \vert \le \overline{b}_{0, \Delta} \right) \\
		\ge\, & \P \left( \vert Z_s \vert \le \frac{\Delta \overline{b}_{s, \Delta} - \delta}{1 + \delta}, \forall s \in \NN \right) \ge 1 - 2 \sum_{s=0}^{\infty} \Phi \left( - \frac{\Delta \overline{b}_{s, \Delta} - \delta}{1 + \delta} \right), 
		\end{split}\\
\label{eq:process_lower_bd}
\begin{split}
		& \P \left( \vert \bar \alpha_{t+s} \vert \ge \underline{b}_{s, \Delta}, \ \forall s \in \NN \ \big\vert \ \vert \bar \alpha_t \vert \ge \underline{b}_{0, \Delta} \right) \\
		\ge\, & \P \left( \vert Z_s \vert \le \frac{ \Delta \underline{b}_{s, \Delta} - \delta}{1 + \delta}, \forall s \in \NN \right) \ge 1 - 2 \sum_{s=0}^{\infty} \Phi \left( - \frac{ \Delta \underline{b}_{s, \Delta} - \delta}{1 + \delta} \right).
	\end{split}
	\end{align}
\end{lem}
We prove \cref{lem:dominant_process} in Appendix \ref{sec:proof-lem:dominant_process}. 

\subsection{Lower bound on $\bar T_{\veps}$}
\label{sec:lower-bound-T}

We start with two useful propositions, whose proofs are deferred to Appendices \ref{sec:proof-prop:first_lower_bound} and \ref{sec:proof-prop:process_lower_bd}, respectively.

\begin{prop}\label{prop:first_lower_bound}
	Let $C_k = (k-2)^{k-2}/(k-1)^{k-1}$ and $\delta = (C + 1)n^{-1/6}$. Define
	\begin{equation}
		M(k, \gamma_n, \delta) := \frac{1}{k-2} \left( \frac{C_k}{\gamma_n(1 + \delta)} \right)^{1/(k-2)}, \ N(k, \gamma_n, \delta) := \frac{1}{1 + \delta} \left( \frac{C_k}{\gamma_n(1 + \delta)} \right)^{1/(k-2)} - \frac{\delta}{1 + \delta}.
	\end{equation}
	Then, for any $T \in \mathbb{N}$, with probability at least $1 - 2 T \Phi ( -N(k, \gamma_n, \delta) )$, we have
	\begin{equation}
		\max_{0 \le t \le T} \left\vert \bar \alpha_t \right\vert \le M(k, \gamma_n, \delta).
	\end{equation}
\end{prop}

\begin{prop}\label{prop:process_lower_bd}
	Assume $\gamma_n = \modif{n^{o(1)}}$ and $\gamma_n \gg (\log n)^{-(k - 2) / 2}$. Let $\veps = \veps_k = (6k - 11)/12 (k-1)$. Then, for any fixed $\eta \in (0, 1)$, and $t_n = \lfloor (1 - \eta) \log_{k-1} \frac{\log_{k-1} n}{\max\{\log_{k-1} \gamma_n, 1\}} \rfloor$, we have
	\begin{equation}\label{eq:lower_bd_X}
		\P \left( \max_{0 \le t \le t_n} \vert \bar \alpha_{t} \vert \le \frac{n^{\veps}}{2} \right) \ge 1 - n^{-C},
	\end{equation}
	where $C > 0$ is an absolute constant.
\end{prop}

With the aid of \cref{prop:first_lower_bound} and \cref{prop:process_lower_bd}, we prove the follwing lower bound on $\bar T_{\veps}$:
For any fixed $\eta \in (0,1)$ and large enough $n$,
with probability at least $1 - n^{-C}$ we have
\begin{equation}\label{eq:lower_bd_Te}
	\bar T_{\veps} \ge \, \max \left\{ \exp \left( \frac{1 - \eta}{2} \left( \frac{C_k}{\gamma_n} \right)^{2/(k-2)} \right), \ (1 - \eta) \log_{k-1} \frac{\log_{k-1} n}{\max\{\log_{k-1} \gamma_n, 1\}} \right\}.
\end{equation}
To show \cref{eq:lower_bd_Te}, note that by definition of $\bar T_{\veps}$ and \cref{prop:process_lower_bd}, it immediately follows that
\begin{equation*}
	\P \left( \bar T_{\veps} \ge (1 - \eta) \log_{k-1} \frac{\log_{k-1} n}{\max\{\log_{k-1} \gamma_n, 1\}} \right) \ge 1 - n^{-C}.
\end{equation*}
It then suffices to consider the case $(\log n)^{-(k-2)/2} \ll \gamma_n \ll 1$, otherwise the lower bound on the right hand side of \cref{eq:lower_bd_Te} is just $(1 - \eta) \log_{k-1} (\log_{k-1} n / \max\{\log_{k-1} \gamma_n, 1\})$ for a large enough $n$. Recall that $M(k, \gamma_n, \delta)$ and $N(k, \gamma_n, \delta)$ are defined in \cref{prop:first_lower_bound}. For a large enough $n$, we have
	\begin{equation*}
		M(k, \gamma_n, \delta) \le \frac{1}{k-2} \left( \frac{C_k}{\gamma_n} \right)^{1/(k-2)} \ll n^{\veps},\,\,\,\,\, \ N(k, \gamma_n, \delta) \ge \frac{1}{1 + 10\delta} \left( \frac{C_k}{\gamma_n} \right)^{1/(k-2)} \gg 1,
	\end{equation*}
	and
	\begin{equation*}
		1 - 2 T \Phi \left( - N(k, \gamma_n, \delta) \right) \ge 1 - \frac{2 T}{N(k, \gamma_n, \delta)} \exp \left( -\frac{1}{2} N(k, \gamma_n, \delta)^2 \right).
	\end{equation*}
	As a consequence, as long as $T \le \exp(N(k, \gamma_n, \delta)^2 / 2)$, with high probability we have
	\begin{equation*}
		\max_{0 \le t \le T} \left\vert \bar \alpha_t \right\vert \le M(k, \gamma_n, \delta) \ll n^{\veps},
	\end{equation*}
	which further implies that
	\begin{equation*}
		\bar T_{\veps} \ge \exp \left( \frac{1}{2(1+10 \delta)^2} \left( \frac{C_k}{\gamma_n} \right)^{2/(k-2)} \right) \ge \exp \left( \frac{1 - \eta}{2} \left( \frac{C_k}{\gamma_n} \right)^{2/(k-2)} \right),
	\end{equation*}
	since $\delta \to 0$ as $n \to \infty$. This completes the proof of the lower bound given in \cref{eq:lower_bd_Te}.
	
\subsection{Upper bound on $\bar T_{\veps}$}
\label{sec:upper-bound-T}

Next, we establish an upper bound for $\bar T_{\veps}$.
Our proof consists of two steps.
First, we show that in a moderately many steps, the alignment $\bar{\alpha}_t$ will reach a sufficiently large magnitude. Then, we use \cref{lem:dominant_process} to prove that, after reaching this magnitude, it takes at most $(1+\eta) \log_{k-1} ( \log_{k-1} n / \max\{\log_{k-1} \gamma_n, 1\})$ steps for $\vert \bar \alpha_t \vert$ to reach $n^{\eps}$, where $\eta \in (0,1)$ is an arbitrarily small fixed positive constant. 
To get started, we establish the following proposition, the proof of which can be found in Appendix \ref{sec:proof-prop:first_stage}.

\begin{prop}\label{prop:first_stage}
	For any $m > 0$ and $T \in \mathbb{N}$, we have
	\begin{equation*}
		\P \left( \max_{0 \le t \le T} \left\vert \bar \alpha_t \right\vert \ge m \right) \ge 1 - \exp \left( - T \Phi \left( - \left( \frac{m}{\gamma_n (1 - \delta)^k} \right)^{1/(k-1)} \right) \right),
	\end{equation*}
	where we recall that $\delta = (C + 1)n^{-1/6}$. 
\end{prop}
For $\Delta \in (\delta, 1)$ and $L > 0$, we define
\begin{equation}
	m(\Delta, L) := \max \left\{ \left( \frac{1 + \Delta}{\gamma_n (1 - \Delta)^k} \right)^{1/(k-2)}, L \right\},
\end{equation}
we know that (using $\Delta \ge \delta$)
\begin{equation*}
	\left( \frac{m(\Delta, L)}{\gamma_n (1 - \delta)^k} \right)^{1/(k-1)} \le \max \left\{ \left( \frac{(1+\Delta)^{1/(k-1)}}{\gamma_n (1-\Delta)^k} \right)^{1/(k-2)}, \ \left( \frac{L (1+\Delta)^{1/(k-1)}}{\gamma_n (1-\Delta)^k} \right)^{1/(k-1)} \right\}.
\end{equation*}
Define $A_{k, \Delta} := (1+\Delta)^{1/(k-1)} / (1 - \Delta)^k$ and $B_{k, \Delta} := (1 + \Delta) / (1 - \Delta)^k$. Invoking Proposition \ref{prop:first_stage} and taking $m = m(\Delta, L)$, we know that with probability at least 
\begin{equation*}
	1 - \max \left\{ \exp \left( -T \Phi \left( - \left( \frac{A_{k, \Delta}}{\gamma_n} \right)^{1/(k-2)} \right) \right), \ \exp \left( -T \Phi \left( - \left( \frac{A_{k, \Delta} L}{\gamma_n} \right)^{1/(k-1)} \right) \right) \right\},
\end{equation*}
there exists $t \le T$ satisfying
\begin{equation*}
	\vert \bar \alpha_t \vert \ge m(\Delta, L) = \max \left\{ \left( \frac{B_{k, \Delta}}{\gamma_n} \right)^{1/(k-2)}, L \right\}.
\end{equation*}
 In what follows, we will show that for a sufficiently large $L$,  
 starting from such $\bar \alpha_t$, it takes at most $(1 + \eta) \log_{k-1} \frac{\log_{k-1} n}{\max\{\log_{k-1} \gamma_n, 1\}}$ steps for $\bar \alpha_t$ reach order $n^{\eps}$.
 Such a result is established as Proposition \ref{prop:process_upper_bd} below. We delay the proof of this proposition to Appendix \ref{sec:proof-prop:process_upper_bd}.

\begin{prop}\label{prop:process_upper_bd}
Assume $\gamma_n = \modif{n^{o(1)}}$ and $\gamma_n \gg (\log n)^{-(k - 2) / 2}$. Let $\veps = \veps_k = (6k - 11)/12 (k-1)$. Then, for any fixed $\eta \in (0, \infty)$ and $t_n = \lfloor (1 + \eta) \log_{k-1} \frac{\log_{k-1} n}{\max\{\log_{k-1} \gamma_n, 1\}} \rfloor$, for $L, n$ sufficiently large we have
	\begin{equation}
		\P \left( \vert \bar \alpha_{t + t_n} \vert \ge n^{\eps} \ \big\vert \ \vert \bar \alpha_t \vert \ge m(\Delta, L) \right) \ge 1 - C \exp \left( - \frac{\Delta^2 m(\Delta, L)^2}{8} \right),
	\end{equation}
	where $C$ is an absolute constant.
\end{prop}

Putting together \cref{prop:first_stage} and \cref{prop:process_upper_bd}, we obtain the following theorem:

\begin{lem}[Upper bound on $\bar T_{\veps}$]\label{thm:upper_bd_X}
Assume $\gamma_n = \modif{n^{o(1)}}$, $\gamma_n \gg (\log n)^{-(k - 2) / 2}$, and $\veps = \veps_k = (6k - 11)/12 (k-1)$.
	For any fixed $\eta > 0$ and sufficiently large $n \in \mathbb{N}$, with probability $1 - o_n(1)$ we have
	\begin{align}
		\bar T_{\veps} \le \, \exp \left( \frac{1 + \eta}{2} \left( \frac{1}{\gamma_n} \right)^{2/(k-2)} \right) + (1 + \eta) \log_{k-1} \frac{\log_{k-1} n}{\max\{\log_{k-1} \gamma_n, 1\}}.
	\end{align}
\end{lem}

\begin{proof}[Proof of \cref{thm:upper_bd_X}]
	Let $L = L_n$ be such that $L \to \infty$ as $n \to \infty$ (specific choice of $L$ will be discussed later). Then, from the discussion following \cref{prop:first_stage}, we know that as long as
	\begin{equation}\label{eq:condition_T}
		T \gg \max \left\{ \exp \left( \frac{1+\eta}{2} \left( \frac{A_{k, \Delta}}{\gamma_n} \right)^{2/(k-2)} \right), \ \exp \left( \frac{1+\eta}{2} \left( \frac{A_{k, \Delta} L}{\gamma_n} \right)^{2/(k-1)} \right) \right\},
	\end{equation}
	then with high probability there exists $t \le T$ satisfying
	\begin{equation*}
		\vert \bar \alpha_t \vert \ge m(\Delta, L) = \max \left\{ \left( \frac{B_{k, \Delta}}{\gamma_n} \right)^{1/(k-2)}, L \right\}.
	\end{equation*}
	Applying \cref{prop:process_upper_bd} yields that
	\begin{equation}
		\left\vert \bar \alpha_{t + t_n} \right\vert \ge n^{\eps},\,\,\,\, \ t_n = \lfloor (1 + \eta) \log_{k-1} \frac{\log_{k-1} n}{\max\{\log_{k-1} \gamma_n, 1\}} \rfloor.
	\end{equation}
	The above calculation implies that
	\begin{equation}
		\bar T_{\veps} \le T +  (1 + \eta) \log_{k-1} \frac{\log_{k-1} n}{\max\{\log_{k-1} \gamma_n, 1\}}.
	\end{equation}
	Next we discuss the choice of $L$, which will eventually lead to an upper bound on $T$. Let $L$ be the solution to the equation below
	\begin{equation*}
		\exp \left( \frac{1+\eta}{2} \left( \frac{A_{k, \Delta} L}{\gamma_n} \right)^{2/(k-1)} \right) = \exp \left( \frac{1+\eta}{2} \left( \frac{A_{k, \Delta}}{\gamma_n} \right)^{2/(k-2)} \right) + \left(\log_{k-1} \frac{\log_{k-1} n}{\max\{\log_{k-1} \gamma_n, 1\}} \right)^{1 - \eta},
	\end{equation*}
	then one can show that $L \to \infty$ as $n \to \infty$. In fact, the above equation implies
	\begin{align*}
		L \ge \, & C_{k, \Delta, \eta} \cdot \max \left\{ \gamma_n^{-1/(k-2)}, \,\, \gamma_n \cdot \left( (1 - \eta) \cdot \log  \log_{k-1} \frac{\log_{k-1} n}{\max\{\log_{k-1} \gamma_n, 1\}}  \right)^{(k - 1) / 2} \right\} \\
		\ge \, & C_{k, \Delta, \eta} \cdot \left( (1 - \eta) \cdot \log  \log_{k-1} \frac{\log_{k-1} n}{\max\{\log_{k-1} \gamma_n, 1\}}  \right)^{1 / 2} \to \infty,
	\end{align*}
	where $C_{k, \Delta, \eta}$ is a constant only depending on $k, \Delta, \eta$. Under this choice of $L$, setting
	\begin{equation}
		T = \exp \left( \frac{1 + 2\eta}{2} \left( \frac{A_{k, \Delta}}{\gamma_n} \right)^{2/(k-2)} \right) + \eta \log_{k-1} \frac{\log_{k-1} n}{\max\{\log_{k-1} \gamma_n, 1\}}
	\end{equation}
	verifies \cref{eq:condition_T}. We finally deduce that
	\begin{equation}
		\bar T_{\veps} \le \exp \left( \frac{1 + 2\eta}{2} \left( \frac{A_{k, \Delta}}{\gamma_n} \right)^{2/(k-2)} \right) + (1 + 2 \eta) \log_{k-1} \frac{\log_{k-1} n}{\max\{\log_{k-1} \gamma_n, 1\}}.
	\end{equation}
	Since $\eta$ and $\Delta$ can be arbitrarily small, the above upper bound is equivalent to
	\begin{equation}
		\bar T_{\veps} \le \exp \left( \frac{1 + \eta}{2} \left( \frac{1}{\gamma_n} \right)^{2/(k-2)} \right) + (1 + \eta) \log_{k-1} \frac{\log_{k-1} n}{\max\{\log_{k-1} \gamma_n, 1\}}.
	\end{equation}
	This concludes the proof.
\end{proof}

\subsection{Proof of the lemma}

Finally, we put together results in the previous two sections and prove \cref{lem:conv_power_iter}. 
	Denote the lower and upper bounds in the statement of the theorem as $L_{\veps}$ and $U_{\veps}$, respectively. Then we know that $L_{\veps} \le \bar T_{\veps}  \le U_{\veps} \le n^{1/2(k-1)}$ with high probability. Note that $\P (E_n) \to 1$ as $n \to \infty$, where
	\begin{equation}
		E_n := \left\{ \alpha_t = \bar \alpha_t \ \text{for all} \ t < \min(T_{\veps}, n^{1/2(k-1)}) + 1 \ \text{and} \ L_{\veps} \le \bar T_{\veps} \le U_{\veps} \right\}.
	\end{equation}
	We will show that $T_{\veps} = \bar T_{\veps}$ on $E_n$. Assume this is not true, then there are two possibilities:
	\begin{itemize}
		\item [(a)] $T_{\veps} < \bar T_{\veps}$. Since $\bar T_{\veps} \le n^{1/2(k-1)}$, we know that $\bar\alpha_t = \alpha_t$ for all $t \le T_{\veps}$, which further implies that $\vert \bar\alpha_{T_{\veps}} \vert \ge n^{\veps}$. As a consequence, $T_{\veps} \ge \bar T_{\veps}$, a contradiction.
		\item [(b)] $T_{\veps} > \bar T_{\veps}$. In this case, we know that $\bar \alpha_t = \alpha_t$ for all $t < \min(\bar T_{\veps} + 1, n^{1/2(k-1)}) = \bar T_{\veps} + 1$. Therefore, $ \vert \alpha_{\bar T_{\veps}} \vert \ge n^{\veps}$, and we know that $T_{\veps} \le \bar T_{\veps}$, a contradiction.
	\end{itemize}
	We thus conclude that $T_{\veps} = \bar T_{\veps}$ on $E_n$. Hence, with high probability $L_{\veps} \le T_{\veps} \le U_{\veps}$ as well.
This completes the proof of \cref{lem:conv_power_iter}.

\section{Proofs of auxiliary lemmas in \cref{sec:proof-lem:conv_power_iter}}\label{sec:proof_process}

\subsection{Proof of \cref{lem:dominant_process}}
\label{sec:proof-lem:dominant_process}

	We first prove \cref{eq:formula_polynomial}. Note that for any sequence $\{ b_t \}_{t \ge 0}$ satisfying $b_{t+1} = q b_t^{k-1}$ for some $q \in \R$, one has
	\begin{equation*}
		q^{1/(k-2)} b_{t+1} = \left( q^{1/(k-2)} b_t \right)^{k-1},
	\end{equation*}
	thus leading to
	\begin{equation*}
		q^{1/(k-2)} b_{t} = \left( q^{1/(k-2)} b_0 \right)^{(k-1)^t} \implies b_t = q^{((k-1)^t - 1) / (k-2)} b_0^{(k-1)^t}.
	\end{equation*}
	Specializing the above equation to $(b_0, q) = ( \bar{b}_{0, \Delta}, \gamma_n (1+\Delta)^{k})$ and $(b_0, q) = ( \underline{b}_{0, \Delta},  \gamma_n (1-\Delta)^{k})$ proves \cref{eq:formula_polynomial}.
	
	 Next, we prove \cref{eq:process_upper_bd}. The proof of \cref{eq:process_lower_bd} follows similarly and we omit it for the sake of simplicity. Without loss of generality we may assume $t=0$. The proof applies without change to positive $t$. 
	  To prove \cref{eq:process_upper_bd}, it suffices to prove the following claim:
	\begin{equation}
		\vert \bar \alpha_0 \vert \le \overline{b}_{0, \Delta}, \ \mbox{and} \ \vert Z_s \vert \le \frac{\Delta \overline{b}_{s, \Delta} - \delta}{1 + \delta},\,\, \forall s \in \NN \implies \vert \bar \alpha_{s} \vert \le \overline{b}_{s, \Delta}, \ \forall s \in \NN.
	\end{equation}
	We establish the above relationship via induction. For $s = 0$, it holds trivially. Now assume $\vert \bar \alpha_{s} \vert \le \overline{b}_{s, \Delta}$, then we know that
	\begin{align*}
		\vert \bar \alpha_{s+1} \vert = \, & \gamma_n \bar{\zeta}_s \cdot \left\vert \bar \alpha_s + \bar b_s + \bar c_s Z_s \right\vert^{k-1} \le \gamma_n(1 + \delta) \cdot \left( \vert \bar \alpha_s \vert + \delta + (1 + \delta) \vert Z_s \vert \right)^{k-1} \\
		\le \, & \gamma_n (1 + \Delta) \cdot \left( \overline{b}_{s, \Delta} + \Delta \overline{b}_{s, \Delta} \right)^{k-1} = \gamma_n (1 + \Delta)^k \cdot \overline{b}_{s, \Delta}^{k-1} = \overline{b}_{s+1, \Delta}.
	\end{align*}
	This completes the induction. As a consequence, we obtain that
	\begin{align*}
		& \P \left( \vert \bar \alpha_{t+s} \vert \le \overline{b}_{s, \Delta}, \ \forall s \in \NN \ \big\vert \ \vert \bar \alpha_t \vert \le \overline{b}_{0, \Delta} \right) \\
		\ge \, & \P \left( \vert Z_s \vert \le \frac{\Delta \overline{b}_{s, \Delta} - \delta}{1 + \delta}, \forall s \in \NN \right) = \prod_{s=0}^{\infty} \P \left( \vert Z_s \vert \le \frac{\Delta \overline{b}_{s, \Delta} - \delta}{1 + \delta} \right) \\
		= \, & \prod_{s=0}^{\infty} \left( 1 - 2 \Phi \left( - \frac{\Delta \overline{b}_{s, \Delta} - \delta}{1 + \delta} \right) \right) \ge 1 - 2 \sum_{s=0}^{\infty} \Phi \left( - \frac{\Delta \overline{b}_{s, \Delta} - \delta}{1 + \delta} \right),
	\end{align*}
	where the last line follows from the inequality $\prod_{s = 0}^{\infty} (1-x_s) \ge 1 - \sum x_s$ for $x_s \in [0, 1]$. This completes the proof of the lemma.

\subsection{Proof of \cref{prop:first_lower_bound}}
\label{sec:proof-prop:first_lower_bound}

	Note that by definition of $\{ \bar \alpha_t \}_{t \ge 0}$, we have
	\begin{equation*}
		\left\vert \bar \alpha_{t+1} \right\vert \le \gamma_n(1 + \delta) \left( \left\vert \bar \alpha_t \right\vert + \delta + (1 + \delta) \left\vert Z_t \right\vert \right)^{k-1}.
	\end{equation*}
	Recall that $C_k = (k-2)^{k-2}/(k-1)^{k-1}$, we next show that as long as
	\begin{equation*}
		\left\vert Z_t \right\vert \le \frac{1}{1 + \delta} \left( \frac{C_k}{\gamma_n(1 + \delta)} \right)^{1/(k-2)} - \frac{\delta}{1 + \delta} \ \text{for all} \ 0 \le t \le T-1,
	\end{equation*}
	then we also have
	\begin{equation*}
		\left\vert \bar \alpha_t \right\vert \le \frac{1}{k-2} \left( \frac{C_k}{\gamma_n(1 + \delta)} \right)^{1/(k-2)} \ \text{for all} \ 0 \le t \le T.
	\end{equation*}
	To this end, we use induction. For $t = 0$, we already have $\alpha_0 = 0$, so the above inequality holds automatically. Assume that it is true for all $t \le T - 1$, then one has
	\begin{align*}
		\left\vert \bar \alpha_{t+1} \right\vert \le \, & \gamma_n(1 + \delta) \left( \left\vert \bar \alpha_t \right\vert + \delta + (1 + \delta) \left\vert Z_t \right\vert \right)^{k-1} \\
		\le \, & \gamma_n(1 + \delta) \left( \left\vert \bar \alpha_t \right\vert + \left( \frac{C_k}{\gamma_n(1 + \delta)} \right)^{1/(k-2)} \right)^{k-1} \\
		\le \, & \gamma_n(1 + \delta) \times  \frac{(k-1)^{k-1}}{(k-2)^{k-1}} \times \left( \frac{C_k}{\gamma_n(1 + \delta)} \right)^{(k-1)/(k-2)} \\
		= \, & \frac{1}{k-2} \left( \frac{C_k}{\gamma_n(1 + \delta)} \right)^{1/(k-2)} = M(k, \gamma_n, \delta).
	\end{align*}
	This completes the induction. As a consequence, we deduce that
	\begin{align*}
		& \P \left( \left\vert \bar \alpha_t \right\vert \le \frac{1}{k-2} \left( \frac{C_k}{\gamma_n(1 + \delta)} \right)^{1/(k-2)} \ \text{for all} \ 0 \le t \le T \right) \\
		\ge \, & \P \left( \left\vert Z_t \right\vert \le \frac{1}{1 + \delta} \left( \frac{C_k}{\gamma_n(1 + \delta)} \right)^{1/(k-2)} - \frac{\delta}{1 + \delta} \ \text{for all} \ 0 \le t \le T-1 \right) \\
		= \, & \left( 1 - 2 \Phi \left( \frac{\delta}{1 + \delta} - \frac{1}{1 + \delta} \left( \frac{C_k}{\gamma_n(1 + \delta)} \right)^{1/(k-2)} \right) \right)^T \\
		\ge \, & 1 - 2 T \Phi \left( \frac{\delta}{1 + \delta} - \frac{1}{1 + \delta} \left( \frac{C_k}{\gamma_n(1 + \delta)} \right)^{1/(k-2)} \right) = 1 - 2 T \Phi \left( - N(k, \gamma_n, \delta) \right).
	\end{align*}
	This completes the proof.

\subsection{Proof of \cref{prop:process_lower_bd}}
\label{sec:proof-prop:process_lower_bd}
	Take $\Delta = 1$. Consider the sequence $\{ \overline{b}_{t, \Delta} \}_{t \ge 0}$ defined via $\overline{b}_{0, \Delta} = M_{k, n} = \log_{k-1} n > 0$, and
	\begin{equation*}
		\overline{b}_{t+1, \Delta} = \gamma_n 2^{k} \cdot \overline{b}_{t, \Delta}^{k-1}.
	\end{equation*}
	Then, according to \cref{lem:dominant_process} we know that
	\begin{equation}
		\overline{b}_{t, \Delta} = \left( 2^{k} \gamma_n \right)^{((k-1)^t - 1) / (k-2)} \cdot M_{k, n}^{(k-1)^t},
	\end{equation}
	and that
	\begin{align*}
		\P \left( \vert \bar \alpha_t \vert \le \overline{b}_{t, \Delta}, \ \forall t \in \NN \right) \stackrel{(i)}{=} \, \P \left( \vert \bar \alpha_t \vert \le \overline{b}_{t, \Delta}, \ \forall t \in \NN \ \big\vert \ \vert \bar \alpha_0 \vert \le M_{k, n} \right) \ge\, 1 - 2 \sum_{t=0}^{\infty} \Phi \left( - \frac{\overline{b}_{t, \Delta} - \delta}{1 + \delta} \right),
	\end{align*}
	where $(i)$ is because $\bar \alpha_0 = 0$. Since $\gamma_n \gg (\log n)^{-(k - 2) / 2}$, we then see that $M_{k, n} \gg \max\{\gamma_n^{-1/(k-2)}, 1\}$, hence for large enough $n$:
	\begin{equation*}
		\frac{\overline{b}_{t, \Delta} - \delta}{1 + \delta} \ge \frac{\overline{b}_{t, \Delta}}{2} \ \text{for all} \ t \ge 0.
	\end{equation*}
	Therefore, for sufficiently large $n$ it holds that
	\begin{align*}
		1 - 2 \sum_{t=0}^{\infty} \Phi \left( - \frac{\overline{b}_{t, \Delta} - \delta}{1 + \delta} \right) \ge \, & 1 - 2 \sum_{t=0}^{\infty} \Phi \left( - \frac{\overline{b}_{t, \Delta}}{2} \right) \ge \, 1 - 2 \sum_{t=0}^{\infty} \Phi \left( - \left( 2^{k-1} \gamma_n M_{k, n}^{k-2} \right)^t \cdot M_{k, n} \right) \\
		\ge\, & 1 - C \sum_{t=0}^{\infty} \exp \left( - \frac{M_{k, n}^2}{2} \cdot \left( 2^{k-1} \gamma_n M_{k, n}^{k-2} \right)^{2 t} \right) \\
		\ge\, & 1 - C \sum_{t=0}^{\infty} \exp \left( - \frac{(t+1) M_{k, n}^2}{2} \right) \ge 1 - C \exp \left( - \frac{M_{k, n}^2}{2} \right),
	\end{align*}
	where $C$ is a positive numerical constant. Furthermore, we have
	\begin{equation*}
		M_{k, n} = \log_{k-1} n \implies 1 - C \exp \left( - \frac{M_{k, n}^2}{2} \right) \ge 1 - n^{-C}.
	\end{equation*}
	It then follows that for large enough $n$, with probability at least $1 - n^{-C}$,
	\begin{align*}
		& \max_{0 \le t \le t_n} \vert \bar \alpha_t \vert \le \max_{0 \le t \le t_n} \overline{b}_{t, \Delta} \le \max_{0 \le t \le t_n}\left( \max(2^{k} \gamma_n, 1) \cdot M_{k, n} \right)^{(k-1)^t} \\
		\le\, & \left( \max(2^{k} \gamma_n, 1) \cdot M_{k, n} \right)^{(k-1)^{t_n}} \le\, \left( \max(2^{k} \gamma_n, 1) \cdot M_{k, n} \right)^{\left(\frac{\log_{k-1} n}{\max\{\log_{k-1} \gamma_n, 1\}} \right)^{1 - \eta}} \\
		=\, & \exp \left( \left(\frac{\log_{k-1} n}{\max\{\log_{k-1} \gamma_n, 1\}} \right)^{1 - \eta} \cdot \left( \log M_{k, n} + \log \left( \max(2^{k} \gamma_n, 1) \right) \right) \right) \\
		\le\, & \exp \left( C_0 \left( \left( \log n \right)^{1 - \eta/2} + (\log n)^{1 - \eta} (\log \left( \max\{\gamma_n, 1\} \right))^{\eta} \right) \right) \stackrel{(i)}{=} \modif{n^{o(1)}} \le \frac{n^{\veps}}{2}
	\end{align*}
	for sufficiently large $n$, where $(i)$ is due to our assumption: $\gamma_n = \modif{n^{o(1)}}$. In the above display, $C_0$ is another positive numerical constant. This completes the proof of the proposition.
	
\subsection{Proof of \cref{prop:first_stage}}
\label{sec:proof-prop:first_stage}

	Note that for any $0 \le t \le T - 1$, $Z_t$ is independent of $\bar{\alpha}_t + \bar{b}_t$. Therefore, for any $x \ge 0$, with probability $ \Phi (-x)$, we have $\vert Z_t \vert \ge x$ and $\sign (Z_t) = \sign (\bar{\alpha}_t + \bar{b}_t)$. When this event occurs, we have
	\begin{equation*}
		\left\vert \bar{\alpha}_t + \bar{b}_t + \bar{c}_t Z_t \right\vert \ge \bar{c}_t \left\vert Z_t \right\vert \ge (1 - \delta) x \implies \left\vert \bar{\alpha}_{t+1} \right\vert \ge \gamma_n (1 - \delta)^k x^{k-1}.
	\end{equation*}
	Define $E_t := \{ \vert Z_t \vert < x \ \text{or} \ \sign(Z_t) \neq \sign(\bar{\alpha}_t + \bar{b}_t) \}$, then $\P (E_t) = \P \left( E_t \vert \cap_{s=0}^{t-1} E_s \right) = 1 - \Phi(-x)$ for all $0 \le t \le T - 1$. Note that on $(\cap_{t=0}^{T-1} E_t)^{c}$, we have
	\begin{equation}
		\max_{0 \le t \le T} \left\vert \bar{\alpha}_t \right\vert \ge \gamma_n(1 - \delta)^k x^{k-1}.
	\end{equation}
	Since $\{ \bar{\alpha}_t \}_{t \ge 0}$ is a Markov chain, we obtain that
	\begin{align*}
		\P \left( (\cap_{t=0}^{T-1} E_t)^{c} \right) = \, & 1 - \P \left( \cap_{t=0}^{T-1} E_t \right) = 1 - \prod_{t=0}^{T-1} \P \left( E_t \vert \cap_{s=0}^{t-1} E_s \right) \\
		= \, & 1 - (1 - \Phi(-x))^T \ge 1 - \exp (-T \Phi(-x)).
	\end{align*}
	Now, choosing $x = (m / \gamma_n (1 - \delta)^k)^{1/(k-1)}$, it follows that
	\begin{equation}
		\P \left( \max_{0 \le t \le T} \left\vert \bar{\alpha}_t \right\vert \ge m \right) \ge 1 - \exp \left( - T \Phi \left( - \left( \frac{m}{\gamma_n (1 - \delta)^k} \right)^{1/(k-1)} \right) \right).
	\end{equation}
	This completes the proof.

\subsection{Proof of \cref{prop:process_upper_bd}}
\label{sec:proof-prop:process_upper_bd}

Consider the sequence $\{ \underline{b}_{t, \Delta} \}_{t \ge 0}$ defined as per \cref{eq:lower_process} with $\underline{b}_{0, \Delta} = m(\Delta, L)$. Applying \cref{lem:dominant_process} to $\{ \bar \alpha_t \}$ and $\{ \underline{b}_{t, \Delta} \}$ yields 
\begin{equation}
\label{eq:lower-bound-b}
	\underline{b}_{t, \Delta} = \left( \gamma_n (1 - \Delta)^k \right)^{((k-1)^t - 1) / (k-2)} \cdot m(\Delta, L)^{(k-1)^t},
\end{equation}
and consequently,
\begin{align}
\label{eq:72}
	& \P \left( \vert \bar \alpha_{t+t_n} \vert \ge \underline{b}_{t_n, \Delta} \ \big\vert \ \vert \bar \alpha_t \vert \ge m(\Delta, L) \right) \nonumber \\
	\ge \, & \P \left( \vert \bar \alpha_{t+s} \vert \ge \underline{b}_{s, \Delta}, \ \forall s \in \NN \ \big\vert \ \vert \bar \alpha_t \vert \ge m(\Delta, L) \right) \nonumber \\
	\ge \, & 1 - 2 \sum_{s=0}^{\infty} \Phi \left( - \frac{ \Delta \underline{b}_{s, \Delta} - \delta}{1 + \delta} \right). 
	\end{align}
Note that by definition $\gamma_n (1 - \Delta)^k m(\Delta, L)^{k-2} \ge 1 + \Delta$. Also note that 
\begin{align*}
	\underline{b}_{t, \Delta} = \left( \gamma_n (1 - \Delta)^k \right)^{-1/(k - 2)} \cdot \left( m(\Delta, L) \cdot \gamma_n^{1/(k - 2)} (1 - \Delta)^{k / (k - 2)} \right)^{(k - 1)^t}, 
\end{align*}
hence $\{\underline{b}_{t, \Delta}\}_{t \geq 0}$ is an increasing sequence, lower bounded by $m(\Delta, L)$. Therefore, for sufficiently large $n$ the last line in \cref{eq:72} if further no smaller than $1 - 2 \sum_{s=0}^{\infty} \Phi ( - { \Delta \underline{b}_{s, \Delta}}/{2} )$, which by \cref{eq:lower-bound-b} is equal to 
	
	\begin{align*}
	& 1 - 2 \sum_{s=0}^{\infty} \Phi \left( - \frac{\Delta}{2} \left( \gamma_n (1 - \Delta)^k \right)^{((k-1)^s - 1) / (k-2)} \cdot m(\Delta, L)^{(k-1)^s} \right) \\
	=\, &  1 - 2 \sum_{s=0}^{\infty} \Phi \left( - \frac{\Delta}{2} \left( \gamma_n (1 - \Delta)^k m(\Delta, L)^{k-2} \right)^{((k-1)^s - 1) / (k-2)} \cdot m(\Delta, L) \right) \\
	\stackrel{(i)}{\ge}\, & 1 - 2 \sum_{s=0}^{\infty} \Phi \left( - \frac{\Delta}{2} (1 + \Delta)^{((k-1)^s - 1) / (k-2)} \cdot m(\Delta, L) \right) \stackrel{(ii)}{\ge} 1 - 2 \sum_{s=0}^{\infty} \Phi \left( - \frac{\Delta}{2} (1+\Delta)^{s} m(\Delta, L) \right) \\
	\stackrel{(iii)}{\ge}\, & 1 - C \sum_{s=0}^{\infty} \exp \left( - \frac{\Delta^2}{8} (1 + \Delta)^{2s} m(\Delta, L)^2 \right) \ge 1 - C \sum_{s=0}^{\infty} \exp \left( - \frac{1+2 s \Delta}{8} \Delta^2 m(\Delta, L)^2 \right) \\
	=\, & 1 - \frac{C \exp (-\Delta^2m(\Delta, L)^2/8 )}{1 - \exp (-\Delta^3 m(\Delta, L)^2/4)} \ge 1 - C \exp \left( - \frac{\Delta^2 m(\Delta, L)^2}{8} \right),
\end{align*}
where  $(i)$ follows from our choice of $m(\Delta, L)$: $\gamma_n (1 - \Delta)^k m(\Delta, L)^{k-2} \ge 1 + \Delta$, $(ii)$ is due to the inequality $(k-1)^s \ge 1 + (k-2) s$, and $(iii)$ follows from the well-known fact regarding Gaussian tail bound: $\Phi(-x) \le C \exp(-x^2/2)$ for $x \le 0$, where $C > 0$ is a numerical constant.

It then suffices to show that $\underline{b}_{t_n, \Delta} \ge n^{\eps}$. By direct calculation, we obtain that for a sufficiently large $L$ and $n$, it holds that
\begin{align}
\label{eq:74}
	\underline{b}_{t_n, \Delta} = \, & \left( \gamma_n (1 - \Delta)^k \right)^{((k-1)^{t_n} - 1) / (k-2)} \cdot m(\Delta, L)^{(k-1)^{t_n}} \nonumber \\
	\ge \, & \left( \gamma_n (1 - \Delta)^k m(\Delta, L)^{k-2} \right)^{((k-1)^{t_n} - 1) / (k-2)} \nonumber \\
	\ge \, &  (\gamma_n (1 - \Delta)^k m(\Delta, L)^{k-2})^{(k-1)^{t_n - 1}} \nonumber \\
	\ge \, & (\gamma_n (1 - \Delta)^k m(\Delta, L)^{k-2})^{(k-1)^{(1 + \eta/2) \log_{k-1} (\log_{k-1} n  /  \max\{\log_{k - 1} \gamma_n, 1\})}} \nonumber \\
	\ge \, & (\gamma_n (1 - \Delta)^k m(\Delta, L)^{k-2})^{(\log_{k-1} n  /  \max\{\log_{k - 1} \gamma_n, 1\})^{(1 + \eta / 2)}}. 
\end{align}
%
%
By definition of $m(\Delta, L)$, we see that $\gamma_n (1 - \Delta)^k m(\Delta, L)^{k-2} \geq C_{\Delta, k, L} \max\{\gamma_n, e\}$, where $C_{\Delta, k, L} > 1$ is a constant that depends only on $(\Delta, k, L)$ (this is true if we choose $L$ large enough). For simplicity, define $\omega = \max\{\gamma_n, e\}$. 
To show that the last line of \cref{eq:74} is no smaller than $n^{\eps}$, it suffices to prove the following:
\begin{align*}
	& \left( \frac{\log_{k - 1} n}{ \log_{k - 1} \omega } \right)^{1 + \eta / 2} \cdot \log_{k - 1} (C_{\Delta,k, L} \omega)
 \geq \eps \log_{k - 1} n \\
 \Longleftrightarrow & \left( \log_{k-1} n \right)^{\eta / 2} \cdot \log_{k - 1} (C_{\Delta,k, L} \omega) \ge \veps \left( \log_{k - 1} \omega \right)^{1 + \eta / 2}.
 \end{align*}
This is true for large enough $n$ since $\omega \ll n$ and $C_{\Delta, k, L} > 1$. The proof is complete. 


%


\section{Proof of \cref{thm:stopping}}\label{sec:proof_stop}

In this section, we follow the definitions and notations introduced in the proof of \cref{thm:conv_power_iter}. Recall from \cref{eq:T-eps} that $T_{\veps} = \min \{ t \in \NN_+: |\alpha_t| \geq n^{\varepsilon} \}$, where
\begin{equation*}
	\alpha_t = \lambda_n \langle \vv, \tilde{\vv}^{t-1} \rangle^{k - 1} = \gamma_n \left( \sqrt{n} \langle \vv, \tilde{\vv}^{t-1} \rangle \right)^{k-1}.
\end{equation*}

\subsubsection*{Part I: before $T_{\eps}$}

We first prove that with probability $1 - o_n(1)$, 
\begin{align*}
	\|\tilde{\vv}^t - \tilde{\vv}^{t - 1}\|_2 < 1 / 2
\end{align*}
simultaneously for all $t \in \{1, 2, \cdots, T_{\eps} - 1\}$.
Invoking \cref{lem:conv_power_iter}, we see that for $\eps = (6k - 11) / 12(k - 1)$, with probability $1 - o_n(1)$ we have $T_{\eps} + 10 < n^{1/10(k - 1)}$.
Recall that $\ww_{i_1, i_2, \cdots, i_{k - 1}} \overset{d}{=} \bg_t \overset{d}{=} \normal(\bzero, \id_n)$. 
Employing standard Gaussian concentration inequalities, we see that with probability at least $1 - o_n(1)$, it holds that 
\begin{align}
\label{eq:gaussian-bounds}
	\big\| \ww_{i_1, i_2, \cdots, i_{k - 1}} \big\|_2 \leq \sqrt{n \log n}, \qquad \left|\big\| \bg_t \big\|_2 - \sqrt{n} \right| \leq \log n, \qquad \big\| \bg_t - \bg_{t + 1} \big\|_2 \geq \sqrt{n}
\end{align}
for all $t \in \{0, 1, \cdots, \lceil n^{1 / 10(k - 1)} \rceil\}$ and $(i_1, \cdots, i_{k - 1}) \in \HH_t$.
In addition, \cref{lem:err_bound} implies that with probability $1 - o_n(1)$ we have 
\begin{align*}
	\left| \sqrt{1 - \|\tilde \vv_{\parallel}^{t - 1}\|_2^{2k - 2}} - 1 \right| = \left| c_t - 1 \right| \leq  C n^{-5 / 6}
\end{align*}
for an absolute constant $C > 0$ and all $t \in \{0, 1, \cdots, T_{\eps} - 1\}$. 
Observe that for $t \in \{1, \cdots, \lceil n^{1 / 10(k - 1)} \rceil\}$, 
\begin{align*}
	\sum_{(i_1, \cdots, i_{k - 1}) \in \HH_{t - 2}}|\beta_{i_1, i_2, \cdots, i_{k - 1}}|^2 + c_t^2 = 1, \qquad \sum_{(i_1, \cdots, i_{k - 1}) \in \HH_{t - 3}}|\beta_{i_1, i_2, \cdots, i_{k - 1}}|^2 + c_{t - 1}^2 = 1, 
\end{align*}
hence using Cauchy-Schwartz inequality we get
\begin{align}
\label{eq:sum-beta}
	\sum_{(i_1, \cdots, i_{k - 1}) \in \HH_{t - 2}}|\beta_{i_1, i_2, \cdots, i_{k - 1}}| \leq C n^{-5 / 12 + 1 / 20}, \qquad \sum_{(i_1, \cdots, i_{k - 1}) \in \HH_{t - 3}}|\beta_{i_1, i_2, \cdots, i_{k - 1}}| \leq C n^{-5 / 12 + 1 / 20}. 
\end{align}
 Plugging these bounds into \cref{eq:expression-vt+1} and applying the triangle inequality, we see that with probability $1 - o_n(1)$
\begin{align*}
	& \|\vv^t\|_2 \leq  n^{1/2} + \log n + n^{(6k - 11) / 12(k - 1)} + C n^{-11 / 30} \sqrt{n \log n} + C n^{-5 / 6} \cdot (n^{1/2} + \log n), \\
	& \|\vv^t\|_2 \geq n^{1/2} - \log n - n^{(6k - 11) / 12(k - 1)} - C n^{-11 / 30} \sqrt{n \log n} - C n^{-5 / 6} \cdot (n^{1/2} + \log n). 
\end{align*}
for all $t \in \{0, 1, \cdots, T_{\eps} - 1\}$. 
From the above upper bound, we see that with high probability 
\begin{align*}
	\left| \|\vv^t\|_2 - n^{1/2} \right| \leq 2n^{(6k - 11) / 12(k - 1)}
\end{align*}
for all $t \in \{0, 1, \cdots, T_{\eps} - 1\}$. 

We can also employ \cref{eq:gaussian-bounds} and \cref{eq:sum-beta} to lower bound $\|\vv^t - \vv^{t - 1}\|_2$. Invoking the triangle inequality, we get:  
\begin{align*}
	& \left\| \vv^t - \vv^{t - 1} \right\|_2 \\
	\geq & \, \|\bg_{t} - \bg_{t - 1}\|_2 - |\alpha_t| - |\alpha_{t - 1}| - |1 - c_t| \cdot \|\bg_t\|_2 - |1 - c_{t- 1}| \cdot \|\bg_{t - 1}\|_2 - \\
	& \sum_{(i_1, \cdots, i_{k - 1}) \in \HH_{t - 2}}|\beta_{i_1, i_2, \cdots, i_{k - 1}}| \cdot \|\ww_{i_1, i_2, \cdots, i_{k - 1}}\|_2 - \sum_{(i_1, \cdots, i_{k - 1}) \in \HH_{t - 3}}|\beta_{i_1, i_2, \cdots, i_{k - 1}}| \cdot \|\ww_{i_1, i_2, \cdots, i_{k - 1}}\|_2. 
\end{align*}
By definition of $T_{\eps}$ we have $\max\{|\alpha_t|, |\alpha_{t - 1}|\} \leq n^{\eps} = n^{(6k - 11) / 12(k - 1)}$. By \cref{lem:err_bound} we have $\max\{|c_t - 1|, |c_{t - 1} - 1|\} \leq C n^{-5 / 6}$. 
Putting together the above equations, we get 
\begin{align*}
	& \left\| \vv^t - \vv^{t - 1} \right\|_2 \\
	\geq &\, \|\bg_{t} - \bg_{t - 1}\|_2 - 2n^{(6k - 11) / 12(k - 1)} - C n^{-5/6} \cdot \left( n^{1/2} + \log n \right) -  C n^{-11 / 30} \sqrt{n \log n} \\
	\geq &\, n^{1/2} - 3 n^{(6k - 11) / 12(k - 1)},
\end{align*}
which holds with high probability for all $t \in \{0, 1, \cdots, T_{\eps} - 1\}$ for a sufficiently large $n$. 
We can then use this to give a high-probability lower bound for $\|\tilde\vv^{t} - \tilde\vv^{t - 1}\|_2$. Leveraging the triangle inequality, we obtain that for a large enough $n$, with probability $1 - o_n(1)$
\begin{align*}
	\left\|\tilde\vv^{t} - \tilde\vv^{t - 1} \right\|_2 = & \left\| \frac{\vv^{t}}{\|\vv^{t}\|_2} - \frac{\vv^{t - 1}}{\|\vv^{t - 1}\|_2} \right\|_2 \\
	\geq & \left\| \frac{\vv^{t}}{\sqrt{n}} - \frac{\vv^{t - 1}}{\sqrt{n}} \right\|_2 - \left| \frac{1}{\|\vv^{t}\|_2} - \frac{1}{\sqrt{n}} \right| \cdot \|\vv^{t}\|_2 - \left| \frac{1}{\|\vv^{t - 1}\|_2} - \frac{1}{\sqrt{n}} \right| \cdot \|\vv^{t - 1}\|_2 \\
	\geq & \, 1 - 20n^{-5 / 12(k - 1)} 
\end{align*}
for all $t \in \{0, 1, \cdots, T_{\eps} - 1\}$. 
Therefore, with probability $1 - o_n(1)$ we have $\|\tilde\vv^t - \tilde\vv^{t - 1}\|_2 < 1 / 2$ for all $t \in \{1, 2, \cdots, T_{\eps} - 1\}$. This completes the proof for Part I. The takeaway message is that with high probability $T_{\sto} \geq T_{\eps} + 1$. 

\subsubsection*{Part II: after $T_{\eps}$}

In the second part of the proof, we show with high probability $|\langle \tilde\vv^{T_{\eps} + 1}, \tilde\vv^{T_{\eps} + 2} \rangle| \geq 1 / 2$, hence $T_{\sto} \leq T_{\eps} + 4$. 

We first prove with probability $1 - o_n(1)$ we have $|\langle \vv, \tilde\vv^{T_{\eps} + i} \rangle| \geq 1 - \delta$ for all $i \in \{1, 2, 3, 4\}$. 
According to \cref{eq:25}, we see that with probability $1 - o_n(1)$ we have $|\langle \tilde\vv^{T_{\eps} + 1}, \vv\rangle| \geq 1 - C n^{-1/24}$. This already proves the desired result for $i = 1$. 

Once again we apply standard Gaussian concentration inequalities, and obtain that  with probability $1 - o_n(1)$
\begin{align}\label{eq:gaussian-bound-2}
	\left| \langle \ww_{i_1, i_2, \cdots, i_{k - 1}}, \vv \rangle \right| \leq \log n, \qquad \left| \langle \bg_t, \vv \rangle \right| \leq \log n
\end{align}
for all $t \in \{0, 1, \cdots, \lceil n^{1 / 10(k - 1)} \rceil\}$ and $(i_1, \cdots, i_{k - 1}) \in \HH_t$. 

Applying \cref{eq:gaussian-bounds} and the triangle inequality to \cref{eq:expression-vt+1}, we get
\begin{align*}
	\|\vv^{T_{\eps} + 2} \|_2 \geq &\, |\alpha_{T_{\eps} + 2}| - \sum\limits_{(i_1, \cdots, i_{k - 1}) \in \HH_{T_{\eps}}} |\beta_{i_1, i_2, \cdots, i_{k - 1}}^{(T_{\eps} + 1)}| \cdot \|\ww_{i_1, i_2, \cdots, i_{k - 1}}\|_2 - c_{T_{\eps} + 2} \| \bg_{T_{\eps} + 2} \|_2 \\
	\geq & \, \lambda_n (1 - C n^{-1/24})^{k - 1} - n^{1/10} \cdot  \sqrt{n \log n}, 
\end{align*}
where to get the second lower bound we use the equality $|\alpha_{T_{\eps} + 2}| = \lambda_n \cdot |\langle \vv, \tilde{\vv}^{T_{\eps} + 1} \rangle^{k - 1}|$. 
Similarly, we obtain $\|\vv^{T_{\eps} + 2} \|_2 \leq \lambda_n + n^{1/10} \cdot  \sqrt{n \log n}$. Therefore, 
\begin{align*}
	& |\langle \tilde \vv^{T_{\eps} + 2}, \vv \rangle| =  \frac{|\langle \vv^{T_{\eps} + 2}, \vv \rangle|}{\|\vv^{T_{\eps} + 2}\|_2} \\
	 = & \frac{1}{\|\vv^{T_{\eps} + 2}\|_2}  \cdot \left| \alpha_{T_{\eps} + 2} + \sum_{(i_1, \cdots, i_{k - 1}) \in \HH_{T_{\eps}}} \beta_{i_1, i_2, \cdots, i_{k - 1}}^{(T_{\eps} + 1)} \langle  \ww_{i_1, i_2, \cdots, i_{k - 1}}, \vv \rangle + c_{T_{\eps} + 2} \langle \bg_{T_{\eps} + 2}, \vv \rangle \right| \\
	 \geq & \frac{1}{\lambda_n + n^{1/10} \cdot \sqrt{n \log n}} \cdot \left(\lambda_n (1 - C n^{-1/24})^{k - 1} - n^{1/10} \log n \right) \geq 1 - o_n(1)
\end{align*}
with probability $1 - o_n(1)$. Following the same route, we are able to conclude that with high probability $|\langle \tilde \vv^{T_{\eps} + 3}, \vv \rangle| = 1 - o_n(1)$ and $|\langle \tilde \vv^{T_{\eps} + 3}, \vv \rangle| = 1 - o_n(1)$. 
As a direct consequence, we see that with high probability $\min_{i \in \{1,2,3,4\}} |\langle \tilde \vv^{T_{\eps} + i}, \vv \rangle| \geq 1 - \delta$. 

Finally, we show $|\langle \tilde{\vv}^{T_{\eps} + 1}, \tilde{\vv}^{T_{\eps} + 2} \rangle| >  1 / 2$. Let $s = \sign (\langle \vv, \tilde\vv^{T_{\eps} + 1} \rangle)$, then 
\begin{align}
\label{eq:76}
	|\langle \tilde\vv^{T_{\eps} + 1}, \tilde\vv^{T_{\eps} + 2} \rangle| =   |\langle \tilde\vv^{T_{\eps} + 1}, \tilde\vv^{T_{\eps} + 2} - s\vv  \rangle + s \langle \vv, \tilde \vv^{T_{\eps} + 1} \rangle| \geq  |\langle \vv, \tilde\vv^{T_{\eps} + 1} \rangle| - \|\tilde\vv^{T_{\eps} + 2} - s\vv\|_2. 
\end{align}
With high probability 
\begin{align*}
	\|\tilde\vv^{T_{\eps} + 2} - s \vv\|_2^2 = 2 - 2s \langle \vv, \tilde\vv^{T_{\eps} + 2} \rangle = 2 - 2|\langle \vv, \tilde\vv^{T_{\eps} + 2} \rangle| = o_n(1).
\end{align*}
Plugging the above upper bound into \cref{eq:76}, we get 
\begin{align*}
	|\langle \tilde\vv^{T_{\eps} + 1}, \tilde\vv^{T_{\eps} + 2} \rangle| \geq 1 - o_n(1), 
\end{align*}
which further implies that with probability $1 - o_n(1)$ the inner product $|\langle \tilde\vv^{T_{\eps} + 1}, \tilde\vv^{T_{\eps} + 2} \rangle|$ is larger than $1 / 2$. Hence, $T_{\sto} \leq T_{\eps} + 4$. Recall that we have proved $\min_{i \in \{1,2,3,4\}} |\langle \tilde \vv^{T_{\eps} + i}, \vv \rangle| \geq 1 - \delta$ with high probability, which further implies that with high probability $|\langle \vv, \tilde\vv^{T_{\sto}} \rangle| \geq 1 - \delta$. The proof is complete.

\section{Additional experiments}
\label{sec:additional_exp}

\subsection{Comparing alignment and the polynomial recurrence process}

We collect in this section additional experiments for \cref{sec:exp-compare}. The basic setup is the same as that presented in the main text.  Throughout the experiment, we fix $\lambda_n = n^{(k - 1) / 2}$, and use different combinations of $(n, k)$. 
As before, we compare the marginal distributions of $\{\alpha_t\}_{t \geq 0}$ and $\{X_t\}_{t \geq 0}$ by comparing the histograms of their empirical marginal distributions, generated from 1000 independent experiments. Observing the simulation outcomes, we see that they all match well. 

\subsubsection*{Setting I: $\mathbf{n = 100, k = 3}$} 

Simulation results are plotted as Figure \ref{fig:comparison1}.

\begin{figure}[ht]
  \begin{minipage}[ht]{.247\linewidth}
    \includegraphics[width=\linewidth]{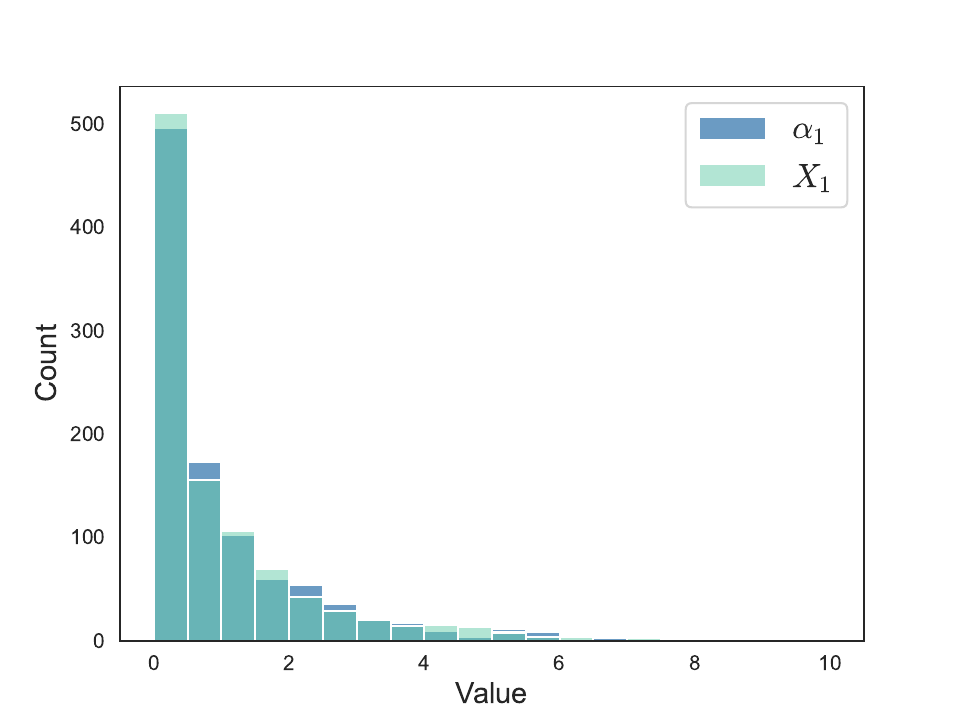}%
  \end{minipage}\hfil
  \begin{minipage}[ht]{.247\linewidth}
    \includegraphics[width=\linewidth]{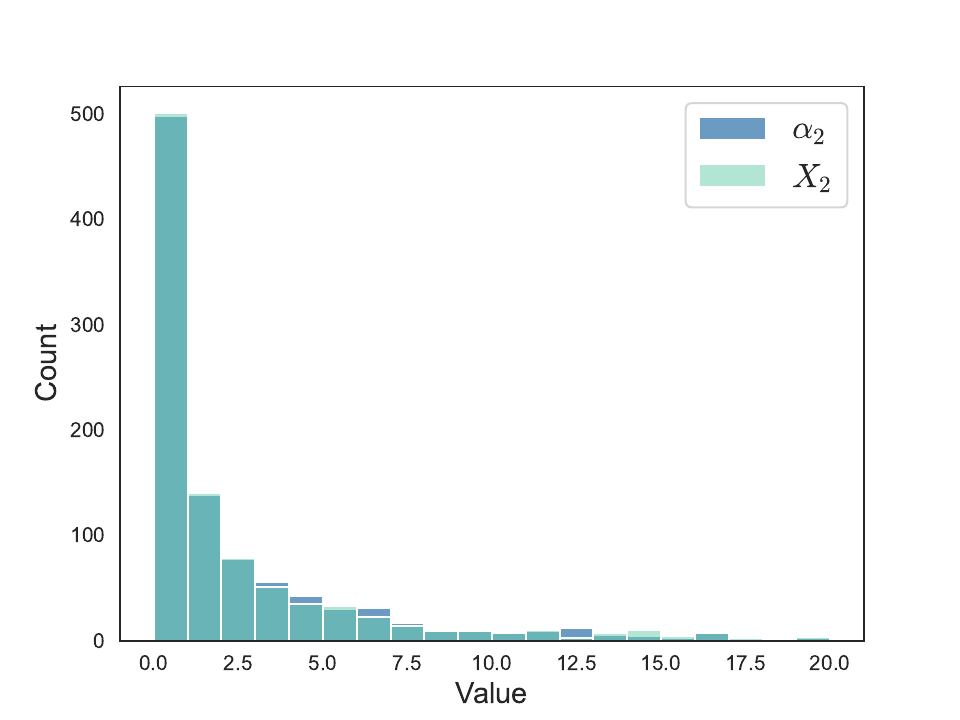}%
  \end{minipage}\hfil 
  \begin{minipage}[ht]{.247\linewidth}
    \includegraphics[width=\linewidth]{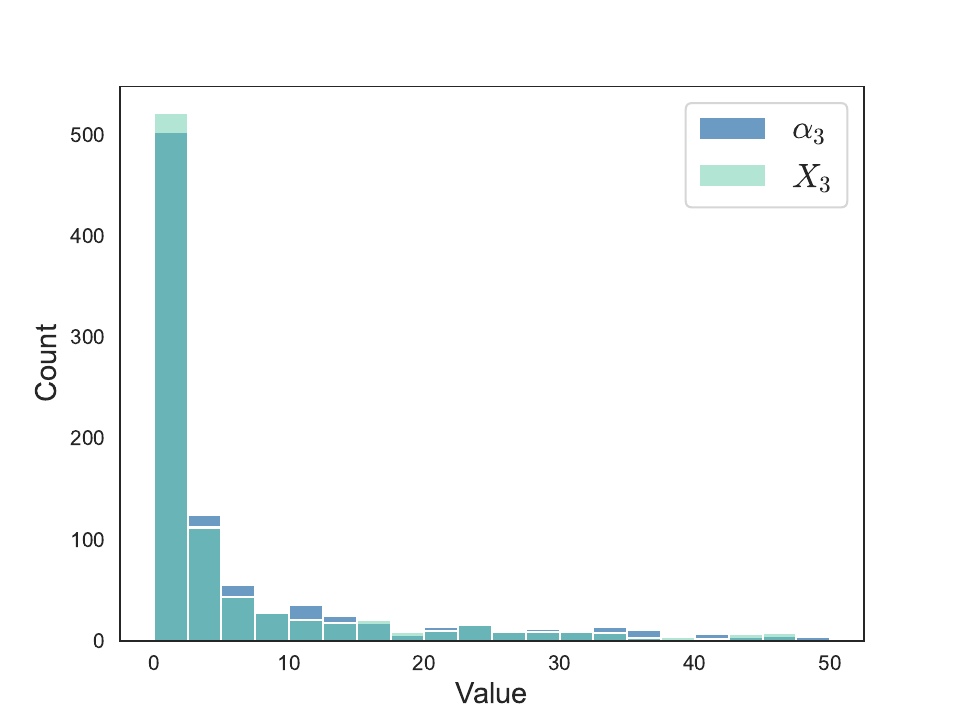}%
  \end{minipage}
  \begin{minipage}[ht]{.247\linewidth}
    \includegraphics[width=\linewidth]{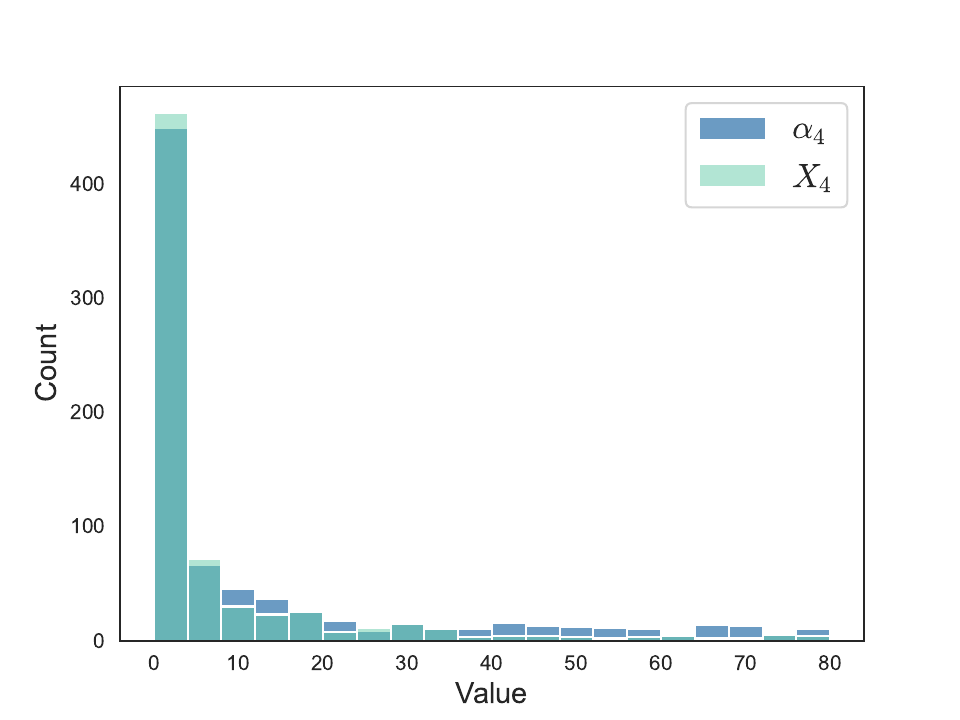}%
  \end{minipage}
  \caption{Comparison of the marginal distributions between $\alpha_t$ and $X_t$, for $t \in \{1, 2, 3, 4\}$ from left to right. Here, we set $n = 100$, $k = 3$, $\lambda_n = n^{(k - 1) / 2}$, and run tensor power iteration from random initialization on independent datasets for 1000 times. }
\label{fig:comparison1}
\end{figure}

\subsubsection*{Setting II: $\mathbf{n = 500, k = 3}$}  

Simulation results are plotted as Figure \ref{fig:comparison2}. 

\begin{figure}[ht]
  \begin{minipage}[ht]{.247\linewidth}
    \includegraphics[width=\linewidth]{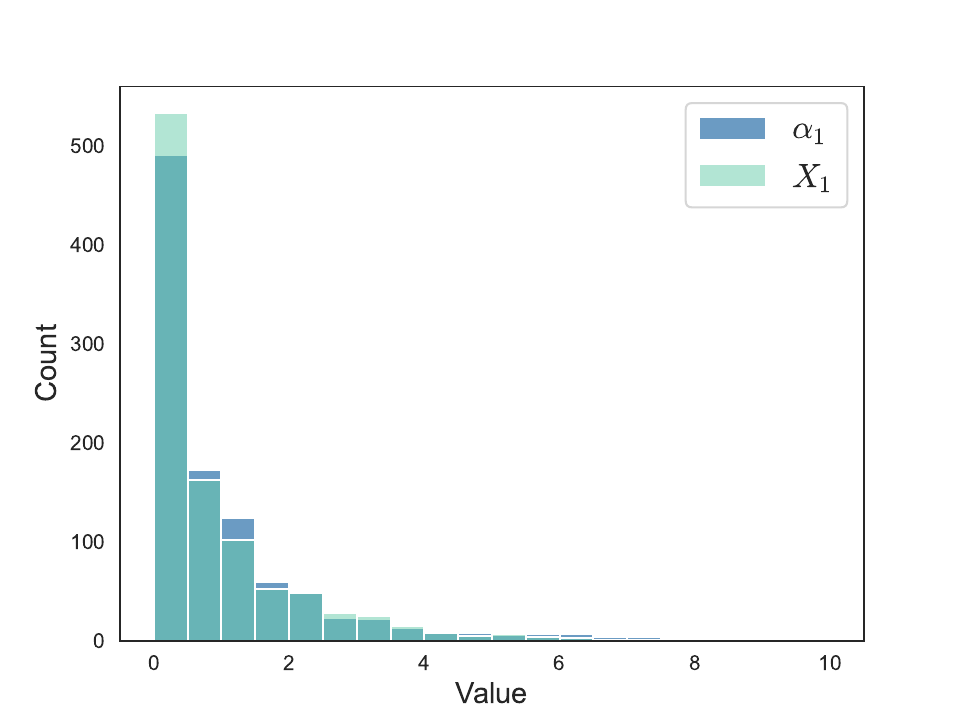}%
  \end{minipage}\hfil
  \begin{minipage}[ht]{.247\linewidth}
    \includegraphics[width=\linewidth]{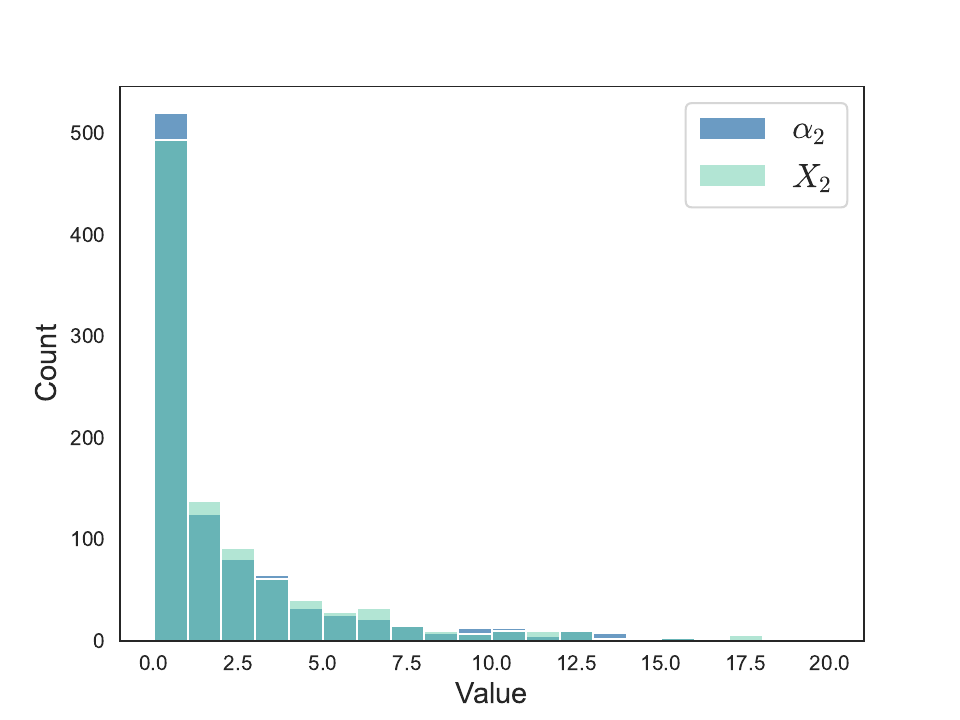}%
  \end{minipage}\hfil 
  \begin{minipage}[ht]{.247\linewidth}
    \includegraphics[width=\linewidth]{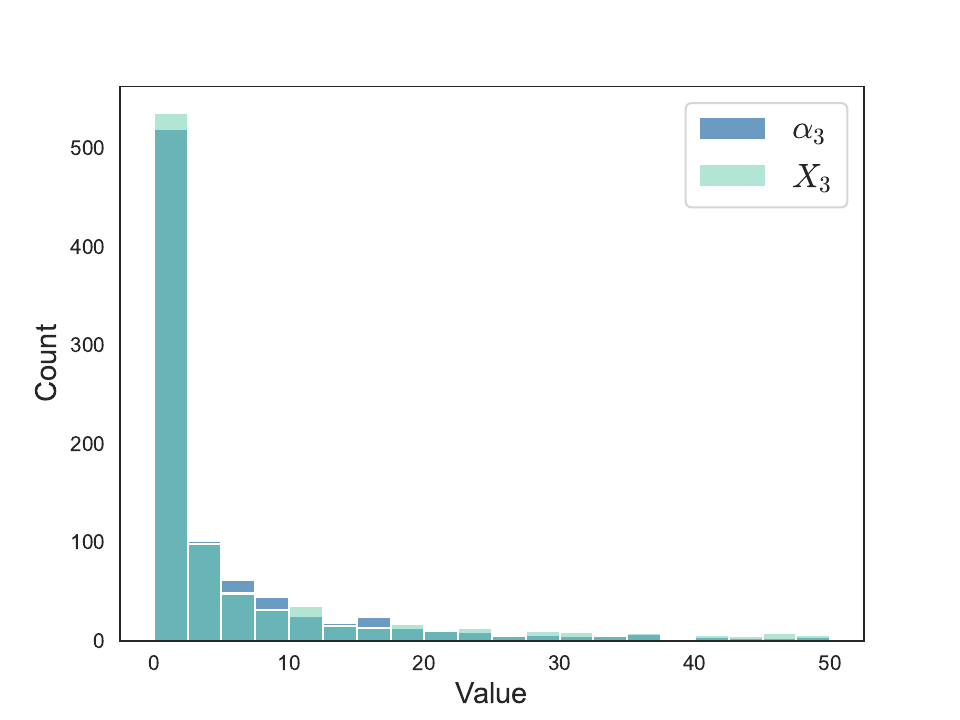}%
  \end{minipage}
  \begin{minipage}[ht]{.247\linewidth}
    \includegraphics[width=\linewidth]{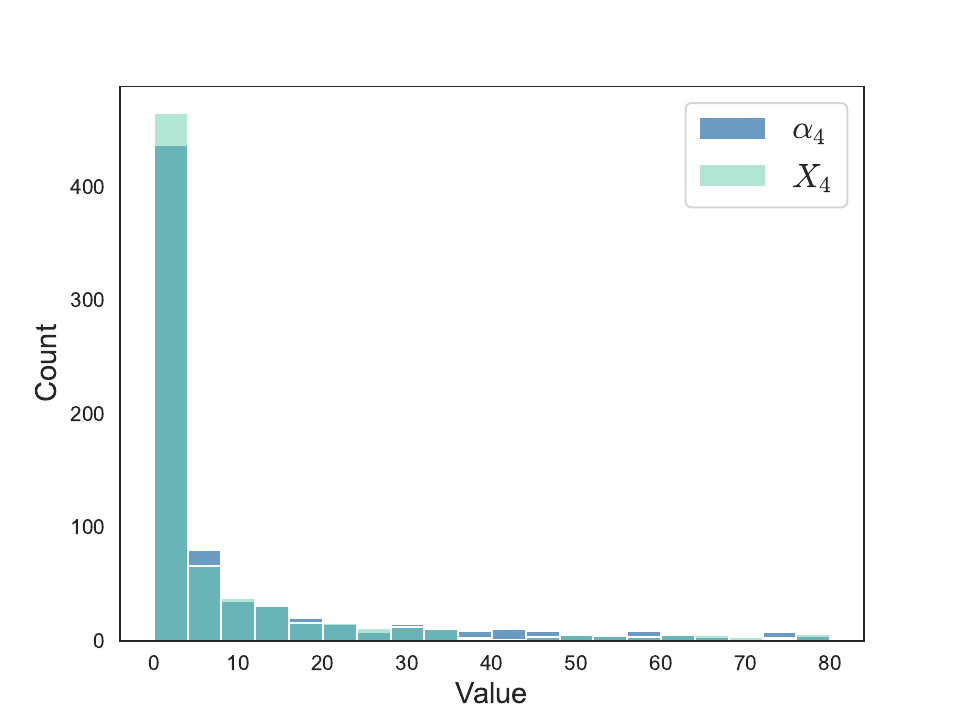}%
  \end{minipage}
  \caption{Comparison of the marginal distributions between $\alpha_t$ and $X_t$, for $t \in \{1, 2, 3, 4\}$ from left to right. Here, we set $n = 500$, $k = 3$, $\lambda_n = n^{(k - 1) / 2}$, and run tensor power iteration from random initialization on independent datasets for 1000 times. }
\label{fig:comparison2}
\end{figure}

\subsubsection*{Setting III: $\mathbf{n = 1000, k = 3}$}

Simulation results are plotted as Figure \ref{fig:comparison3}. 

\begin{figure}[ht]
  \begin{minipage}[ht]{.247\linewidth}
    \includegraphics[width=\linewidth]{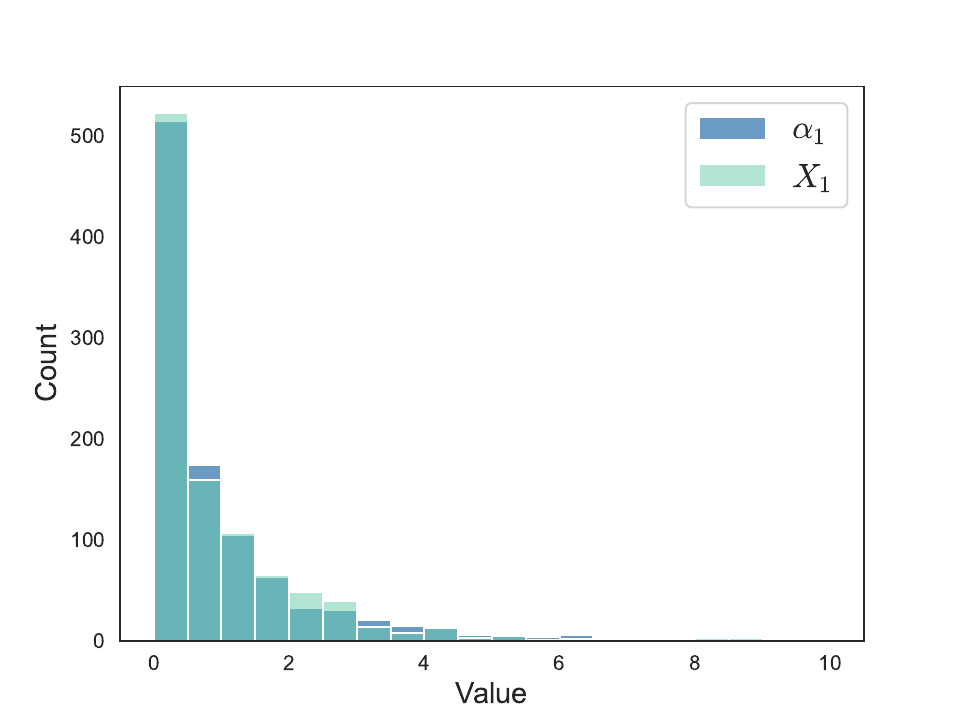}%
  \end{minipage}\hfil
  \begin{minipage}[ht]{.247\linewidth}
    \includegraphics[width=\linewidth]{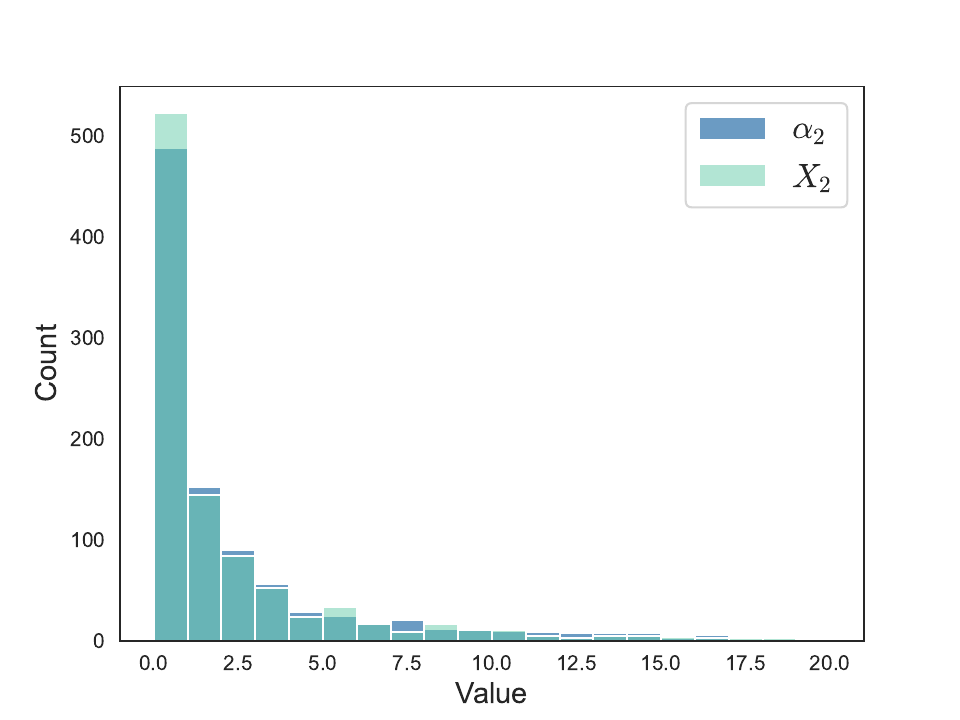}%
  \end{minipage}\hfil 
  \begin{minipage}[ht]{.247\linewidth}
    \includegraphics[width=\linewidth]{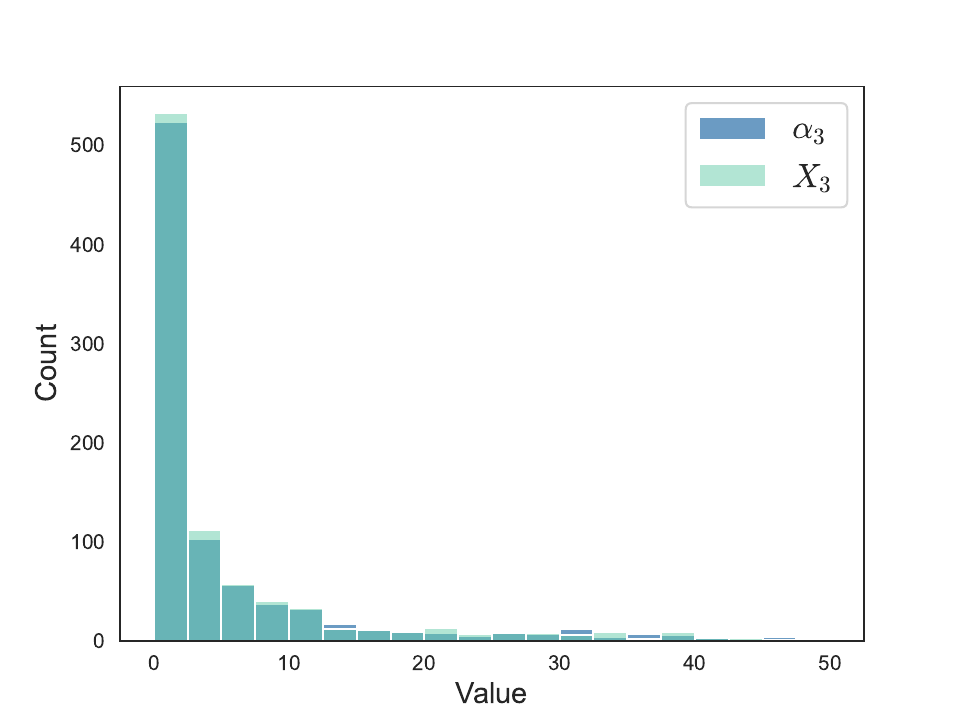}%
  \end{minipage}
  \begin{minipage}[ht]{.247\linewidth}
    \includegraphics[width=\linewidth]{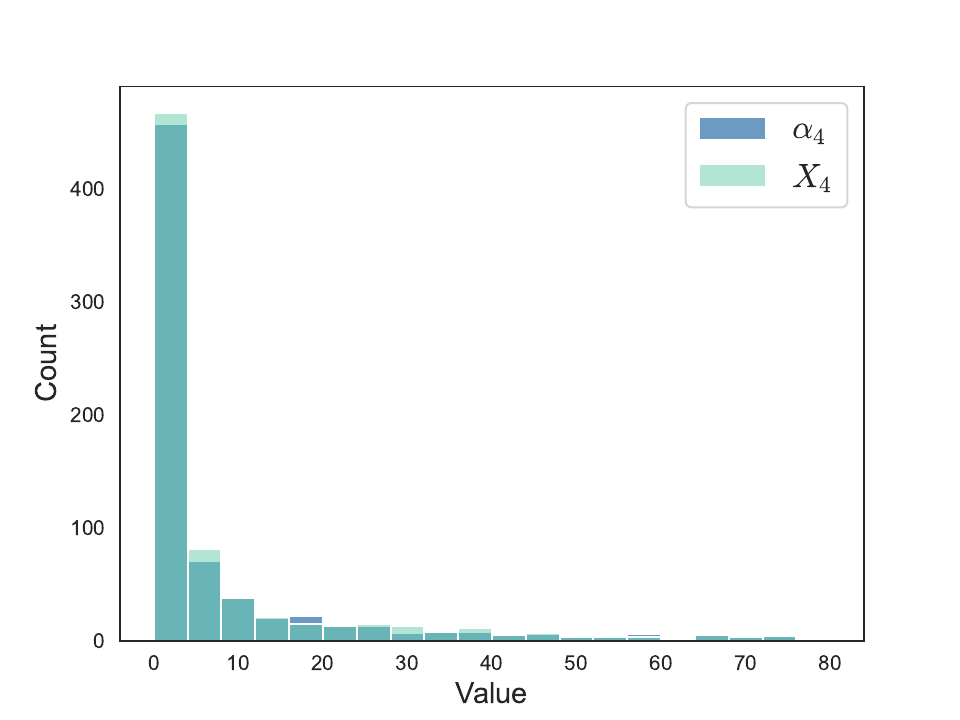}%
  \end{minipage}
  \caption{Comparison of the marginal distributions between $\alpha_t$ and $X_t$, for $t \in \{1, 2, 3, 4\}$ from left to right. Here, we set $n = 1000$, $k = 3$, $\lambda_n = n^{(k - 1) / 2}$, and run tensor power iteration from random initialization on independent datasets for 1000 times. }
\label{fig:comparison3}
\end{figure}

\subsubsection*{Setting IV: $\mathbf{n = 100, k = 4}$}  

Simulation results are plotted as Figure \ref{fig:comparison4}. 

\begin{figure}[ht]
  \begin{minipage}[ht]{.247\linewidth}
    \includegraphics[width=\linewidth]{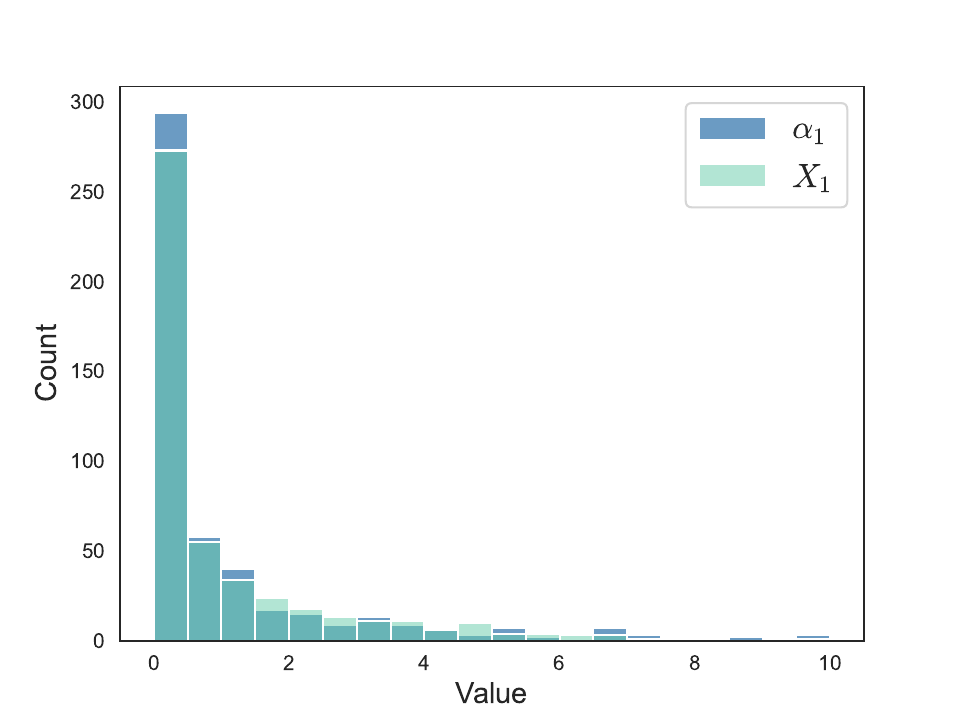}%
  \end{minipage}\hfil
  \begin{minipage}[ht]{.247\linewidth}
    \includegraphics[width=\linewidth]{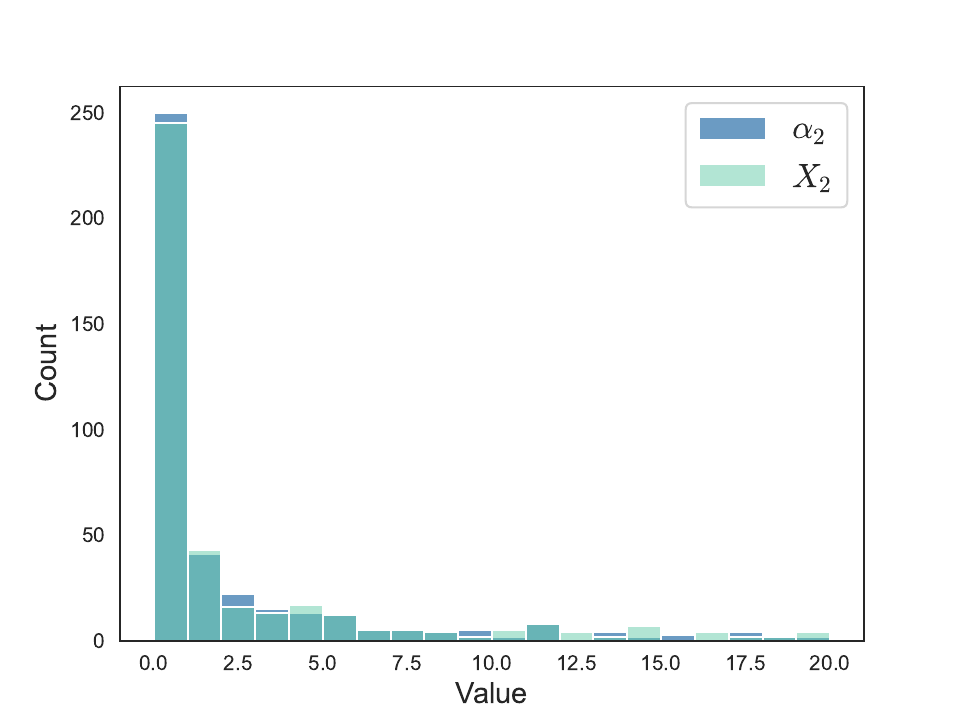}%
  \end{minipage}\hfil 
  \begin{minipage}[ht]{.247\linewidth}
    \includegraphics[width=\linewidth]{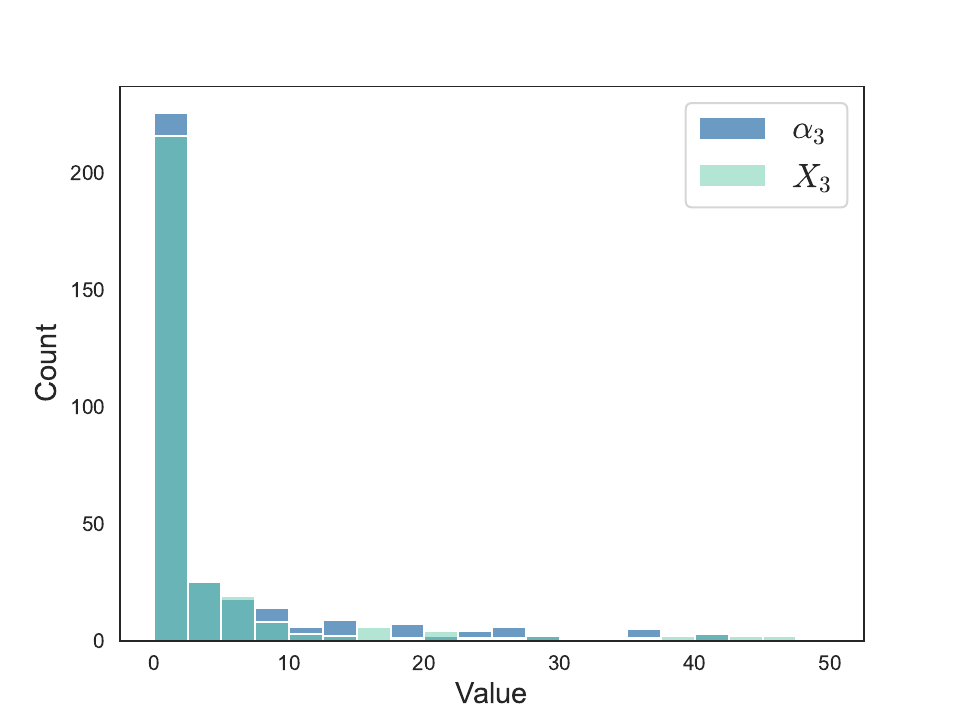}%
  \end{minipage}
  \begin{minipage}[ht]{.247\linewidth}
    \includegraphics[width=\linewidth]{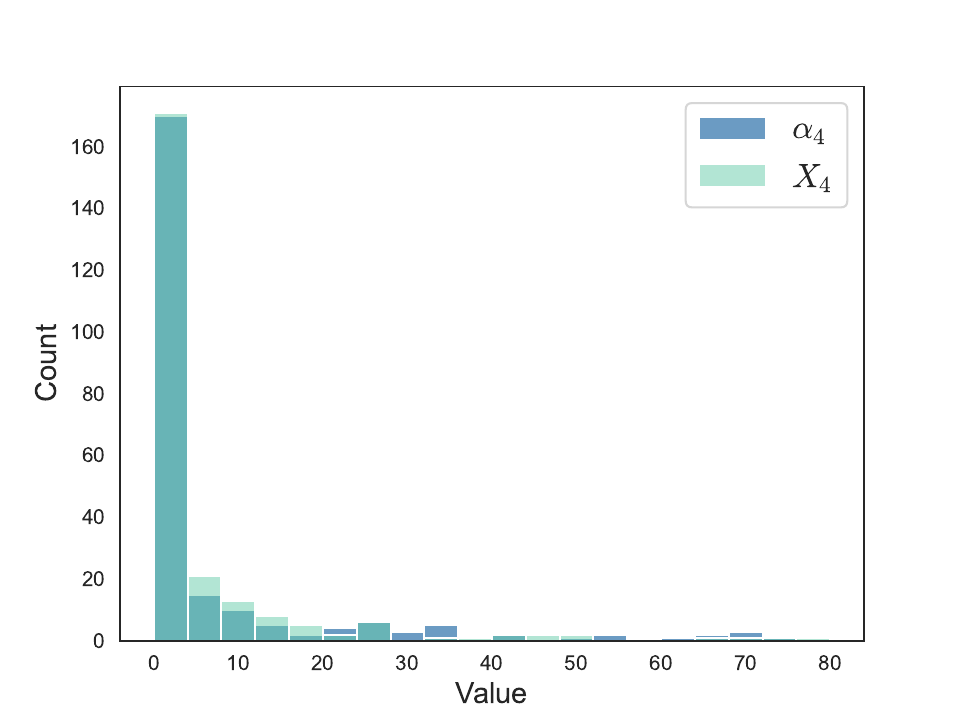}%
  \end{minipage}
  \caption{Comparison of the marginal distributions between $\alpha_t$ and $X_t$, for $t \in \{1, 2, 3, 4\}$ from left to right. Here, we set $n = 100$, $k = 4$, $\lambda_n = n^{(k - 1) / 2}$, and run tensor power iteration from random initialization on independent datasets for 1000 times. }
\label{fig:comparison4}
\end{figure}

\subsubsection*{Setting IV: $\mathbf{n = 200, k = 4}$} 

Simulation results are plotted as Figure \ref{fig:comparison5}. 

\begin{figure}[ht]
  \begin{minipage}[ht]{.247\linewidth}
    \includegraphics[width=\linewidth]{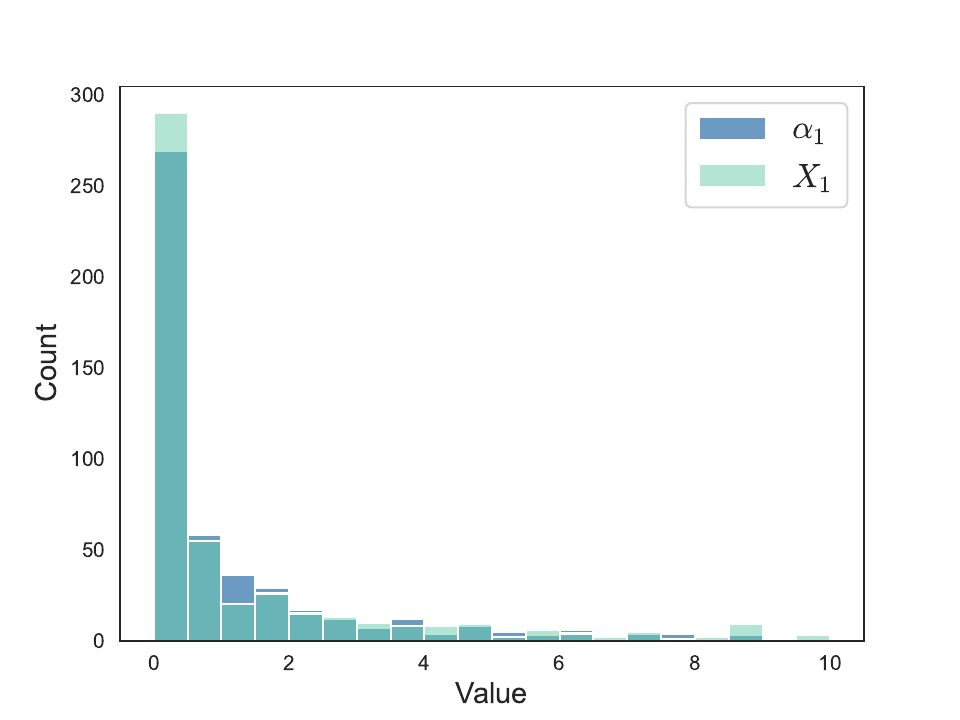}%
  \end{minipage}\hfil
  \begin{minipage}[ht]{.247\linewidth}
    \includegraphics[width=\linewidth]{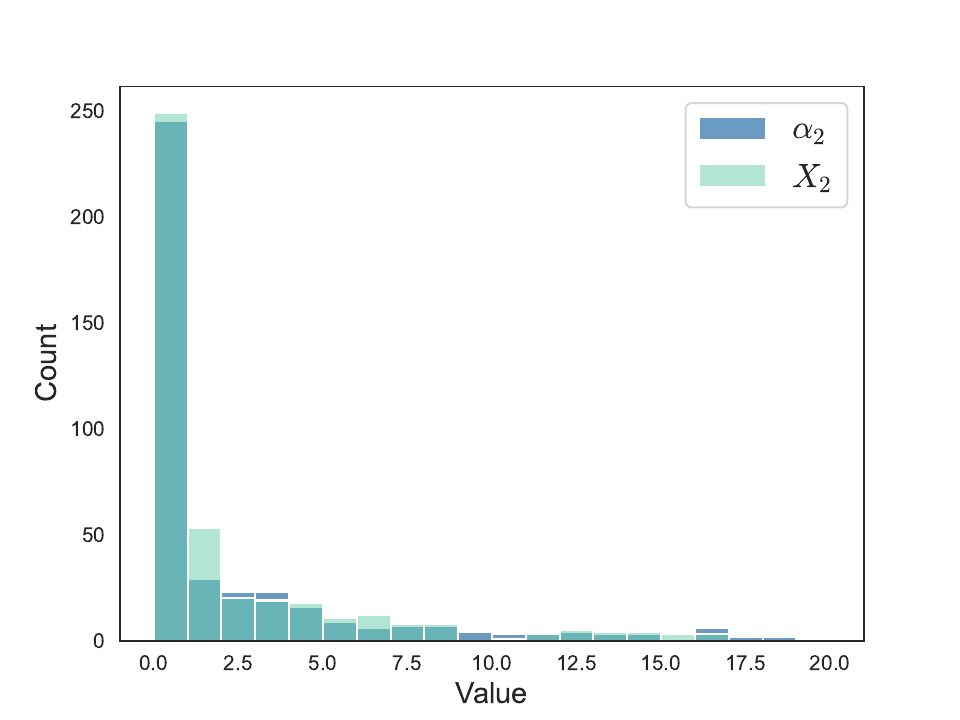}%
  \end{minipage}\hfil 
  \begin{minipage}[ht]{.247\linewidth}
    \includegraphics[width=\linewidth]{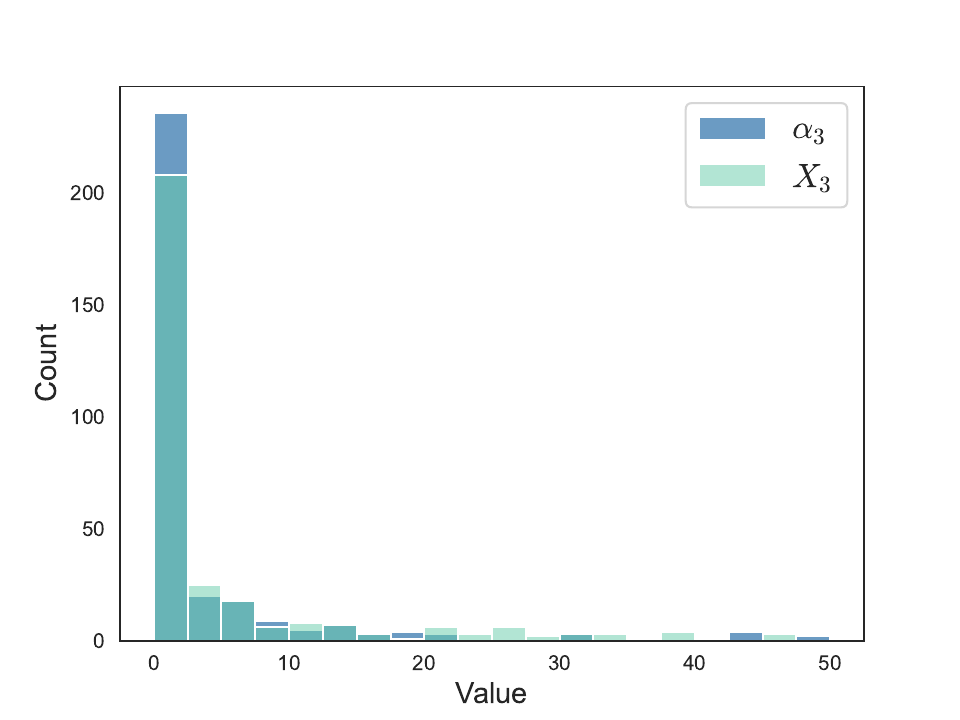}%
  \end{minipage}
  \begin{minipage}[ht]{.247\linewidth}
    \includegraphics[width=\linewidth]{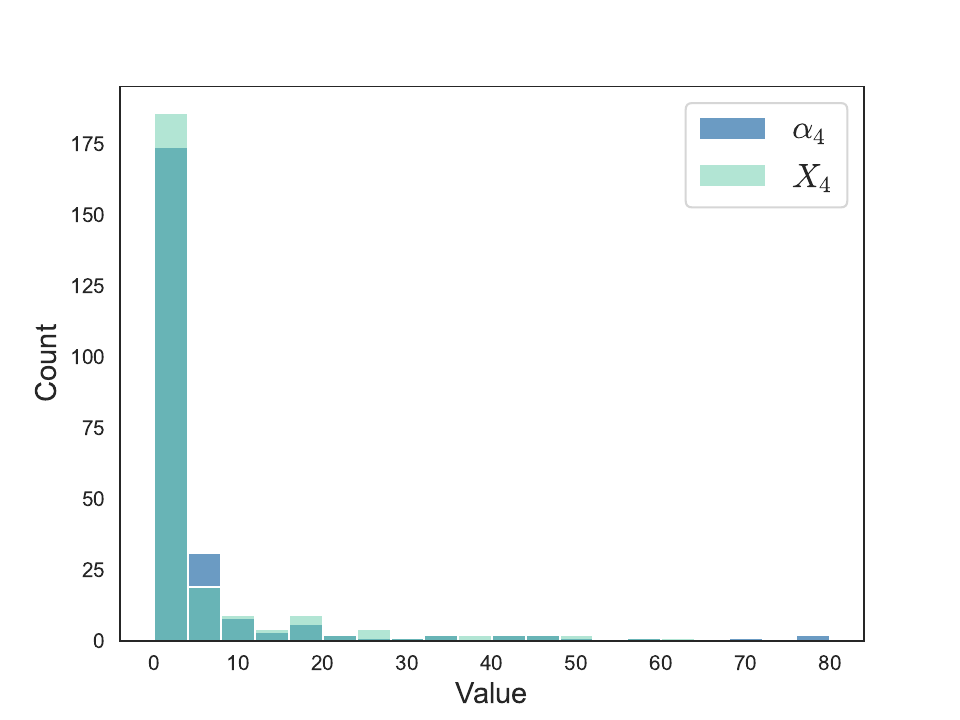}%
  \end{minipage}
  \caption{Comparison of the marginal distributions between $\alpha_t$ and $X_t$, for $t \in \{1, 2, 3, 4\}$ from left to right. Here, we set $n = 200$, $k = 4$, $\lambda_n = n^{(k - 1) / 2}$, and run tensor power iteration from random initialization on independent datasets for 1000 times. }
\label{fig:comparison5}
\end{figure}

\modif{\subsection{Experiments with different stopping thresholds}
\label{sec:stopping-rule-extra}
In this section, we adjust the stopping threshold value, and test the proposed stopping criterion under these values. 
To be specific, we consider here thresholds $0.3$ and $0.7$, and conduct the corresponding experiments. 
Counterparts of Figure \ref{fig:stopping-rule} are presented as Figures \ref{fig:stopping-rule-0.3} and \ref{fig:stopping-rule-0.7}. These figures suggest that the effectiveness of the stopping rule is not sensitive to the choice of the stopping threshold.} 

\begin{figure}[ht]
	\centering
	\includegraphics[width=0.9\linewidth]{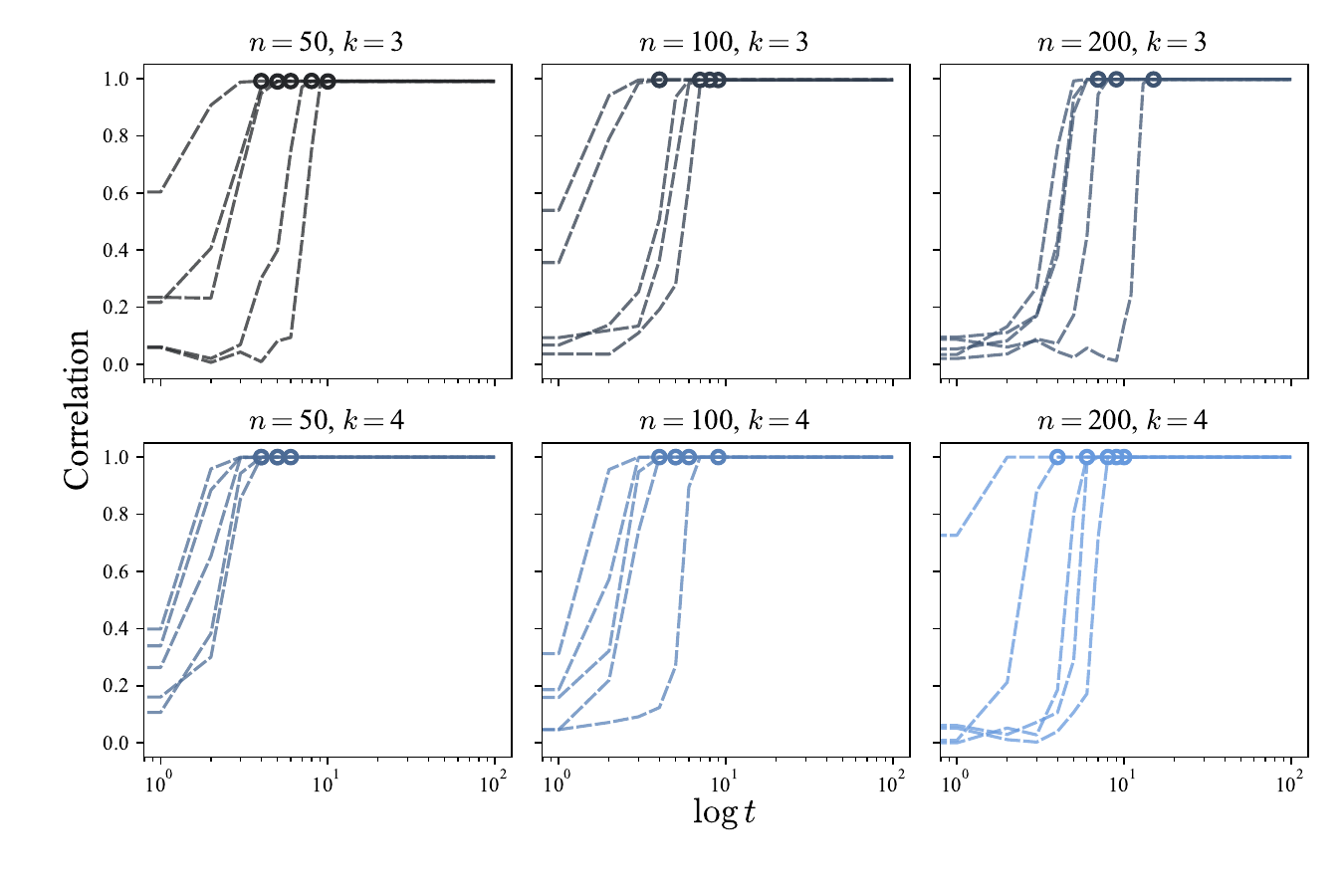}
	\caption{Illustration of the effectiveness of the stopping rule with stopping threshold $0.3$. }
	\label{fig:stopping-rule-0.3}
\end{figure}

\begin{figure}[ht]
	\centering
	\includegraphics[width=0.9\linewidth]{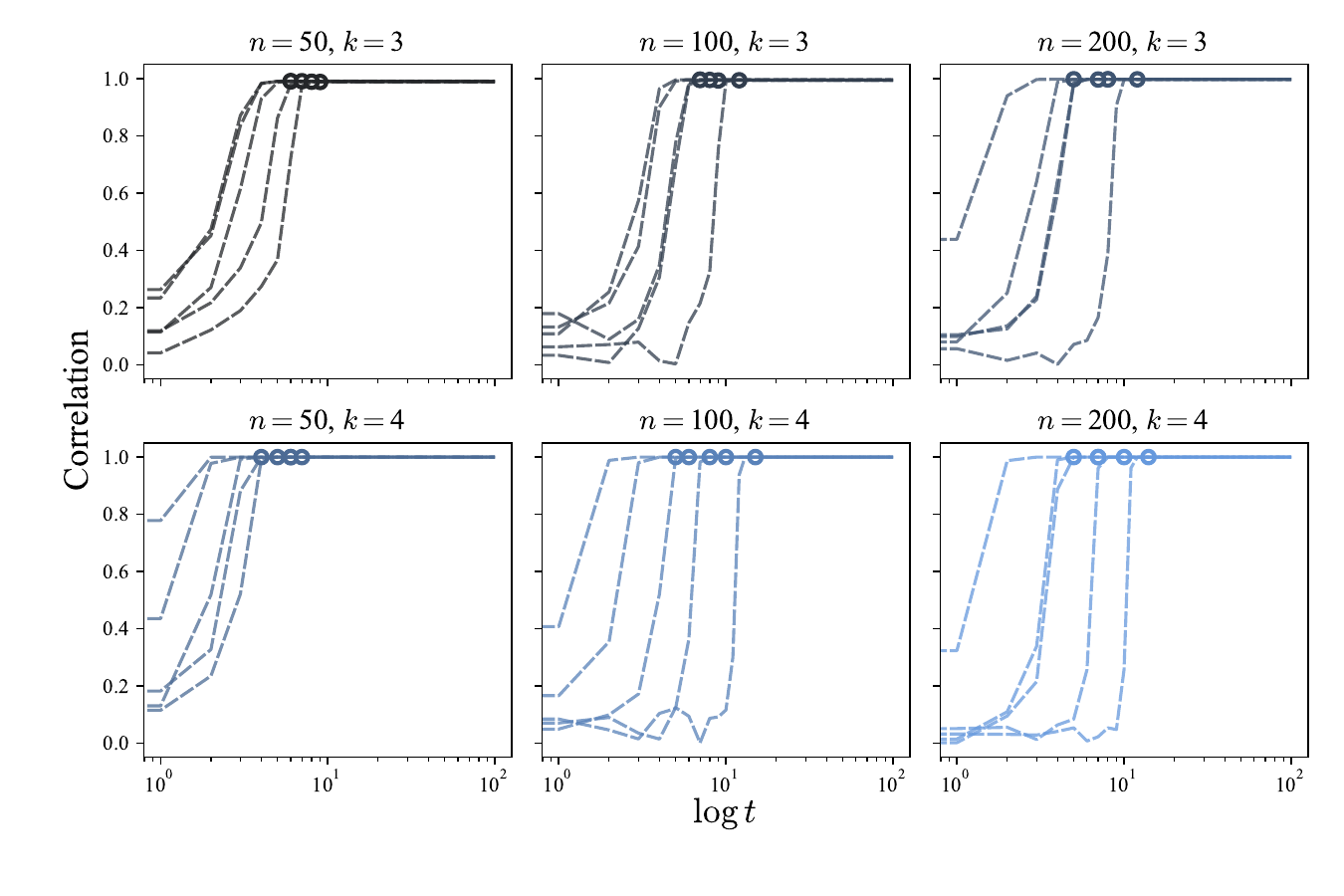}
	\caption{Illustration of the effectiveness of the stopping rule with stopping threshold $0.7$. }
	\label{fig:stopping-rule-0.7}
\end{figure}

\end{document}